\documentclass[11pt]{article}

\usepackage{geometry}
\usepackage[utf8]{inputenc} % allow utf-8 input
\usepackage[T1]{fontenc}    % use 8-bit T1 fonts
\usepackage{hyperref}       % hyperlinks
\usepackage{url}            % simple URL typesetting
\usepackage{booktabs}       % professional-quality tables
\usepackage{amsfonts}       % blackboard math symbols
\usepackage{nicefrac}       % compact symbols for 1/2, etc.
\usepackage{microtype}      % microtypography
\usepackage{xcolor}         % colors

\usepackage{authblk}
\usepackage{times}
\usepackage{soul}
\usepackage{url}
\usepackage[utf8]{inputenc}
\usepackage[small]{caption}
\usepackage{graphicx}
\usepackage{amsmath}
\usepackage{amsthm}
\usepackage{booktabs}
\usepackage{algorithm}
\usepackage{algorithmic}
\usepackage{comment}
\usepackage[title]{appendix}

\usepackage{hhline}
\usepackage{multicol}
\usepackage{multirow}
\usepackage{subfiles}
\usepackage{wrapfig,lipsum,booktabs}
\usepackage{amssymb}
\newcommand{\statname}{\texttt{MAGDiff}}
\newcommand{\distrib}{\mathbb{P}}
\newcommand{\expect}{\mathbb{E}}
\newcommand{\R}{\mathbb{R}}

\DeclareMathOperator*{\argmax}{arg\,max}

\usepackage{xcolor}
\newcommand{\todo}[1]{\textcolor{red}{ToDo: #1}}

\newcommand*{\charles}[1]{{\textcolor{blue}{[\textbf{CA:} #1]}}}
\newcommand*{\felix}[1]{{\textcolor{purple}{[\textbf{FH:} #1]}}}

\usepackage{mathtools}

\newtheorem{fact}{Fact}
\newtheorem{corollary}{Corollary}
\newtheorem{proposition}{Proposition}

\title{{\statname{}}: Covariate Data Set Shift Detection via Activation Graphs of Deep Neural Networks}

\author[1]{Charles Arnal$^*$}
\author[1]{Felix Hensel$^*$}
\author[2]{Mathieu Carri\`ere}
\author[3]{Th\'eo Lacombe}
\author[4]{Hiroaki Kurihara}
\author[5]{Yuichi Ike}
\author[1]{Fr\'ed\'eric Chazal}

\affil[1]{Universit{\'e} Paris-Saclay, CNRS, Inria, Laboratoire de Math{\'e}matiques d'Orsay, France}
\affil[2]{Universit{\'e} C{\^o}te d'Azur, Inria, France}
\affil[3]{LIGM, Université Gustave Eiffel, Champs-sur-Marne, France}
\affil[4]{Fujitsu Limited, Japan}
\affil[5]{Institute of Mathematics for Industry, Kyushu University, Japan}

\begin{document}

\title{{\statname{}}: Covariate Data Set Shift Detection via Activation Graphs of Neural Networks}

\def\thefootnote{*}\footnotetext{These authors contributed equally to this work.}\def\thefootnote{\arabic{footnote}}

\maketitle

\begin{abstract}
Despite their successful application to a variety of tasks, neural networks remain limited, like other machine learning methods, by their sensitivity to shifts in the data: their performance can be severely impacted by differences in distribution between the data on which they were trained and that on which they are deployed. 
 In this article, we propose a new family of representations, called \statname{}, that we extract from any given neural network classifier and that allows for efficient covariate data shift detection without the need to train a new model dedicated to this task. These representations are computed by comparing the activation graphs of the neural network for samples belonging to the training distribution and to the target distribution, and yield powerful data- and task-adapted statistics for the two-sample tests commonly used for data set shift detection. We demonstrate this empirically by measuring the statistical powers of two-sample Kolmogorov-Smirnov (KS) tests on several different data sets and shift types, and showing that our novel representations induce significant improvements over a state-of-the-art baseline relying on the network output.
\end{abstract}

\section{Introduction}

During the last decade, neural networks (NN) have become immensely popular, reaching state-of-the-art performances in a wide range of situations. 
Nonetheless, once deployed in real-life settings, NN can face various challenges such as being subject to adversarial attacks \cite{huang2017adversarial}, being exposed to out-of-distributions samples (samples that were not presented at training time) \cite{hendrycks2016baseline}, or more generally being exposed to a \emph{distribution shift}: when the distribution of inputs %progressively
differs from the training distribution (\textit{e.g.}, input objects are exposed to a 
%an increasing level of
corruption due to deterioration of measure instruments such as cameras or sensors). 
Such distribution shifts are likely to degrade performances of presumably well-trained models \cite{wiles2021fine}, and being able to detect such shifts is a key challenge in monitoring NN once deployed in real-life applications. 
%
%\begin{comment}
%Monitoring the performance and behaviour of those networks can prove to be a difficult task, as many aspects of their inner workings remain poorly understood.
%\end{comment}
%
Though shift detection for univariate variables is a well-studied problem, the task gets considerably harder with high-dimensional data, and seemingly reasonable methods often end up performing poorly \cite{Ramdas2014OnTD}. %Nonetheless, progress has been made, particularly in cases where additional assumptions can be made on the nature of the shift. \charles{delete that last sentence ?}

In this work, we introduce the Mean Activation Graph Difference (\statname{}), a new approach that harnesses the powerful dimensionality reduction capacity of deep neural networks in a data- and task-adapted way. The key idea, further detailed in Section \ref{sec:stat}, is to consider the activation graphs generated by inputs as they are processed by a neural network that has already been trained for a classification task, and to compare such graphs to those associated to samples 
%for samples 
from the 
%source and target distributions. 
training distribution.
The method can thus be straightforwardly added as a diagnostic tool on top of preexisting classifiers without requiring any further training ; it is easy to implement, and computationally inexpensive. As the activation graphs depend on the network weights, which in turn have been trained for the data and task at hand, one can also hope for them to capture information that is most relevant to the context. Hence, our method can easily support, and benefit from, any improvements in deep learning.

Our approach is to be compared to {\em Black box shift detection}~(BBSD), a method introduced in \cite{Lipton2018,Rabanser2019} that shares a similar philosophy. 
BBSD uses the output of a trained classifier to efficiently detect various types of shifts (see also Section \ref{sec:stat}); in their experiments, BBSD generally beats other methods, the runner-up being a much more complex and computationally costly multivariate two-sample test combining an autoencoder and the Maximum Mean Discrepancy statistic~\cite{JMLR:v13:gretton12a}.

%We show that on a variety of datasets, shift types and shift intensities, \statname{}{} outperforms BBSD, often by a wide margin.

%\charles{Last paragraph a little too short; combine with previous one ? add a few sentences ?}
%\charles{Should we add a short summary of the article ? I feel like the plan is fairly straightforward, and a summary isn't needed}

\vspace{10pt}

Our contributions are summarized as follows.
\begin{enumerate}
    \item Given any neural network classifier, we introduce a new family of representations \statname{}, that is obtained by comparing the activation graphs of samples to the mean activation graph of each class in the training set.
    \item We propose to use \statname{}{} as a statistic for data set shift detection. 
    More precisely, we combine our representations with the statistical method that was proposed and applied to the Confidence Vectors (CV) of classifiers in \cite{Lipton2018}, yielding a new method for shift detection.
    \item We experimentally show that our shift detection method with \statname{}{} outperforms the state-of-the-art BBSD with CV on a variety of datasets, covariate shift types and shift intensities, often by a wide margin.
    Our code is provided in the Supplementary Material and will be released publicly.
\end{enumerate}

\section{Related Work}\label{sec:related_works}

%\charles{Question : should we write "Confidence Vectors (CV)" only once, then use "CV" everywhere, or do we write "Confidence Vectors (CV)" again if they haven't been mentioned in a while (in case readers have forgotten) ? }
%\yuichi{Personally, I'd like to write "Confidence Vectors (CV)" explicitly again in Section~3 because we define them formally in the section for the first time. After that we can call them just CV. I'm not sure if we should do that in Related Work (we have already introduced CV in Introduction).}
%\theo{I think it's nice to recall that CV = Confidence Vector from time to time.}

Detecting changes or outliers in data can be approached from the angle of anomaly detection, a well-studied problem~\cite{AnomalyDetectionSurvey}, or out-of-distribution (OOD) sample detection~\cite{Shafaei2018DoesYM}.
Among techniques that directly frame the problem as shift detection, kernel-based methods such as Maximum Mean Discrepancy (MMD) \cite{JMLR:v13:gretton12a,Zaremba2013} and Kernel Mean Matching (KMM) \cite{gretton2009covariate,Zhang2013} have proved popular, though they scale poorly with the dimensionality of the data \cite{Ramdas2014OnTD}. Using classifiers to test whether samples coming from two distributions can be correctly labeled, hence whether the distributions can be distinguished, has also been attempted; see, \textit{e.g.}, \cite{Kim2021}.
The specific cases of covariate shift \cite{Jang2022,uehara2020off,Rabanser2019} and label shift \cite{Storkey,Lipton2018} have been further investigated, from the point of view of causality and anticausality \cite{CausalAnticausal}.
Moreover, earlier investigations of similar questions have arisen from the fields of economics  \cite{Heckman1977} and epidemiology \cite{saerens2002adjusting}.

Among the works cited above, \cite{Lipton2018} and \cite{Rabanser2019} are of particular interest to us.
In \cite{Lipton2018}, the authors detect label shifts using shifts in the distribution of the outputs of a well-trained classifier; they call this method Black Box Shift Detection (BBSD). In \cite{Rabanser2019}, the authors observe that BBSD tends to generalize very well to covariate shifts, though without the theoretical guarantees it enjoys in the label shift case.
Our method is partially related to BBSD. Roughly summarized, we apply similar statistical tests---combined univariate Kolmogorov-Smirnov tests---to different features---Confidence Vectors (CV) in the case of BBSD, distances to mean activation graphs (\statname{}) in ours.
Similar statistical ideas have also been explored in  
\cite{alberge:hal-02172275} and \cite{BarShalom}, while neural network activation graph features have been studied in, \textit{e.g.}, \cite{Lacombe2021} and \cite{HORTA2021109}.
The related issue of the robustness of various algorithms to diverse types of shifts has been recently investigated in \cite{ShiftGeneral}.

\section{Background}\label{sec:background}

\subsection{Shift Detection with Two-Sample Tests}\label{subsec:shift-detection}

%\paragraph*{DSDBBP.}

\begin{comment}
This article is based \charles{based makes it sound like we don't innovate much} on {\em data set shift detection with black box predictors} (DSDBBP)~\cite{Lipton2018,Rabanser2019}.
The goal of DSDBBP is to design two-sample tests that are able to detect a shift in the data distribution solely from statistics
that are computed from predictor functions (such as, e.g., any classifier or regressor function available in machine learning).
\end{comment}
There can often be a shift between the distribution $\distrib_0$ of data on which a model has been trained and tested and the distribution $\distrib_1$ of the data on which it is used after deployment; many factors can cause such a shift, \textit{e.g.},~a change in the environment, in the data acquisition process, or the training set being unrepresentative.
Detecting shifts is crucial to understanding, and possibly correcting, potential losses in performance; even shifts that do not strongly impact accuracy can be important symptoms of inaccurate assumptions or changes in deployment conditions.
% Detecting shifts is crucial to understanding, and possibly correcting, potential losses in performance.
% In cases where the shift seems to only minimally impact the accuracy of the network, it is particularly challenging and important to be able to detect such a shift.
% As this can lead to "confident" network predictions which are meaningless and thus non-trustworthy.
% % \felix{Maybe we could add a remark about something related to "non-trustworthy" output in the presence of shift? In the sense that even if the performance seemingly is not affected much (an thus anomalies are hard to detect) some unseen malignant effect may be occurring. I've added the previous 2 sentences, feel free to edit.} 

% \charles{I think the two sentences are a bit too strong; they make it sound like it's even MORE important to detect shifts when the accuracy is not affected much, which is an overstatement. I also find it a bit bold to call predictions that are correct meaningless without further explanations/examples. I suggest the following variant (it sends the right message, without going too deep into tricky and possibly unrigorous details), which would replace the sentence that starts in the same way :}

% \charles{Detecting shifts is crucial to understanding, and possibly correcting, potential losses in performance; even shifts that do not strongly impact accuracy can be important symptoms of inaccurate assumptions or changes in deployment conditions. }

Additional assumptions can sometimes be made on the nature of the shift. In the context of a classification task, where data points are of the shape $(x,y)$ with $x$ the feature vector and $y$ the label, a shift that preserves the conditional distribution $p(x|y)$ (but allows the proportion of each label to vary) is called {\em label shift}. Conversely, a {\em covariate shift} occurs when $p(y|x)$ is preserved, but the distribution of $p(x)$ is allowed to change. In this article, we focus on the arguably harder case of covariate shifts. See Section~\ref{sec:expes} for examples of such shifts in numerical experiments.
% \felix{perhaps we could add a remark saying something like: "We mostly focus on covariate shifts in this paper"? We should take off the focus from label shifts sinc our results in this case are not so good (sometimes very bad).}
% \charles{Done above, tell me what you think (+ we have already changed the title, introduction, etc. for additional clarity)}

\begin{comment}
In this work, we only assume that the target distribution $\distrib_1$ was obtained from the source distribution $\distrib_0$ 
with a shift, either in the proportion of the classes in the data, which we refer to as {\em label shift}, or in the conditional 
distributions associated to each class, which we refer to as {\em data shift}. \charles{covariate shift} In other words, when a label shift is applied, 
the fraction of each class in the target distribution $\distrib_1$ is different from their analogue in $\distrib_0$, but the 
conditional distributions stay the same (for instance, they are Gaussian with class-specific parameters in both cases). On the other hand,
when a data shift is applied, these fractions might stay the same, but the conditional data distributions become different 
(for instance, the class-specific parameters change, or the conditional distributions are perturbed with noise models). 
See Section~\ref{sec:expes} for examples of such shifts in numerical experiments.

\end{comment}

Shifts can be detected using {\em two-sample tests}: that is, a statistical test that aims at deciding between the two hypotheses
\[H_0:\distrib_0=\distrib_1\text{ and }H_1:\distrib_0\neq\distrib_1,\] 
given two random sets of samples, $X_0$ and $X_1$, independently drawn from two distributions $\distrib_0$ and $\distrib_1$ (see, \textit{e.g.}, \cite{heumann2023introduction} for an introduction to hypothesis testing). 
To do so, many statistics have been derived, depending on the assumptions made on $\distrib_0$ and $\distrib_1$.
In the case of distributions supported on $\R$, one such test is the {\em univariate Kolmogorov-Smirnov (KS) test}, of which we make use in this article.
Given, as above, two sets of samples $X_0,X_1\subset \R$, consider the empirical distribution functions $F_i(z) \coloneqq \frac{1}{\text{Card}(X_i)} \sum_{x\in X_i}1_{x\leq z}$ for $i=0,1$ and $z\in\R$, where $\text{Card}$ denotes the cardinality.
Then the statistic associated with the KS test and the samples is $T \coloneqq \sup_{z\in\R}|F_0(z)-F_1(z)|$.
If $\distrib_0 = \distrib_1$, the distribution of $T$ is independent of $\distrib_0$ and converges to a known distribution when the sizes of the samples tend to infinity (under mild assumptions) \cite{smirnov1939estimation}. Hence approximate $p$-values can be derived.
The KS test can also be used to compare multivariate distributions: if $\distrib_0$ and $\distrib_1$ are distributions on $\R^D$, a $p$-value $p_i$ can be computed from the samples by comparing the $i$-th entries of the vectors of $X_0,X_1\subset \R^D$ using the univariate KS test, for $i=1,\dots,D$.
A standard and conservative way of combining those $p$-values is to reject $H_0$ if $\min(p_1,\dots,p_D)<\alpha/D$, where $\alpha$ is the significance level of the test. This is known as the {\em Bonferroni correction} \cite{Voss1073}.
Other tests tackle the multidimensionality of the problem more directly, such as the {\em Maximum Mean Discrepancy (MMD) test}, though not necessarily with greater success (see, \textit{e.g.}, \cite{Ramdas2014OnTD}).

% \paragraph*{Neural networks.}
\subsection{Neural Networks}

We now recall the basics of \emph{neural networks} (NN), which will be our main object of study.
We define a neural network\footnote{While our exposition is restricted to fully-connected feedforward neural networks for the sake of concision, our representations
are well-defined for other types of neural nets (\textit{e.g.}, recurrent neural nets). In particular, they adapt seamlessly to the case of convolutional layers: such a layer can always be represented as a fully-connected layer whose weight matrix is constrained to have many zeroes, and what follows applies without further modifications.} as a (finite) sequence of functions called \emph{layers} $f_1,\dots,f_L$ of the form $f_\ell \colon \R^{n_\ell} \to \R^{n_{\ell+1}}, x\mapsto \sigma_\ell(W_\ell\cdot x + b_\ell)$, where the parameters $W_\ell \in \R^{n_{\ell+1} \times n_{\ell}}$ and $b_\ell \in \R^{n_{\ell+1}}$ are called the weight matrix and the bias vector respectively, and $\sigma_\ell$ is an (element-wise) activation map ({\em e.g.}, sigmoid or ReLU). 
The  neural network encodes a map $F \colon \R^d \to \R^D$ given by $F = f_L \circ \dots \circ f_1$.  We sometimes use $F$ to refer to the neural network as a whole, though it has more structure. 

When the neural network is used as a classifier, the last activation function $\sigma_L$ is often taken to be the \emph{softmax} function, so that $F(x)_i$ can be interpreted as the confidence that the network has in $x$ belonging to the $i$-th class, for $i=1,\dots,D$.
For this reason, we use the terminology \emph{confidence vector} (CV) for the output $F(x)\in \R^D$.
The true class of $x$ is represented by a label $y = (0,\dots,0,1,0,\dots,0) \in \R^D$ that takes value $1$ at the coordinate indicating the correct class and $0$ elsewhere. 
The parameters of each layer $(W_\ell, b_\ell)$ are typically learned from a collection of training observations and labels $\{(x_n, y_n)\}_{n=1}^N$ by minimizing a cross-entropy loss through gradient descent, in order to make $F(x_n)$ as close to $y_n$ as possible on average over the training set. 
The \emph{prediction} of the network on a new observation $x$ is then given by $\argmax_{i=1,\dots,D} F(x)_i$, and its (test) \emph{accuracy} is the proportion of correct predictions on a new set of observations $\{(x'_n, y'_n)\}_{n=1}^{N'}$, that is assumed to have been independently drawn from the same distribution as the training observations.
In this work, we consider NN classifiers that have already been trained on some training data and that achieve reasonable accuracies on test data following the same distribution as training data.

% \paragraph*{Activation graphs.}
\subsection{Activation Graphs}
\label{subsec:activatin_graphs_def}

Given an instance $x = x_0 \in \R^d$, a trained neural network $f_1,\dots,f_L$ with $x_{\ell+1}=f_\ell(x_\ell) = \sigma_\ell(W_\ell \cdot x_\ell + b_\ell)$ and a layer $f_\ell \colon \R^{n_\ell} \longrightarrow \R^{n_{\ell+1}}$, we can define a weighted graph, called the \emph{activation graph} $G_\ell(x)$ of $x$ for the layer $f_\ell$,  as follows.
We let $V \coloneqq V_\ell \sqcup V_{\ell+1}$ be the disjoint union of the two sets $V_\ell = \{1,\ldots, n_\ell \}$ and $V_{\ell+1} = \{1,\ldots, n_{\ell+1}\}$. The edges are defined as $E \coloneqq V_\ell \times V_{\ell+1}$. To each edge $(i,j) \in E_\ell$, we associate the weight $w_{i,j}(x)\coloneqq W_\ell(j,i) \cdot x_\ell(i)$, where $x_\ell(i)$ (resp.\ $W_\ell(j,i)$) denotes the $i$-th coordinate of $x_\ell \in \R^{n_\ell}$ (resp.\ entry $(j,i)$ of $W_\ell \in \R^{n_{\ell+1} \times n_\ell}$).
The activation graph $G_\ell(x)$ is the weighted graph $(V,E,\{w_{i,j}(x)\})$, which can be conveniently represented as a $n_\ell\times n_{\ell+1} $ matrix whose entry $(i,j)$ is $w_{i,j}(x)$. A simple illustration of this definition can be found in the Supplementary Material.
Intuitively, these activation graphs---first considered in \cite{gebhart2019characterizing}---represent how the network ``reacts'' to a given observation $x$ at inner-level, rather than only considering the network output ({\em i.e.}, the Confidence Vector).

%Different features can be extracted from such graphs, such as the distribution of weights of its Maximum Spanning Tree (which is related to topological properties of this graph, see \cite{Lacombe2021}).

\section{Two-Sample Statistical Tests using {\statname{}}}\label{sec:stat}

\subsection{The \statname{} representations}

Let  $\distrib_0$ and $\distrib_1$ be two distributions for which we want to test $H_0 : \distrib_0 = \distrib_1$.
As mentioned above, two-sample statistical tests tend to underperform when used directly on high-dimensional data.
It is thus common practice to extract lower-dimensional representations $\Psi(x)$ from the data\footnote{Here, as in the remainder of the article, we commit a minor abuse of notation: $\distrib_0$ and $\distrib_1$ are distributions on both the features and the labels, \textit{i.e.} $(x,y) \sim \distrib_i$, but we often write $x\sim \distrib_i$ to indicate that $x$ has been drawn from $(\distrib_i)_x$, the marginal of $\distrib_i$ with respect to the features. To avoid any confusion, we always let the letter $x$ (possibly with a subscript) indicate features. } $x \sim \distrib_i$, where $\Psi \colon  {\rm supp}~\distrib_0 \cup {\rm supp}~\distrib_1 \rightarrow \R^M$.
Given a classification task with classes $1,\dots,D$, we define a family of such representations as follows.
Let $T \colon {\rm supp}~\distrib_0 \cup {\rm supp}~\distrib_1 \rightarrow V $ be any map whose codomain $V$ is a Banach space with norm $\|\cdot\|_V$. For each class $i\in \{1,\dots,D\}$, let $\distrib_{0,i}$ be the conditional distribution of data points from $\distrib_0$ in class $i$.
We define
\[\Psi_i(x) \coloneqq \| T(x) - \expect_{\distrib_{0,i}}[T(x')] \|_V\]
for $x\in  {\rm supp}~\distrib_0 \cup {\rm supp}~\distrib_1$.
Given a fixed finite dataset $x_1, \dots,x_m \stackrel{\mathrm{iid}}{\sim} \distrib_0$, we similarly define the approximation 
\vspace{-2.5pt}
\[\tilde{\Psi}_i(x) \coloneqq \| T(x) - \frac{1}{m_i}\sum_{j=1}^{m_i} T(x^i_j) \|_V,\]

\vspace{-10pt}
where $x^i _1,\dots,x^i_{m_i}$ are the points whose class is $i$. 
This defines a map $\tilde{\Psi} \colon {\rm supp}~\distrib_0 \cup {\rm supp}~\distrib_1 \rightarrow \R^D$.

The map $T \colon {\rm supp}~\distrib_0 \cup {\rm supp}~\distrib_1 \rightarrow V $ could \textit{a priori} take many shapes.
In this article, we assume that we are provided with a neural network $F$ that has been trained for the classifying task at hand,
as well as a training set drawn from $\distrib_0$.
We let $T$ be the activation graph $G_\ell$ of the layer $f_\ell$ of $F$ represented as a matrix, so that the expected values $\expect_{\distrib_{0,i}}[G_\ell(x')]$ (for $i=1,\dots,D$) are simply mean matrices, and the norm $\|\cdot\|_V$ is the Frobenius norm $\|\cdot\|_2$.
We call the resulting $D$-dimensional representation \textit{Mean Activation Graph Difference} (\statname{}): 
\[\statname{}(x)_i \coloneqq \| G_\ell(x) - \frac{1}{m_i}\sum_{j=1}^{m_i} G_\ell(x^i_j) \|_2,\]

\vspace{-10pt}
for $i=1,\dots,D$, where $x^i _1,\dots,x^i_{m_i}$ are, as above, samples of the training set whose class is $i$.
Therefore, for a given new observation $x$, we derive a vector $\statname{}(x)\in \R^D$ whose $i$-th coordinate indicates whether $x$ activates the chosen layer of the network in a similar way ``as training observations of the class $i$''.

Many variations are possible within that framework.
One could, \textit{e.g.},~consider the activation graph of several consecutive layers, use another matrix norm, or apply Topological Data Analysis techniques to compute a more compact representation of the graphs,
such as the {\em topological uncertainty}~\cite{Lacombe2021}.
In this work, we focus on \statname{} for dense layers, though it could be extended to other types.

\subsection{Comparison of distributions of features with multiple KS tests}
Given as above a (relatively low-dimensional) representation $\Psi \colon {\rm supp}~\distrib_0 \cup {\rm supp}~\distrib_1 \rightarrow \R^N $ and samples $x_1,\dots,x_n \stackrel{\mathrm{iid}}{\sim} \distrib_0$ and $x'_1,\dots,x'_m \stackrel{\mathrm{iid}}{\sim} \distrib_1$, one can apply multiple univariate (coordinate-wise) KS tests with Bonferroni correction to the sets $\Psi(x_1),\dots,\Psi(x_n)$ and $\Psi(x'_1),\dots,\Psi(x'_m)$, as described in Section \ref{sec:background}.
If $\Psi$ is well-chosen, %one can hope that
a difference between the distributions $\distrib_0$ and $\distrib_1$ (hard to test directly due to the dimensionality of the data) will translate to a difference between the distributions $\Psi(x)$ and $\Psi(x')$ for $x\sim \distrib_0$ and $x'\sim\distrib_1$ respectively.
Detecting such a difference serves as a proxy for testing $H_0:\distrib_0=\distrib_1$.
In our experiments, we apply this procedure to the   $\statname{}$ representations defined above (see Section~\ref{subsec:expes_settings} for a step-by-step description). This is a reasonable approach, as it is a simple fact that a generic shift in the distribution of the random variable $x\sim \distrib_0$ will in turn induce a shift in the distribution of $\Psi(x)$, as long as $\Psi$ is not constant\footnote{See the Supplementary Material, Section 3 for an elementary proof.}; however, this does not give us any true guarantee, as it does not provide any quantitative result regarding the shift in the distribution of $\Psi(x)$. Such results are beyond the scope of this paper, in which we focus on the good experimental performance of the \statname{}~statistic.

\subsection{Differences from BBSD and motivations}
%\yuichi{Do we need this part as a subsection? We could write the following as a part (maybe with a paragraph environment) of Subsection~4.2. The word "comparison" is used in a different way from the previous section, which is a bit confusing.}
%\charles{I agree with you regarding the title. On the other hand, I like that with subsection 4.2, we complete the description of what we do, then move on to what others have done}
%\charles{Remark : we are once again explaining that CV means confidence vectors. How many times do we want to do that ? (this is not a rhetorical question)}
%\felix{I think it's good to repeat it, this may help readers that haven't read all previous sections in detail. Once we make the final passes over the text and feel that it occurs too often, we can still remove some instances of it.}
%\theo{Agree with Felix + it's quite typical that people go very quickly in Intro/Background to focus on experiment, or the otherway around. So it's ok to define it twice imho}. 

The BBSD method described in \cite{Lipton2018} and \cite{Rabanser2019} is defined in a similar manner, except that the representations $\Psi$ on which the multiple univariate KS tests are applied are simply the Confidence Vectors (CV) $F(x)\in \R^D$ of the neural network $F$ (or of any other classifier that outputs confidence vectors), rather than our newly proposed \statname{} representations. In other words, they detect shifts in the distribution of the inputs  by testing for shifts in the distribution of the outputs of a given classifier\footnote{This corresponds to the best-performing variant of their method, denoted as \emph{BBSDs} (as opposed to, \textit{e.g.},~\emph{BBSDh}) in \cite{Rabanser2019}.}.

Both our method and theirs share advantages: the features are task- and data-driven, as they are derived from a classifier that was trained for the specific task at hand. They do not require the design or the training of an additional model specifically geared towards shift detection, and they have favorable algorithmic complexity, especially compared to some kernel-based methods. In particular, combining the KS tests with the Bonferroni correction spares us from having to calibrate our statistical tests with a permutation test, which can be costly as shown in \cite{Rabanser2019}.
A common downside is that the Bonferroni correction can be overly conservative; other tests might offer higher power.
The main focus of this article is the relevance of the \statname{}~representations, rather than the statistical tests that we apply to them, and it has been shown in \cite{Rabanser2019} that KS tests % are powerful enough to
yield state-of-the-art performances; as such, we did not investigate alternatives, though additional efforts in that direction might produce even better results.

%As we chose to focus on the relevance of the \statname{}~features, we did not investigate alternatives.

The nature of the construction of the \statname{}~representations is geared towards shift detection since it is directly based on encoding differences (\textit{i.e.},~deviations) from the mean activation graphs (of $\distrib_0$).
 Moreover, they are based on representations from deeper within the NN, which are less compressed than the CV - passing through each layer leads to a potential loss of information. Hence, we can hope for the \statname{}~ to encode more information from the input data than the CV representations used in \cite{Rabanser2019} which focus on the class to which a sample belongs to, while sharing the same favorable dimensionality reduction properties. 
%The \statname{}~features are in a sense closer to the input than the CVs, as they are derived from deeper within the NN.% yet they share the same dimensionality as the CVs. % and more compact and processed than the input data.
%We thus hypothesize \statname{}~to capture more pertinent information for the problem of shift detection than the , while sharing the same favorable dimensionality reduction properties.
%but concisely enough that shifts in their distributions remain detectable by our statistical tests.
Therefore, we expect \statname{}~to perform particularly well with covariate shifts, where shifts in the distribution of the data do not necessarily translate to strong shifts in the distribution of the CV. Conversely, we do not hope for our representations to bring significant improvements over CV in the specific case of label shifts; all the information relative to labels available to the network is, in a sense, best summarized in the CV, as this is the main task of the NN.
These expectations were confirmed in our experiments.

\section{Experiments}\label{sec:expes}

This experimental section is devoted to showcasing the use of the \statname{}~representations and its benefits over the well-established baseline CV when it comes to performing covariate shift detection.
As detailed in Section~\ref{subsec:expes_settings}, we combine coordinate-wise KS tests for both these representations.
Note that in the case of CV, this corresponds exactly to the method termed \emph{BBSDs} in \cite{Rabanser2019}.
%\par
Our code is provided in the Supplementary Material, as well as a more thorough presentation of the datasets and parameters used.

\subsection{Experimental Settings}\label{subsec:expes_settings}

%To illustrate the versatility of our approach, we ran our experiments on various datasets, for various network architectures, various type of shifts along with various levels of intensity. 

\paragraph{Datasets.} We consider the standard datasets MNIST \cite{MNIST}, FashionMNIST (FMNIST) \cite{xiao2017fashionmnist}, CIFAR-10 \cite{CIFAR-10}, SVHN \cite{SVHN}, as well as a lighter version of ImageNet (restricted to $10$ classes) called Imagenette \cite{Imagenette}.

\paragraph{Architectures.} For MNIST and FMNIST, we used a simple CNN architecture consisting of 3 convolutional layers followed by 4 dense layers. 
For CIFAR-10 and SVHN, we considered (a slight modification, to account for input images of size $32\times32$, of) the ResNet18 architecture \cite{Resnet}. For Imagenette, we used a pretrained ResNet18 model provided by Pytorch \cite{ResNet18_pretrained}. 
With these architectures, we reached a test accuracy of $98.6\%$ on MNIST, $91.1\%$ on FMNIST, $94.1\%$ on SVHN, $81\%$ on CIFAR-10 and $99.2\%$ for Imagenette, validating the ``well-trained'' assumption mentioned in Section \ref{sec:stat}.
%\footnote{Further details on the architectures and training procedure can be found in the code associated to this paper.}. 
Note that we used simple architectures, without requiring the networks to achieve state-of-the-art accuracy.

\paragraph{Shifts.} We applied three types of shift to our datasets: Gaussian noise (additive white noise), Gaussian blur (convolution by a Gaussian distribution), and Image shift (random combination of rotation, translation, zoom and shear),
%, and the Knock-out shift (KO) (deletion of a fraction of samples of class 0),
for six different levels of increasing intensities (denoted by I, II,\dots,VI), and a fraction of shifted data $\delta \in \{0.25, 0.5, 1.0\}$. 
For each dataset and shift type, we chose the shift intensities in such a manner that the shift detection for the lowest intensities and low $\delta$ is almost indetectable for both methods (\statname{} and CV), and very easily detectable for high intensities and values of $\delta$. Details (including the impact of the shifts on model accuracy) and illustrations can be found in the Supplementary Material.
%%%%%Commented for submission:

% %\charles{Reference to supplementary material for precise description of each shift and shift intensity}
% %\charles{I have removed the description of the k0 shift}

\paragraph{Sample size.} 
We ran the shift detection tests with sample sizes\footnote{That is, the number of elements from the clean and shifted sets on which the statistical tests are performed; see the paragraph \textbf{Experimental protocol} for more details.} $\{10, 20, 50, 100, 200, 500, 1000\}$ to assess how many samples a given method requires to reliably detect a distribution shift. A good method should be able to detect a shift with as few samples as possible.  

\paragraph{Experimental protocol.}
In all of the experiments below, we start with a neural network that is pre-trained on the training set of a given dataset. The test set will be referred to as the \textit{clean set} (CS). 
We then apply the selected shift (type, intensity, and proportion $\delta$) to the clean set and call the resulting set the \textit{shifted set} $SS$; it represents the target distribution $\distrib_1$ in the case where $\distrib_1\neq \distrib_0$.

As explained in Section~\ref{sec:stat}, for each of the classes $i=1,\dots,D$ (for all of our datasets, $D=10$), we compute the mean activation graph of a chosen dense layer $f_\ell$ of (a random subset of size $1000$ of all) samples in the training set whose class is $i$; this yields $D$ mean activation graphs $G_1,\dots,G_D$.
We compute for each sample $x$ in $CS$ and each sample in $SS$ the representation $\statname{}(x)$, where $\statname{}(x)_i = \|G_\ell(x)-G_i\|_2$ for $i=1,\dots,D$ and $G_\ell(x)$ is the activation graph  of $x$ for the layer $f_\ell$ (as explained in Section \ref{sec:stat}).
Doing so, we obtain two sets $\{\statname{}(x) \mid x\in CS\} $ and $\{\statname{}(x') \mid x'\in SS\} $ of $D$-dimensional features with the same cardinality as the test set. 

Now, we estimate the power of the test for a given sample size\footnote{The same sample size that is mentioned in the \textbf{Sample size} paragraph.} $m$ and for a type I error of at most $0.05$; in other words, the probability that the test rejects $H_0$ when $H_1$ is true and when it has access to only $m$ samples from the respective datasets, and under the constraint that it does not falsely reject $H_0$ in more than $5\%$ of cases.
To do so, we randomly sample (with replacement) $m$ elements $x_1', \dots, x_m'$ from $SS$, and consider for each class $i=1,\dots,D$ the discrete empirical univariate distribution $q_i$ of the values $\statname{}(x'_1)_i,\dots,\statname{}(x'_m)_i$.
Similarly, by randomly sampling $m$ elements from $CS$, we obtain another discrete univariate distribution $p_i$ (see Figure \ref{fig:comparison_distributions} for an illustration).
Then, for each $i=1,\dots,D$, the KS test is used to compare $p_i$ and $q_i$ to obtain a $p$-value $\lambda_i$, and reject $H_0$ if $\min(\lambda_1,\dots,\lambda_D)<\alpha/D$, where $\alpha$ is the threshold for the univariate KS test at confidence $0.05$ (\textit{cf.}~Section~\ref{subsec:shift-detection}).
Following standard bootstrapping protocol, we repeat that experiment (independently sampling $m$ points from $CS$ and $SS$, computing $p$-values, and possibly rejecting $H_0$) $1500$ times; the percentage of rejection of $H_0$ is the estimated \emph{power} of the statistical test (since $H_0$ is false in this scenario). 
We use the asymptotic normal distribution of the standard Central Limit Theorem to compute approximate $95\%$-confidence intervals on our estimate.

% \vspace{-5.5pt}
\begin{figure}[h]
\begin{center}
\includegraphics[width=0.52\textwidth]{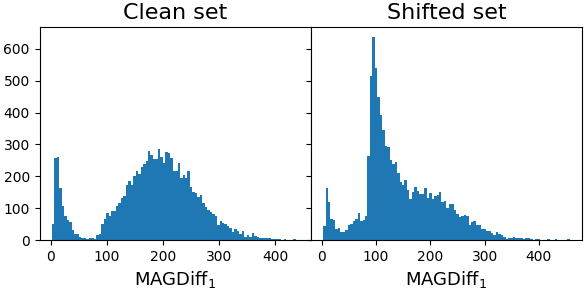}
% \vspace{-8pt}
\captionof{figure}{Empirical distributions of $\statname{}_1$ for the $10,000$ samples of the clean and shifted sets (MNIST, Gaussian noise, $\delta=0.5$, last dense layer). For the clean set, the distribution of the component $\statname{}{}_1$ of \statname{} exhibits a peak close to $0$.
This corresponds to those samples whose distance to the mean activation graph of (training) samples belonging to the associated class is very small, \textit{i.e.},~these are samples that presumably belong to the same class as well.
Note that, for the shifted set, this peak close to $0$ is substantially  diminished,
%(recall that $\delta=0.5$),
which indicates that the activation graph of samples affected by the shift is no longer as close to the mean activation graph of their true class.}
\label{fig:comparison_distributions}
\vspace{-3mm}
\end{center}
\end{figure}
%\vspace{-10.5pt}
To illustrate that the test is well calibrated, we repeat the same procedure while sampling twice $m$ elements from $CS$ (rather than $m$ elements from $SS$ and $m$ elements from $CS$), which allows us to estimate the type I error (\textit{i.e.},~the percentage of incorrect rejections of $H_0$) and assert that it remains below the significance level of $5\%$ (see, \textit{e.g.},~Figure~\ref{fig:Experiment1}).

We experimented with a few variants of the \statname{}~representations: we tried reordering the coordinates of each vector $\statname{}(x)\in\R^D $ in increasing order of the value of the associated confidence vectors. We also tried replacing the matrix norm of the difference to the mean activation graph by either its Topological Uncertainty (TU) \cite{Lacombe2021}, or variants thereof.
Early analysis suggested that these variations did not bring increased performances, despite their increased complexity. Experiments also suggested that \statname{}~representations brought no improvement over CV in the case of label shift.
We also tried to combine (\textit{i.e.}, concatenate) the CV and \statname{}~representations, but the results were unimpressive, which we attribute to the Bonferroni correction getting more conservative the higher the dimension. 
We thus only report the results for the standard \statname{}.

% \remark{
\begin{comment}
\par
In Figure~\ref{fig:comparison_distributions}, for the clean set, the distribution of the first component $\statname{}{}_1$ of \statname{} exhibits a peak close to $0$.
This corresponds to those samples whose distance to the mean activation graph of (training) samples belonging to the first class, is very small, \textit{i.e.},~these are samples that presumably belong to the same class as well.
Note that, for the shifted set, this peak close to $0$ is substantially  diminished (recall that $\delta=0.5$), which indicates that the activation graph of samples affected by the shift is no longer as close to the mean activation of their true class.
% }
\end{comment}

\paragraph{Competitor.} 
We used multiple univariate KS tests applied to CV (the method BBSDs from \cite{Rabanser2019}) as the baseline, which we denote by ``CV'' in the figures and tables, in contrast to our method denoted by ``\statname{}''.
The similarity in the statistical testing between BBSDs and \statname{} allows us to easily assess the relevance of the \statname{}~features.  We chose them as our sole competitors as it has been convincingly shown in \cite{Rabanser2019} that they outperform on average all other standard methods, including the use of dedicated dimensionality reduction models, such as autoencoders, or of  multivariate kernel tests. Many of these methods are also either computationally more costly (to the point where they cannot be practically applied to more than a thousand samples) or harder to implement (as they require an additional neural network to be implemented) than both BBSDs and  \statname{}.

\begin{comment}
We also recall that our method can be applied with different types of \todo{how do we call the function we consider?}. 
We thus also considered the Topological Uncertainty (TU) method for shift detection \cite{Lacombe2021}, and a variation of it that we called \emph{Topological Difference} (TD). 
However, it happens that the simple Matrix Norm (MN) approach substantially outperforms both TU and TD, while being more efficient to compute.  \charles{Maybe move references to TD and TU to the exposition section \ref{sec:stat} ?}
\felix{Yes, I think that would be better, since we don't provide the results for TU/TD.}
Therefore, we decided for the sake of concision to only report results for BBSD-CV and BBSD-MN in the main body, and defer complementary details to the supplementary material.  \charles{(depends on whether we actually do it)}
\end{comment}
%\begin{remark} We report in the main paper the results that we believe are the most enlightening, either on some specific instances (e.g.~MNIST with Gaussian noise of intensity II with $\delta=0.5$) or aggregated on different parameters (e.g.~averaging the performances over datasets or noise types) to summarize that, overall, our method outperforms its competitor CV. 
%We also considered using a combination of BBSD-MN and BBSD-CV, but this yielded worse results than BBSD-MN alone. 
%Comprehensive experimental report can be found in the supplementary material. 
%\end{remark}
%\charles{Maybe remove part of the remark ? It is implicitly understood that we pick the results we consider the most relevant}

\subsection{Experimental results and influence of parameters.}
% \charles{Important : I have added titles to the figures and tables (very easy to delete) - should we keep them, or are they too unrigorous/restrictive ? I thought they might help the reader immediately understand what the figure is about (especially considering our graphs look kinda similar to each other). }\mathieu{I like it. In order to save space (if needed), you can also put the title in the caption.}

We now showcase the power of %the statistical test for 
shift detection using our \statname{}{} representations in various settings and compare it to the state-of-the-art competitor CV. 
%
%Though we try to make as much information as possible available to the reader, 
Since there were a large number of hyper-parameters in our experiments (datasets, shift types, shift intensities, etc.),
%means that we cannot display the results for all combinations of datasets, shift types, shift intensities, etc. 
we started with a standard set of hyper-parameters that yielded representative and informative results according to our observations (MNIST and Gaussian noise, as in \cite{Rabanser2019}, $\delta=0.5$, sample size $100$, \statname{} computed with the last layer of the network) and let some of them vary in the successive experiments.
We focus on the well-known MNIST dataset to allow for easy comparisons, and refer to the Supplementary Material for additional experimental results that confirm our findings on other datasets.

%We pick as ``reference'' the MNIST dataset, the Gaussian noise as shift (following \cite{Rabanser2019}) with intensity IV, a sample size of $100$, a shift proportion of $\delta=0.5$ and the \statname{} features corresponding to the last dense layer of the CNN architecture (denoted by layer $\ell_{-1}$), for which we will show detailed results.

% \charles{I have rewritten the paragraph about the "reference" a bit, tell me what you think. I still don't think it's strictly necessary, and don't mind if it's deleted}

 %In addition to that, by averaging the power over experiments with varying parameters, we show that the power of the statistical test using \statname{} significantly outperforms the CV baseline in a reliable way.

\paragraph{Sample size.} The first experiment consists of estimating the power of the shift detection test as a function of the sample size (a common way of measuring the performance of such a test) using either the \statname{} or the baseline CV representations. 
% What truly matters then, is to assess which method is the most capable of detecting shift of intermediate intensity and intermediate sample size. 
Figure~\ref{fig:Experiment1} shows the powers of the KS tests using the \statname{} (red curve) and CV (green curve) representations with respect to the sample size for the MNIST dataset.
Here, we choose to showcase the results for Gaussian noise of intensities II, IV and IV with shift proportion $\delta=0.5$.
% It can be clearly seen that, for low shift intensity, the shift can be very clearly detected with \statname{}, even for relatively small sample size of $100$, while for BBSD-CV it cannot be detected even for the highest sample size of $1000$.
% In the case for high shift intensity and sample size (in which the shift should be easily detectable) we observe that the power reaches $1.0$ for both features.
% In all cases in between, the results for \statname{} significantly outperform the BBSD-CV baseline and the shift can be detected much earlier, both in terms of shift intensity as well as sample size.

%%% VERSION MULTIPLE FIGURES
\begin{figure*}[t]
    \centering
    \textbf{Power as a Function of Sample Size}\par
    \includegraphics[width=0.32\textwidth]{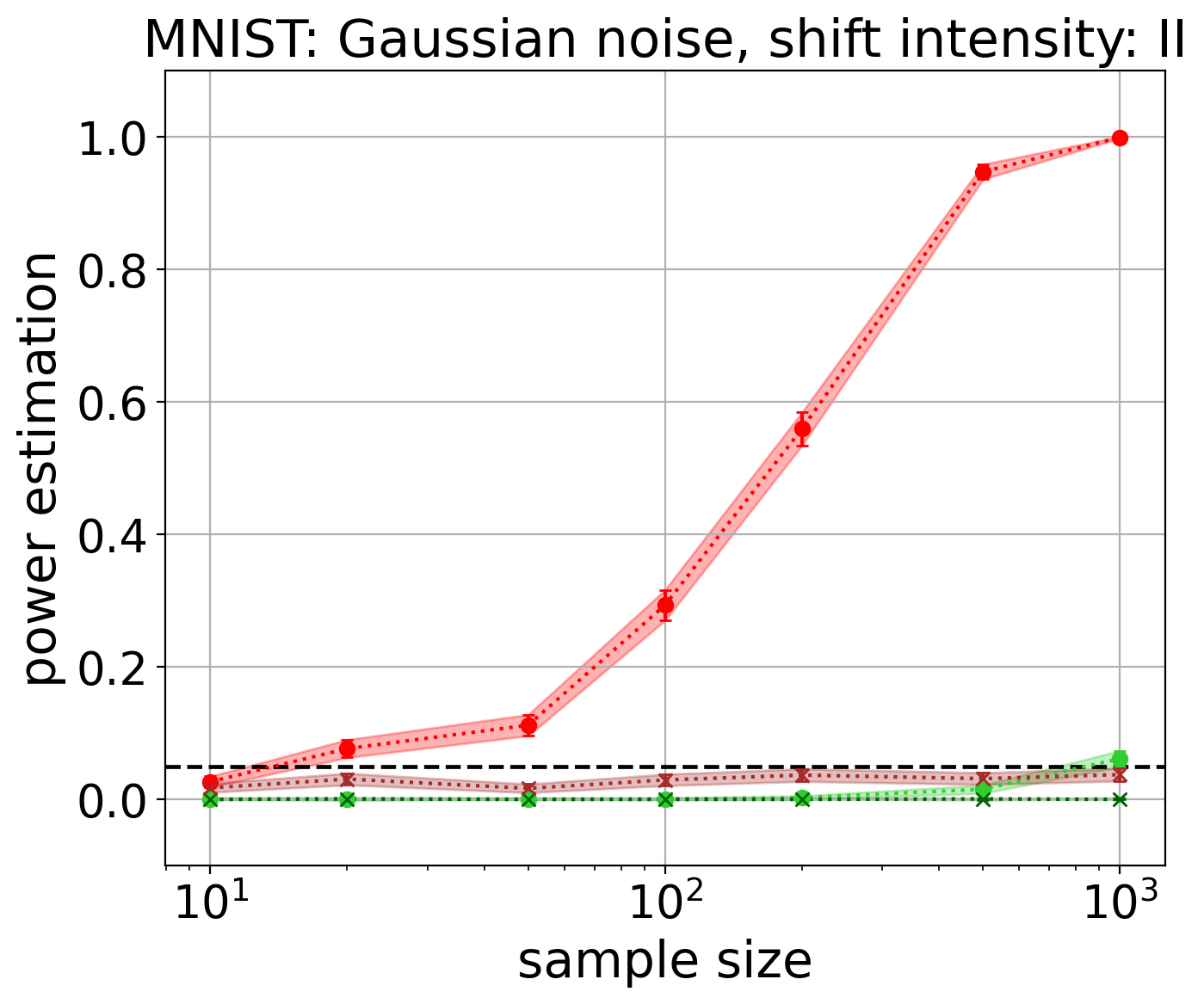}
    \includegraphics[width=0.32\textwidth]{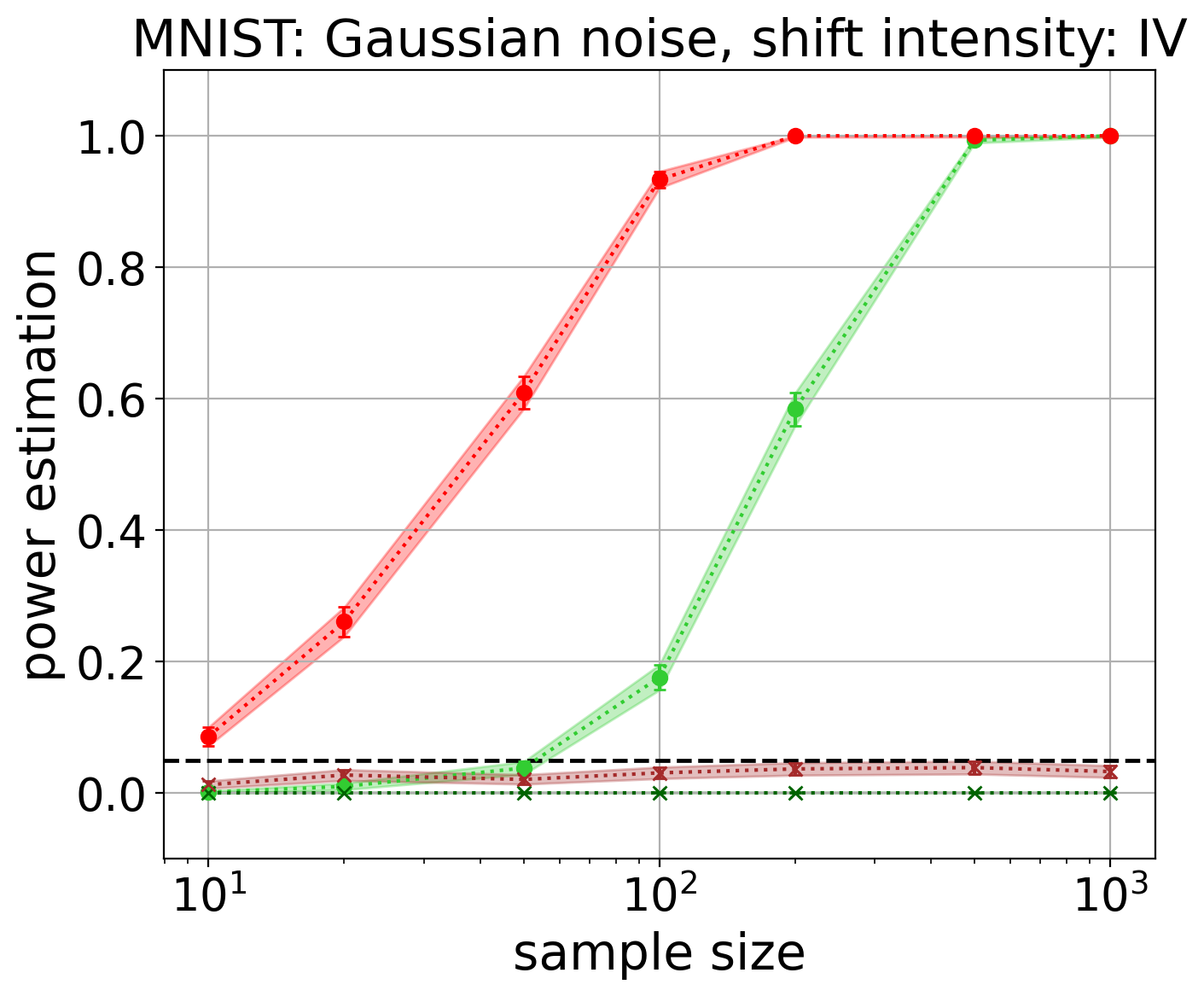}
   \includegraphics[width=0.32\textwidth]{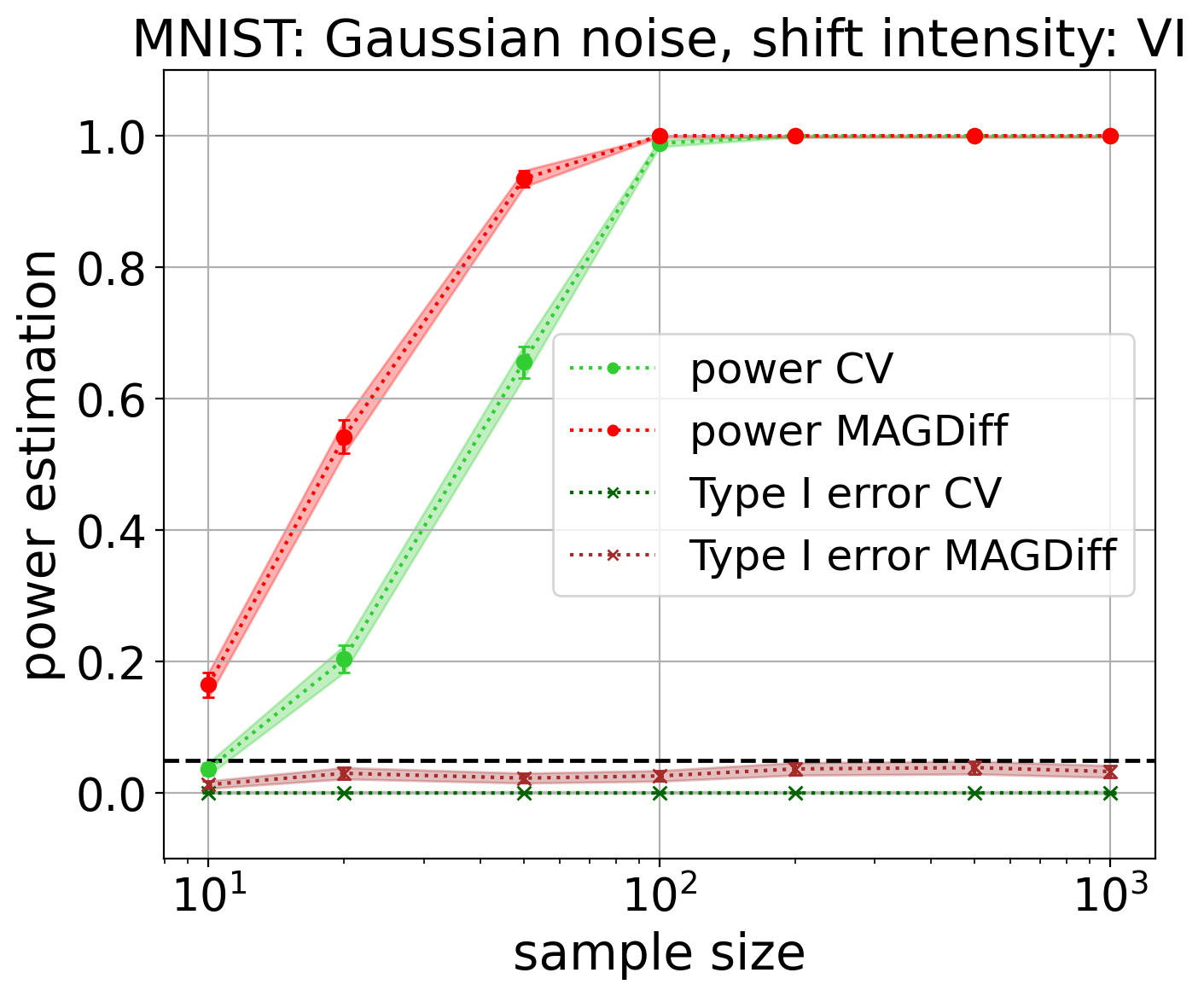}
    \caption{Power and type I error of the statistical test with \statname{}~(red) and CV (green) representations  w.r.t.~sample size (on a log-scale) for three different shift intensities (II, IV, VI) and fixed $\delta = 0.5$ for the MNIST dataset, Gaussian noise and last layer of the network, with 
    estimated $95\%$-confidence intervals. %are displayed around the curves.
    }
    \vspace{-3mm}
     \label{fig:Experiment1}
\end{figure*}
%\vspace{-3mm}

%%% VERSION SINGLE FIGURE
% \begin{figure}[ht]
%     \centering

%     \includegraphics[width=\columnwidth]{./figures/Experiments/data_shift_power_BBSD__gaussian_noise_IV_weights-epoch=49_delta=0.5.png}

%     \caption{\todo{Should we keep intensity II instead? Indeed the result is pretty striking.}}
%      \label{fig:Experiment1}
% \end{figure}

% As one can see on Figure~\ref{fig:Experiment1},
It can clearly be seen that  \statname{} consistently and significantly outperformed the CV representations.
While in both cases, the tests achieved a power of $1.0$ for large sample sizes ($m\approx 1000$) and/or high shift intensity (VI), \statname{} was capable of detecting the shift even with much lower sample sizes.
This was particularly striking for the low intensity level II, where the test with CV was completely unable to detect the shift, even with the largest sample size, while \statname{} was capable of reaching non-trivial power already for a medium sample size of $100$ and exceptional power for large sample size.
Note that the tests were always well-calibrated.
That is, the type I error remained below the significance level of $0.05$, indicated by the horizontal dashed black line in the figures.\par

To further support our claim that \statname{} outperforms CV on average in other scenarios, we provide, in Table \ref{table:Experiment1.5}, averaged results over all parameters except the sample size. 
Though the precise values obtained are not particularly informative (due to the aggregation over very different sets of hyper-parameters), the comparison between the two rows remains relevant. %\par
%, showcasing that (BBSD-MN outperforms BBSD-CV) on average.
In the Supplementary Material, a more comprehensive experimental report (including, in particular, the CIFAR-10 and Imagenette datasets) further supports our claims.
%We refer to the supplementary material for a more comprehensive experimental report which provides further detailed results (in particular for the CIFAR-10 dataset).
%, showing that the conclusion holds for most settings.
%; with FMNIST \todo{?} being the least striking one and \todo{??} being the most convincing one \todo{is this sentence necessary? It shows our honesty, but may be blunt.}. 
% \charles{+ say  out loud results are comparable on the CIFAR-10 dataset ?}
% \felix{Sure. I'm just slightly concerned that it may make it sound like the results on other dataset are not comparable? (which is the case for FMNIST)}
% \charles{Are you happy with the current formulation (which I have changed a bit) ? (if so, delete these remarks)}

\begin{table}[ht]
\begin{center}
%\input{./tables/aggregated_results_subsample_size.tex}
%\begin{comment}
\setlength{\tabcolsep}{0.45em}
\setlength{\tabcolsep}{0.45em}
\begin{tabular}{lrrrrrrr}
\multicolumn{8}{c}{Averaged power (\%)}\\
\hline
Sample size  & $10$ & $20$ & $50$ & $100$ & $200$ & $500$ & $1000$  \\
\hhline{========}
\statname &  $7.4$ & $17.1$ & $27.6$ & $40.7$ & $54.7$ & $71.4$ & $80.4$ \\
CV & $4.0$ & $9.8$ & $15.6$ & $24.7$ & $35.3$ & $49.7$ & $59.2$ \\
\hline
\end{tabular}
%\end{comment}
\begin{comment}
\setlength{\tabcolsep}{0.45em}
\setlength{\tabcolsep}{0.45em}
\begin{tabular}{lrrrrrrr}
\hline
\multicolumn{8}{c}{Averaged power (\%)}\\
\hline
 \multirow{2}{*}{Features}  & \multicolumn{7}{c}{Sample size}\\
\hhline{~-------}
  & $10$ & $20$ & $50$ & $100$ & $200$ & $500$ & $1000$  \\
\hhline{========}
\statname &  $7.4$ & $17.1$ & $27.6$ & $40.7$ & $54.7$ & $71.4$ & $80.4$ \\
CV & $4.0$ & $9.8$ & $15.6$ & $24.7$ & $35.3$ & $49.7$ & $59.2$ \\
\hline
\end{tabular}
\end{comment}
\vspace{3mm}
\caption{
% \charles{I have simplified the layout of the table to remove one line and save some space (old version commented out)}
Averaged test power of \statname{}{} and CV over all hyper-parameters except sample size (dataset, shift type, $\delta$, shift intensity). A $95\%$-confidence interval for the averaged powers has been estimated via bootstrapping and is, in all cases, strictly contained in a $\pm 0.1\%$ interval. }
\label{table:Experiment1.5}
\end{center}
\end{table}
\vspace{-8mm}

\paragraph{Shift intensity.}
%%% VERSION MULTIPLE FIGURES: different deltas:
\begin{figure*}[ht]
    \centering
    \textbf{Impact of Shift Intensity}\par
    \includegraphics[width=0.32\textwidth]{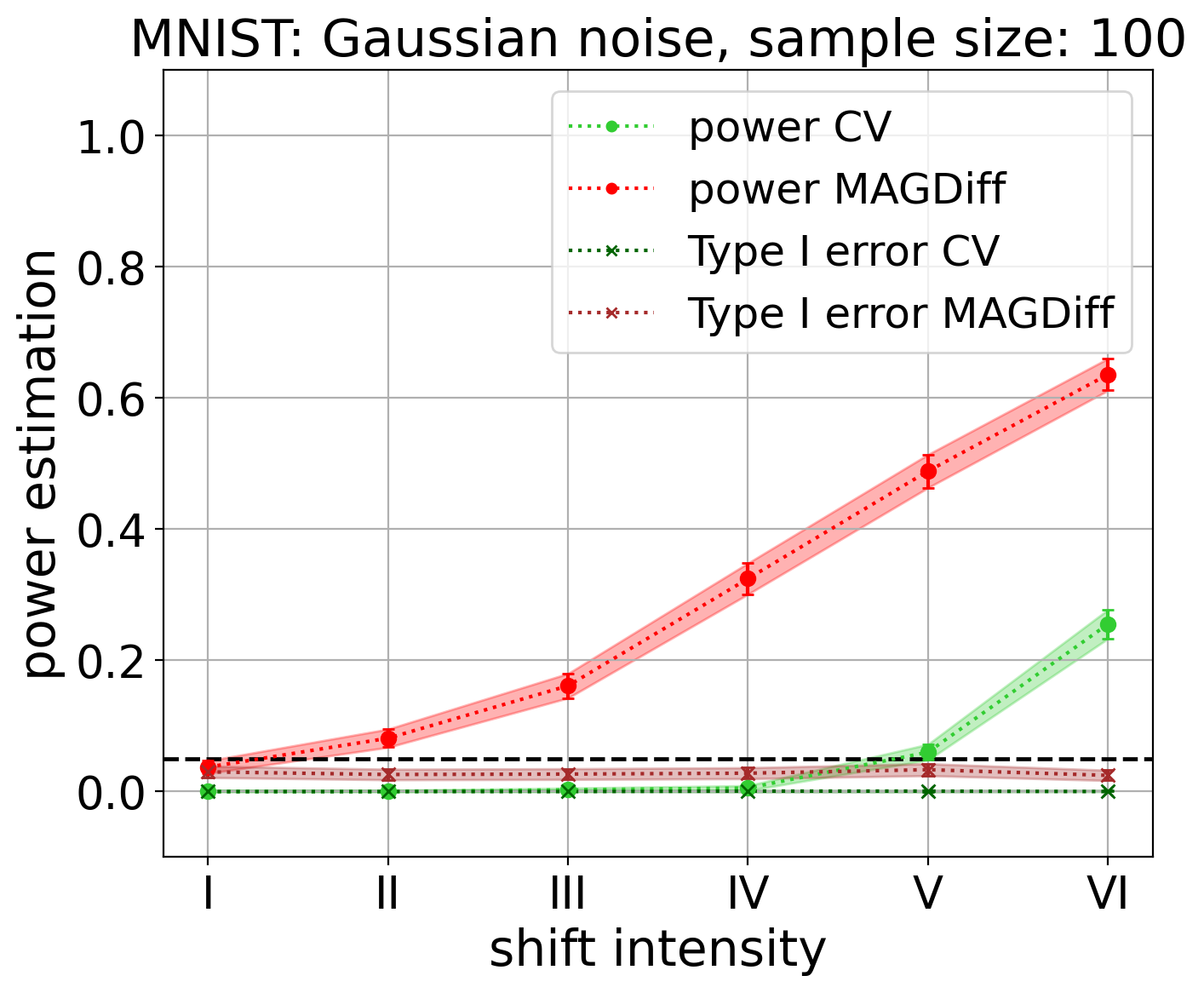}
    \includegraphics[width=0.32\textwidth]{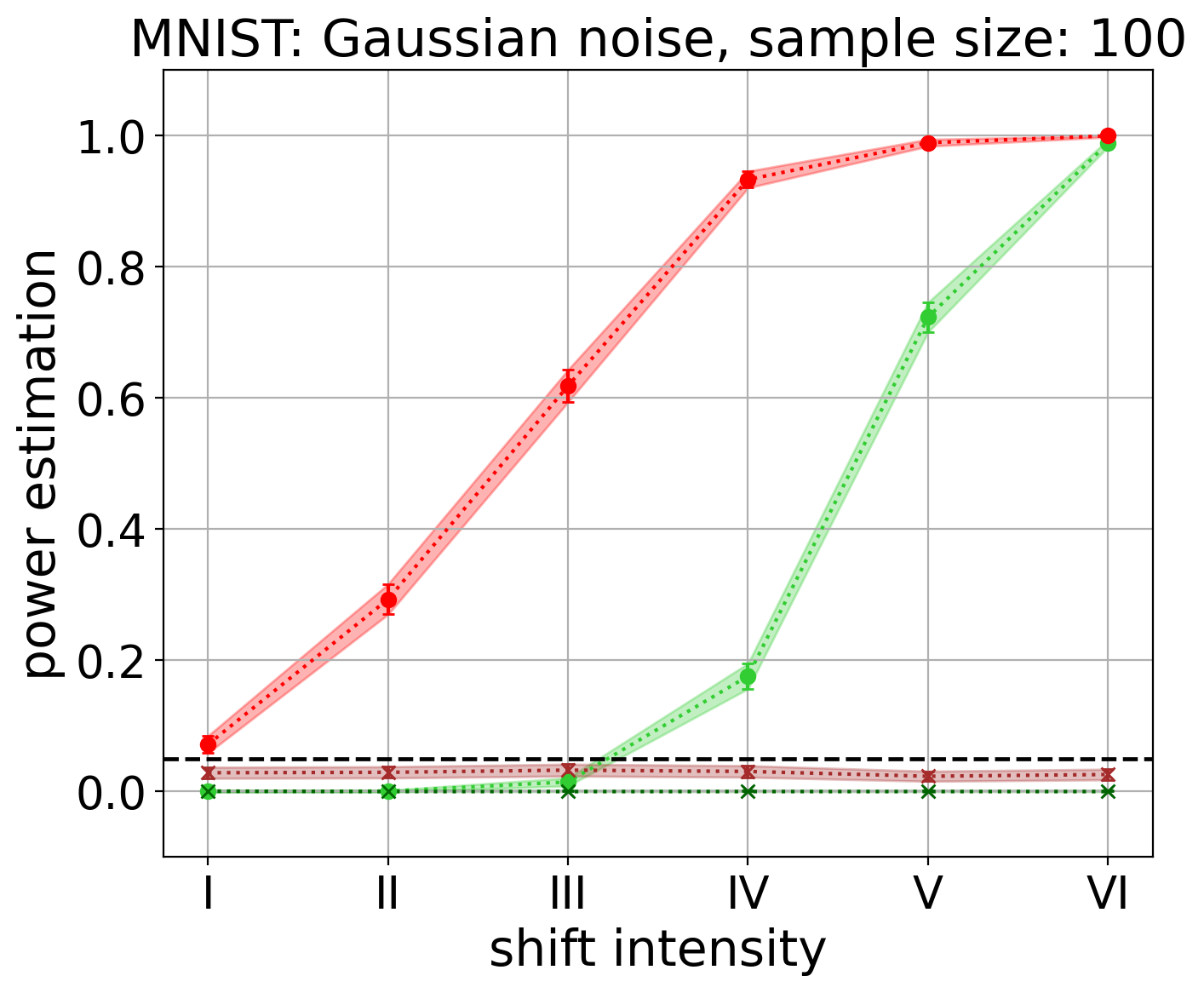}
   \includegraphics[width=0.32\textwidth]{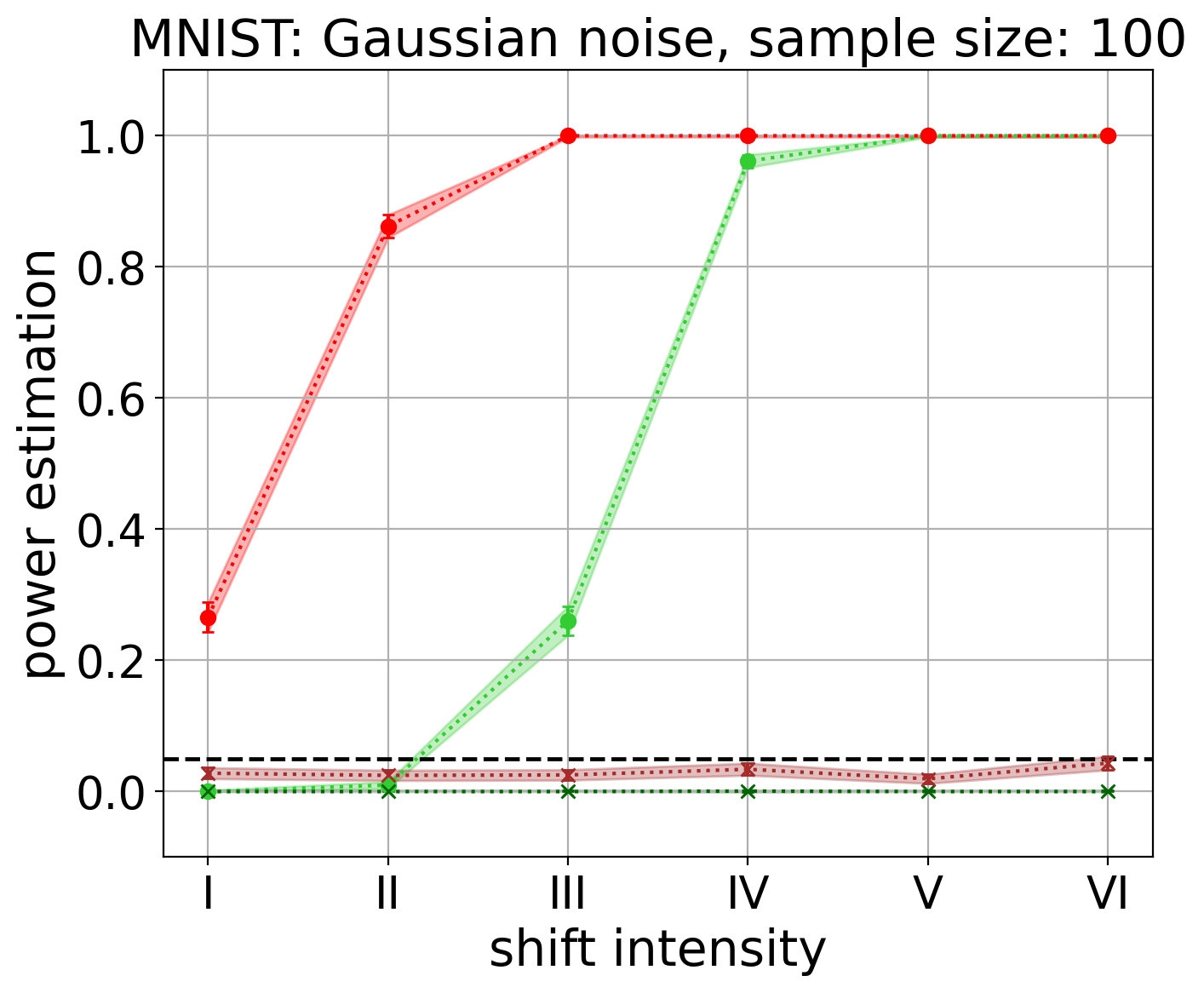}
    \caption{Power and type I error of the test with \statname{}~(red) and CV (green) features w.r.t.~the shift intensity for Gaussian noise on the MNIST dataset with sample size $100$ and $\delta = 0.25$ (left), $\delta = 0.5$ (middle), $\delta = 1.0$ (right), for the last dense layer.
    The estimated $95\%$-confidence intervals are displayed around the curves.}
     \label{fig:shift_intensity}
\end{figure*}

The first experiment suggests that \statname{} representations perform particularly well when the shift is hard to detect. In the second experiment, we further investigate the influence of the shift intensity level and $\delta$ (which is, in a sense, another measure of shift intensity) on the power of the tests. 
We chose a fixed sample size of $100$, which was shown to make for a challenging yet doable task.
The results in Figure \ref{fig:shift_intensity} confirm that our representations were much more sensitive to weak shifts than the CV, with differences in power greater than $80\%$ for some intensities.

\begin{table*}
\begin{center}
\textbf{Impact of Shift Type}\par
\setlength{\tabcolsep}{5pt}
\begin{tabular}{llcccccc}
\hline
\multicolumn{8}{c}{Power of the test (\%)}\\
\hline
 \multirow{2}{*}{Shift} & \multirow{2}{*}{Feat.}  & \multicolumn{6}{c}{Shift intensity}\\
\hhline{~~------}
 &   & I & II & III & IV & V & VI  \\
\hhline{========}
\multirow{2}{*}{GN}  & {\texttt{MD}} & \textcolor{red}{$ 7.2 \pm 1.3 $} & $ 29.3 \pm 2.3 $ & \textcolor{blue}{$ 61.9 \pm 2.5 $ }&\textcolor{blue}{ $ 93.3 \pm 1.3 $ }&\textcolor{blue}{ $ 98.9 \pm 0.5 $ }&\textcolor{blue}{ $ 100.0 - 0.0 $ } \\
 & CV & \textcolor{red}{$ 0.0 + 0.2 $} & \textcolor{red}{$ 0.1 \pm 0.1 $} &\textcolor{red}{ $ 1.5 \pm 0.6 $ }& $ 17.6 \pm 1.9 $ & \textcolor{blue}{$ 72.3 \pm 2.3 $ }& \textcolor{blue}{$ 98.9 \pm 0.5 $}   \\
\hline
\multirow{2}{*}{GB}  & \texttt{MD} & \textcolor{red}{$ 3.7 \pm 1.0 $ } & \textcolor{red}{$ 4.3 \pm 1.0 $ }& $ 27.7 \pm 2.3 $ & \textcolor{blue}{$ 63.1 \pm 2.4 $ }&\textcolor{blue}{ $ 85.0 \pm 1.8 $ }&\textcolor{blue}{ $ 92.4 \pm 1.3 $}  \\
 & CV & \textcolor{red}{$ 0.0 + 0.2 $ }&\textcolor{red}{ $ 0.0 + 0.2 $ }&\textcolor{red}{ $ 0.0 + 0.2 $ }&\textcolor{red}{ $ 0.4 \pm 0.3 $ }& \textcolor{red}{$ 1.3 \pm 0.6 $ }&\textcolor{red}{ $ 5.3 \pm 1.1 $}  \\
\hline
\multirow{2}{*}{IS}  & \texttt{MD} & $ 10.3 \pm 1.5 $ & $ 32.7 \pm 2.4 $ & \textcolor{blue}{$ 53.5 \pm 2.5 $ }&\textcolor{blue}{ $ 78.5 \pm 2.1 $ }& \textcolor{blue}{$ 90.6 \pm 1.5 $} &\textcolor{blue}{ $ 98.9 \pm 0.5 $ }\\
 & CV & \textcolor{red}{$ 0.0 + 0.2 $ }&\textcolor{red}{ $ 0.1 \pm 0.2 $ }& \textcolor{red}{$ 2.1 \pm 0.7 $ }& $ 11.5 \pm 1.6 $ & $ 37.0 \pm 2.4 $ & \textcolor{blue}{$ 86.3 \pm 1.7 $ } \\
\hline
\end{tabular}

\begin{comment}
\begin{tabular}{llrrrrrr}
\multicolumn{8}{c}{Power of the test (\%)}\\
\hline
 \multirow{2}{*}{Shift type} & \multirow{2}{*}{Features}  & \multicolumn{6}{c}{Shift intensity}\\
\hhline{~~------}
 &   & I & II & III & IV & V & VI  \\
\hhline{========}
\multirow{2}{*}{Gaussian noise}  & \statname & $ 7.2 \pm 1.3 $ & $ 29.3 \pm 2.3 $ & $ 61.9 \pm 2.5 $ & $ 93.3 \pm 1.3 $ & $ 98.9 \pm 0.5 $ & $ 100.0 - 0.0 $  \\
 & CV & $ 0.0 + 0.2 $ & $ 0.1 \pm 0.1 $ & $ 1.5 \pm 0.6 $ & $ 17.6 \pm 1.9 $ & $ 72.3 \pm 2.3 $ & $ 98.9 \pm 0.5 $   \\
\hline
\multirow{2}{*}{Gaussian blur}  & \statname & $ 3.7 \pm 1.0 $ & $ 4.3 \pm 1.0 $ & $ 27.7 \pm 2.3 $ & $ 63.1 \pm 2.4 $ & $ 85.0 \pm 1.8 $ & $ 92.4 \pm 1.3 $  \\
 & CV & $ 0.0 + 0.2 $ & $ 0.0 + 0.2 $ & $ 0.0 + 0.2 $ & $ 0.4 \pm 0.3 $ & $ 1.3 \pm 0.6 $ & $ 5.3 \pm 1.1 $  \\
\hline
\multirow{2}{*}{Image shift}  & \statname & $ 10.3 \pm 1.5 $ & $ 32.7 \pm 2.4 $ & $ 53.5 \pm 2.5 $ & $ 78.5 \pm 2.1 $ & $ 90.6 \pm 1.5 $ & $ 98.9 \pm 0.5 $ \\
 & CV & $ 0.0 + 0.2 $ & $ 0.1 \pm 0.2 $ & $ 2.1 \pm 0.7 $ & $ 11.5 \pm 1.6 $ & $ 37.0 \pm 2.4 $ & $ 86.3 \pm 1.7 $  \\
\hline
\end{tabular}
\end{comment}
\caption{Power of the two methods (our method, denoted as \texttt{MD}, and CV)  as a function of the shift intensity 
%for various shift types 
for the shift types Gaussian noise (GN), Gaussian blur (GB) and Image shift (IS)
on the MNIST dataset with $\delta = 0.5$, sample size $100$, for the last dense layer. Red indicates that the estimated power is below $10\%$, blue that it is above $50\%$.
The $95\%$-confidence intervals have been estimated as mentioned in Section~\ref{sec:expes}.
}
\vspace{-3mm}
\label{table:shift-intensity-MNIST}
\end{center}
\end{table*}

\begin{wraptable}{r}{7cm}
\center
\vspace{-7mm}
\textbf{Choice of Layer}\par
%\hspace{2mm}
\begin{tabular}{lcccc}
\multicolumn{5}{c}{Averaged power (\%)}\\
\hline
 \multirow{2}{*}{Dataset}  & \multicolumn{4}{c}{Features}\\
\hhline{~----}
 & CV & $\ell_{-1}$ & $\ell_{-2}$ & $\ell_{-3}$ \\
\hhline{=====}
 MNIST & $25.1$ & $51.9$ & $53.0$ & $56.4$ \\
\hline
 FMNIST & $46.2$ & $44.9$ & $47.6$ & $53.7$ \\
\hline
\end{tabular}
\caption{Averaged performance of the various layers for \statname{}{} over all other parameters (for MNIST and FMNIST), compared to BBSD with CV.
A $95\%$ confidence interval for the averaged powers was estimated %via bootstrap 
and is in all cases  contained in a $\pm 0.1\%$ interval.}\label{table:layer-averaged}
\vspace{-10mm}
\end{wraptable} 

\paragraph{Shift type.}

The previous experiments focused on the case of Gaussian noise; in this experiment, we investigate whether the results hold for other shift types.
As detailed in Table~\ref{table:shift-intensity-MNIST}, %Figure~\ref{fig:shift_intensity}, %
the test with \statname{} representations reacted to the shifts even for low intensities of I, II, and III for all shift types (Gaussian blur being the most difficult case), while the KS test with CV was unable to detect anything.
For medium to high intensities III, IV, V and VI, \statname{} again significantly outperformed the baseline and reaches powers close to $1$ for all shift types.
For the Gaussian blur, the shift remained practically undetectable using CV.

\paragraph{\statname{} with respect to different layers.}

The NN architecture we used with MNIST and FMNIST  had several dense layers before the output.
As a variation of our method, we investigate the effect on the shift detection when computing our \statname{} representations with respect to different layers\footnote{Since ResNet18 only has a single dense layer after its convolutional layers, there is no choice to be made in the case of CIFAR-10, SVHN and Imagenette.}.
More precisely, we consider the last three dense layers denoted by $\ell_{-1}$, $\ell_{-2}$ and $\ell_{-3}$, ordered from the closest to the network output ($\ell_{-1}$) to the third from the end ($\ell_{-3}$).
The averaged results over all parameters and noise types are  in Table~\ref{table:layer-averaged}.

In the case of MNIST we only observe a slight increase in power when considering layer $\ell_{-3}$ further from the output of the NN.
In the case of FMNIST, on the other hand, we clearly see a much more pronounced improvement when switching from $\ell_{-1}$ to $\ell_{-3}$.
This hints at the possibility that features derived from encodings further within the NN can, in some cases, be more pertinent to the task of shift detection than those closer to the output.
% More detailed results can be found in the Supplementary Material.
% %%%%%Commented for submission:
% \charles{assuming we add the detailed table (function of shift intensity---Figure 5 in the experiments notes) in the supp mat ?}
% %\felix{Maybe some additional small comment/discussion should be added here?}

\section{Conclusion}

In this article, we derive new representations \statname{} from the activation graphs of a trained NN classifier.
We empirically show that using \statname{}~representations for  data set shift detection via coordinate-wise KS tests (with Bonferroni correction) significantly outperforms the baseline given by using confidence vectors  established in \cite{Lipton2018}, while remaining equally fast and easy to implement, making \statname{}~representations an efficient tool for this critical task.
\par

Our findings open many avenues for future investigations. We focused on classification of image data in this work, but our method is a general one and can 
%clearly 
be applied to other  settings. % in machine learning.
Moreover, adapting our method to regression tasks as well as to settings where shifts occur gradually is feasible and a starting point for future work. Finally, exploring variants of the \statname{} representations---considering several layers of the network at once, extending it to other types of layers, extracting finer topological information from the activation graphs, weighting the edges of the graph by backpropagating their contribution to the output, etc.---could also result in increased performance.

\bibliography{main}

\newcommand{\etalchar}[1]{$^{#1}$}
\begin{thebibliography}{WGS{\etalchar{+}}21b}

\bibitem[AFDP19]{alberge:hal-02172275}
Florence Alberge, Cl{\'e}ment Feutry, Pierre Duhamel, and Pablo Piantanida.
\newblock {Detecting covariate shift with Black Box predictors}.
\newblock In {\em {International Conference on Telecommunications (ICT 2019)}}, Hanoi, Vietnam, April 2019.

\bibitem[BSGEY22]{BarShalom}
Guy Bar-Shalom, Yonatan Geifman, and Ran El-Yaniv.
\newblock Distribution shift detection for deep neural networks, 2022.

\bibitem[CBK09]{AnomalyDetectionSurvey}
Varun Chandola, Arindam Banerjee, and Vipin Kumar.
\newblock Anomaly detection: A survey.
\newblock {\em ACM Comput. Surv.}, 41, 07 2009.

\bibitem[GBR{\etalchar{+}}12]{JMLR:v13:gretton12a}
Arthur Gretton, Karsten~M. Borgwardt, Malte~J. Rasch, Bernhard Sch{{\"o}}lkopf, and Alexander Smola.
\newblock A kernel two-sample test.
\newblock {\em Journal of Machine Learning Research}, 13(25):723--773, 2012.

\bibitem[GSH{\etalchar{+}}09]{gretton2009covariate}
Arthur Gretton, Alex Smola, Jiayuan Huang, Marcel Schmittfull, Karsten Borgwardt, and Bernhard Sch{\"o}lkopf.
\newblock Covariate shift by kernel mean matching.
\newblock {\em Dataset shift in machine learning}, 3(4):5, 2009.

\bibitem[GSH19]{gebhart2019characterizing}
Thomas Gebhart, Paul Schrater, and Alan Hylton.
\newblock Characterizing the shape of activation space in deep neural networks.
\newblock In {\em 2019 18th IEEE International Conference On Machine Learning And Applications (ICMLA)}, pages 1537--1542. IEEE, 2019.

\bibitem[Hec77]{Heckman1977}
James Heckman.
\newblock Sample selection bias as a specification error (with an application to the estimation of labor supply functions).
\newblock {\em National Bureau of Economic Research, Inc, NBER Working Papers}, 01 1977.

\bibitem[HG16]{hendrycks2016baseline}
Dan Hendrycks and Kevin Gimpel.
\newblock A baseline for detecting misclassified and out-of-distribution examples in neural networks.
\newblock {\em arXiv preprint arXiv:1610.02136}, 2016.

\bibitem[HPG{\etalchar{+}}17]{huang2017adversarial}
Sandy Huang, Nicolas Papernot, Ian Goodfellow, Yan Duan, and Pieter Abbeel.
\newblock Adversarial attacks on neural network policies.
\newblock {\em arXiv preprint arXiv:1702.02284}, 2017.

\bibitem[HS23]{heumann2023introduction}
C.~Heumann and M.~Schomaker.
\newblock {\em Introduction to Statistics and Data Analysis: With Exercises, Solutions and Applications in R}.
\newblock Springer International Publishing, 2023.

\bibitem[HTLM21]{HORTA2021109}
Vitor~A.C. Horta, Ilaria Tiddi, Suzanne Little, and Alessandra Mileo.
\newblock Extracting knowledge from deep neural networks through graph analysis.
\newblock {\em Future Generation Computer Systems}, 120:109--118, 2021.

\bibitem[HZRS15]{Resnet}
Kaiming He, X.~Zhang, Shaoqing Ren, and Jian Sun.
\newblock Deep residual learning for image recognition.
\newblock {\em 2016 IEEE Conference on Computer Vision and Pattern Recognition (CVPR)}, pages 770--778, 2015.

\bibitem[Ima23]{Imagenette}
Imagenette dataset.
\newblock \url{ https://github.com/fastai/imagenette}, 2023.
\newblock Accessed: 10/05/2023.

\bibitem[JPLB22]{Jang2022}
Sooyong Jang, Sangdon Park, Insup Lee, and Osbert Bastani.
\newblock Sequential covariate shift detection using classifier two-sample tests.
\newblock In Kamalika Chaudhuri, Stefanie Jegelka, Le~Song, Csaba Szepesvari, Gang Niu, and Sivan Sabato, editors, {\em Proceedings of the 39th International Conference on Machine Learning}, volume 162 of {\em Proceedings of Machine Learning Research}, pages 9845--9880. PMLR, 17--23 Jul 2022.

\bibitem[Kal61]{kallianpur1961topology}
Gopinath Kallianpur.
\newblock The topology of weak convergence of probability measures.
\newblock {\em Journal of Mathematics and Mechanics}, pages 947--969, 1961.

\bibitem[KH09]{CIFAR-10}
Alex Krizhevsky and Geoffrey Hinton.
\newblock Learning multiple layers of features from tiny images.
\newblock Technical Report~0, University of Toronto, Toronto, Ontario, 2009.

\bibitem[KRSW21]{Kim2021}
Ilmun Kim, Aaditya Ramdas, Aarti Singh, and Larry Wasserman.
\newblock {Classification accuracy as a proxy for two-sample testing}.
\newblock {\em The Annals of Statistics}, 49(1):411--434, 2021.

\bibitem[LBBH98]{MNIST}
Yann LeCun, L{\'e}on Bottou, Yoshua Bengio, and Patrick Haffner.
\newblock Gradient-based learning applied to document recognition.
\newblock {\em Proc. IEEE}, 86:2278--2324, 1998.

\bibitem[LIC{\etalchar{+}}21]{Lacombe2021}
Th{\'{e}}o Lacombe, Yuichi Ike, Mathieu Carri{\`{e}}re, Fr{\'{e}}d{\'{e}}ric Chazal, Marc Glisse, and Yuhei Umeda.
\newblock {Topological Uncertainty: monitoring trained neural networks through persistence of activation graphs}.
\newblock In {\em 30th International Joint Conference on Artificial Intelligence (IJCAI 2021)}, pages 2666--2672. International Joint Conferences on Artificial Intelligence Organization, 2021.

\bibitem[LWS18]{Lipton2018}
Zachary Lipton, Yu-Xiang Wang, and Alexander Smola.
\newblock Detecting and correcting for label shift with black box predictors.
\newblock In {\em 35th International Conference on Machine Learning (ICML 2018)}, volume~80, pages 3122--3130. PMLR, 2018.

\bibitem[NWC{\etalchar{+}}11]{SVHN}
Yuval Netzer, Tao Wang, Adam Coates, Alessandro Bissacco, Bo~Wu, and Andrew~Y. Ng.
\newblock Reading digits in natural images with unsupervised feature learning.
\newblock In {\em NIPS Workshop on Deep Learning and Unsupervised Feature Learning 2011}, 2011.

\bibitem[Res18]{ResNet18_pretrained}
Pre-trained weights for resnet18.
\newblock \url{https://pytorch.org/vision/main/models/generated/torchvision.models.resnet18.html}, 2018.
\newblock Accessed: 10/05/2023.

\bibitem[RGL19]{Rabanser2019}
Stephan Rabanser, Stephan G\"{u}nnemann, and Zachary Lipton.
\newblock Failing loudly: An empirical study of methods for detecting dataset shift.
\newblock In {\em Advances in Neural Information Processing Systems 33 (NeurIPS 2019)}, pages 1396--1408. Curran Associates, Inc., 2019.

\bibitem[RRP{\etalchar{+}}14]{Ramdas2014OnTD}
Aaditya Ramdas, Sashank~J. Reddi, Barnab{\'a}s P{\'o}czos, Aarti Singh, and Larry~A. Wasserman.
\newblock On the decreasing power of kernel and distance based nonparametric hypothesis tests in high dimensions.
\newblock In {\em AAAI Conference on Artificial Intelligence}, 2014.

\bibitem[SJP{\etalchar{+}}12]{CausalAnticausal}
Bernhard Schölkopf, Dominik Janzing, Jonas Peters, Eleni Sgouritsa, Kun Zhang, and Joris Mooij.
\newblock On causal and anticausal learning.
\newblock {\em Proceedings of the 29th International Conference on Machine Learning, ICML 2012}, 2, 06 2012.

\bibitem[SLD02]{saerens2002adjusting}
Marco Saerens, Patrice Latinne, and Christine Decaestecker.
\newblock Adjusting the outputs of a classifier to new a priori probabilities: a simple procedure.
\newblock {\em Neural computation}, 14(1):21--41, 2002.

\bibitem[Smi39]{smirnov1939estimation}
Nikolai~V Smirnov.
\newblock On the estimation of the discrepancy between empirical curves of distribution for two independent samples.
\newblock {\em Bull. Math. Univ. Moscou}, 2(2):3--14, 1939.

\bibitem[SSL18]{Shafaei2018DoesYM}
Alireza Shafaei, Mark~W. Schmidt, and J.~Little.
\newblock Does your model know the digit 6 is not a cat? a less biased evaluation of "outlier" detectors.
\newblock {\em ArXiv}, abs/1809.04729, 2018.

\bibitem[Sto09]{Storkey}
Amos Storkey.
\newblock When training and test sets are different: Characterizing learning transfer.
\newblock {\em Dataset Shift in Machine Learning}, pages 3--28, 01 2009.

\bibitem[UKY20]{uehara2020off}
Masatoshi Uehara, Masahiro Kato, and Shota Yasui.
\newblock Off-policy evaluation and learning for external validity under a covariate shift.
\newblock {\em Advances in Neural Information Processing Systems}, 33:49--61, 2020.

\bibitem[VG95]{Voss1073}
Simon Voss and Steve George.
\newblock Multiple significance tests.
\newblock {\em BMJ}, 310(6986):1073, 1995.

\bibitem[WGS{\etalchar{+}}21a]{wiles2021fine}
Olivia Wiles, Sven Gowal, Florian Stimberg, Sylvestre Alvise-Rebuffi, Ira Ktena, Krishnamurthy Dvijotham, and Taylan Cemgil.
\newblock A fine-grained analysis on distribution shift.
\newblock {\em arXiv preprint arXiv:2110.11328}, 2021.

\bibitem[WGS{\etalchar{+}}21b]{ShiftGeneral}
Olivia Wiles, Sven Gowal, Florian Stimberg, Sylvestre Alvise-Rebuffi, Ira Ktena, Krishnamurthy Dvijotham, and Taylan Cemgil.
\newblock A fine-grained analysis on distribution shift, 2021.

\bibitem[XRV17]{xiao2017fashionmnist}
Han Xiao, Kashif Rasul, and Roland Vollgraf.
\newblock Fashion-mnist: a novel image dataset for benchmarking machine learning algorithms, 2017.
\newblock cite arxiv:1708.07747Comment: Dataset is freely available at https://github.com/zalandoresearch/fashion-mnist Benchmark is available at http://fashion-mnist.s3-website.eu-central-1.amazonaws.com/.

\bibitem[ZGB13]{Zaremba2013}
Wojciech Zaremba, Arthur Gretton, and Matthew Blaschko.
\newblock B-test: A non-parametric, low variance kernel two-sample test.
\newblock In C.J. Burges, L.~Bottou, M.~Welling, Z.~Ghahramani, and K.Q. Weinberger, editors, {\em Advances in Neural Information Processing Systems}, volume~26. Curran Associates, Inc., 2013.

\bibitem[ZSMW13]{Zhang2013}
Kun Zhang, Bernhard Sch\"{o}lkopf, Krikamol Muandet, and Zhikun Wang.
\newblock Domain adaptation under target and conditional shift.
\newblock In {\em Proceedings of the 30th International Conference on International Conference on Machine Learning - Volume 28}, ICML'13, page III–819–III–827. JMLR.org, 2013.

\end{thebibliography}

\newpage

\appendix

\section{Appendix}

\subsection{Additional Information on the Experimental Procedures.}
\paragraph{Datasets.}
The number of samples in the clean sets (\textit{i.e.}, the test sets) of the datasets we investigated are as follows:
\begin{itemize}
    \item MNIST, FMNIST and CIFAR-10: $10'000$,
    \item SVHN: $26'032$.
    \item Imagenette: $3'925$.
\end{itemize}
A sample from each dataset can be seen in Figure \ref{fig:dataset_illustration}.
%\charles{Maybe give the training set size as well ?}

%\charles{Add illustration}
\begin{figure*}[h]
    \centering
    
    \includegraphics[width=0.13\textwidth]{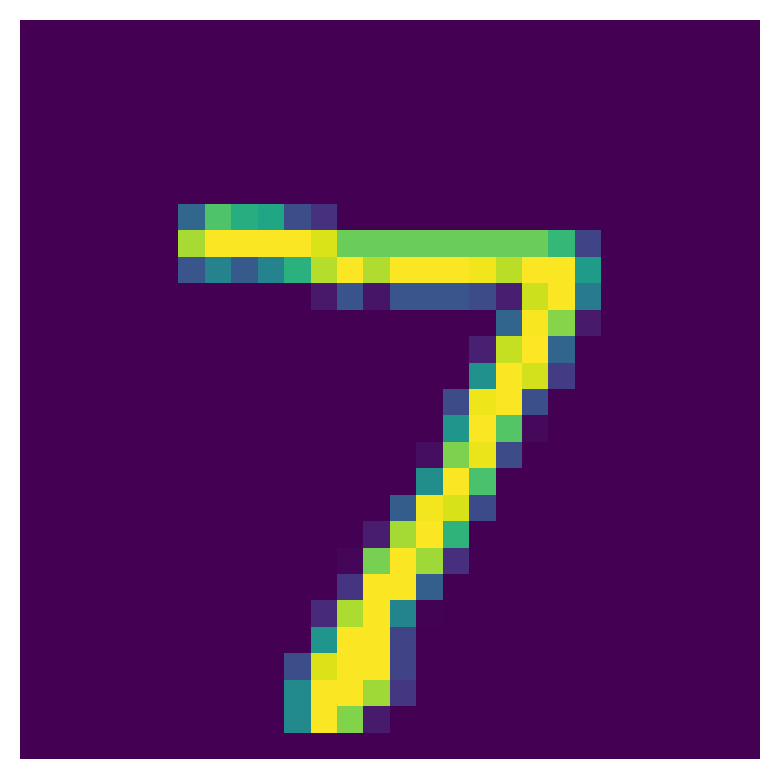}
    \includegraphics[width=0.13\textwidth]{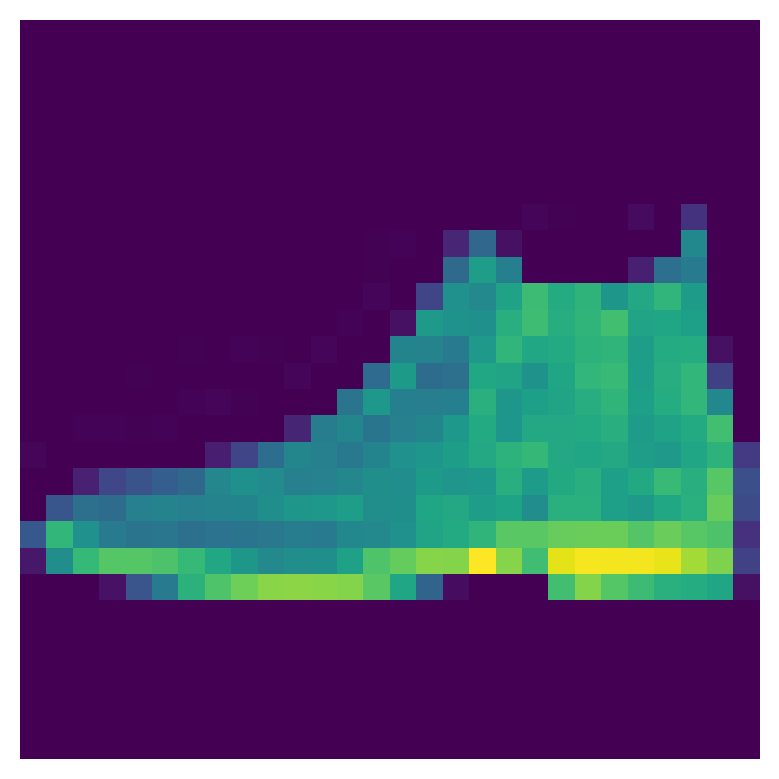}
   \includegraphics[width=0.13\textwidth]{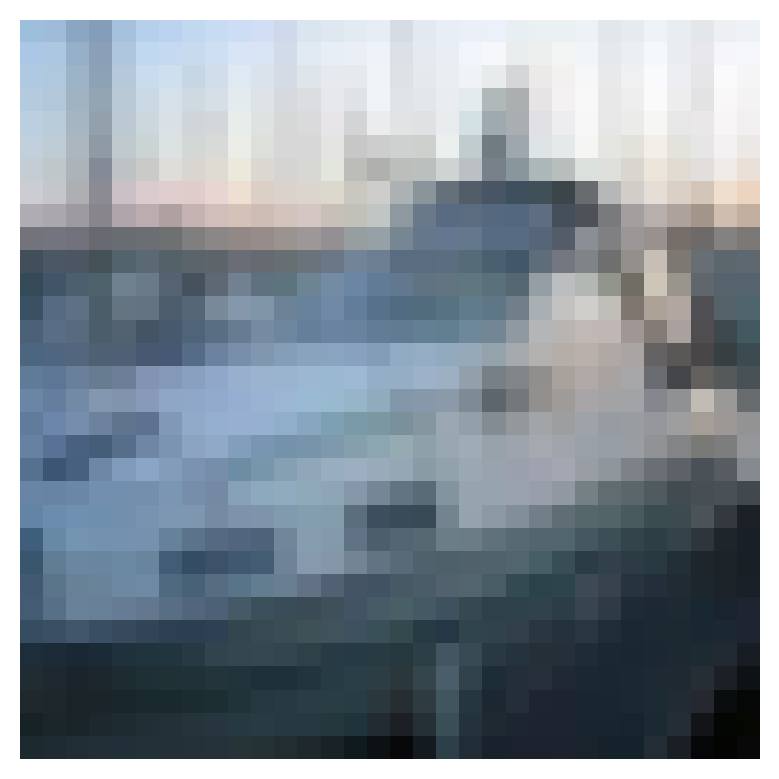}
    \includegraphics[width=0.13\textwidth]{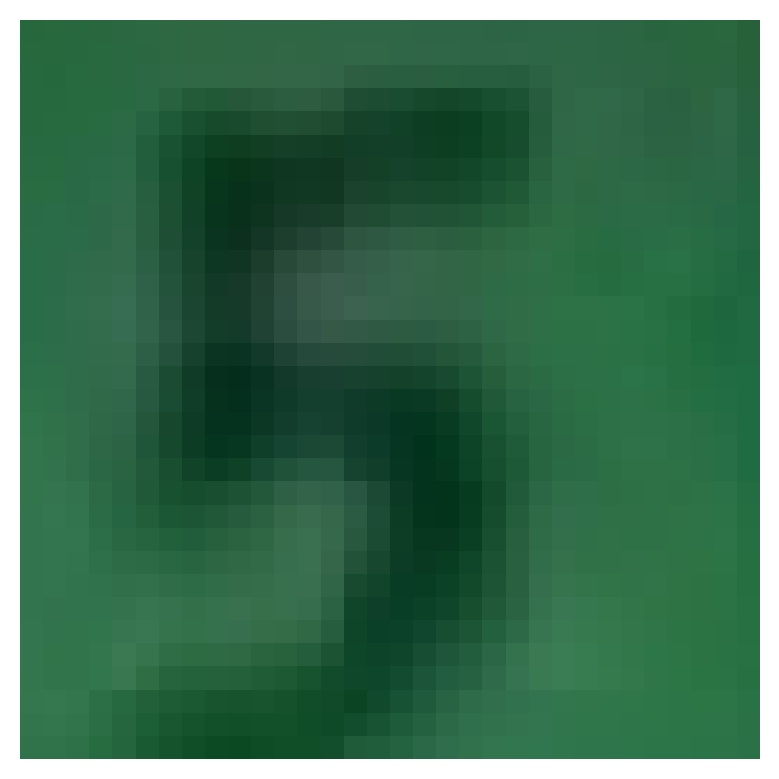}
    \includegraphics[width=0.165\textwidth]{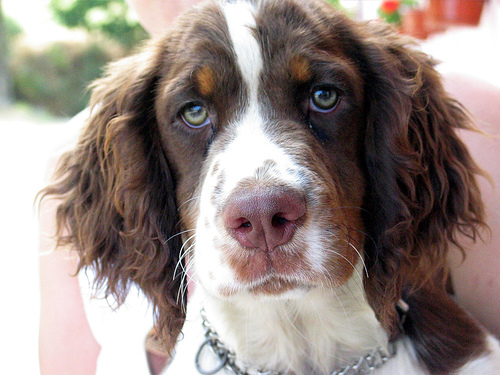}

   \caption{
    Sample images from all datasets used in the paper.
    From left to right: MNIST, FashionMNIST, CIFAR-10, SVHN and Imagenette.
   }
     \label{fig:dataset_illustration}
\end{figure*}

\paragraph{Shifts.}
In order to illustrate the effect of the shift types described in Section~\ref{subsec:expes_settings} of the main article, we show the effects of the shifts and their intensities on the MNIST dataset in Figure~\ref{fig:supplementary_shift_types_MNIST}.
For the detailed parameters of each shift intensity (per dataset) we refer to the associated code.

\begin{figure}[h]
    \centering
    \includegraphics[width= 0.625\columnwidth]{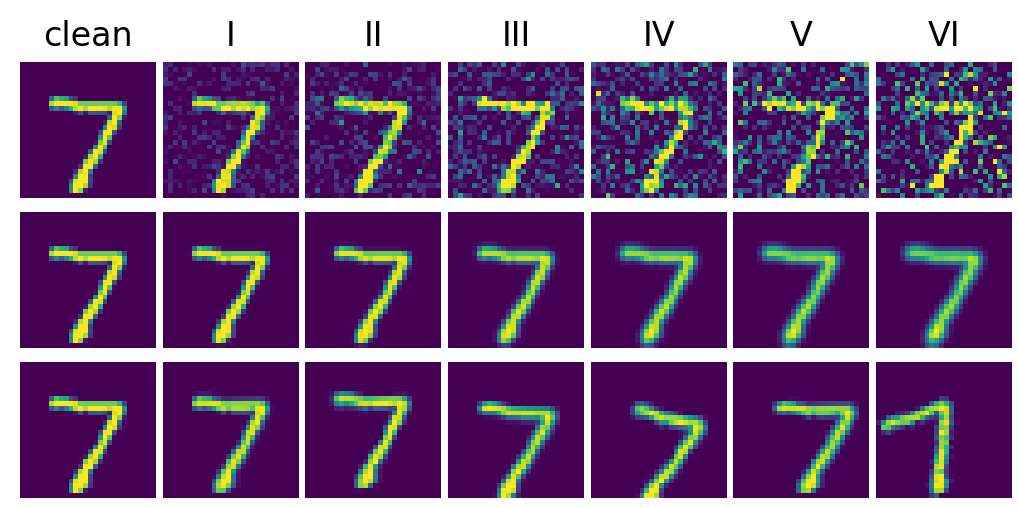}
   \caption{Illustration of intensities of the shift types --- Gaussian noise (top row), Gaussian blur (middle row) and Image shift (bottom row) --- on a sample from the MNIST dataset.}
     \label{fig:supplementary_shift_types_MNIST}
\end{figure}

%\begin{figure}[h]
  %  \centering
 %   \includegraphics[width=\columnwidth]{./figures/supplementary/Imagenette_shifted_image.png}
 %  \caption{Illustration of intensities of the shift types --- Gaussian noise (top row), Gaussian blur (middle row) and Image shift (bottom row) --- on a sample from the Imagenette dataset.}
         %\label{fig:supplementary_shift_types_Imagenette}
%\end{figure}
% \begin{figure}[h]
%     \centering
%     \includegraphics[width=\columnwidth]{./figures/supplementary/FMNIST_shifted_image.png}
%    \caption{Illustration of intensities of the shift types---Gaussian noise (top row), Gaussian blur (middle row) and Image shift (bottom row)---on a sample from the FMNIST dataset.}
%      \label{fig:supplementary_shift_types_FMNIST}
% \end{figure}

\paragraph{A remark on the experimental setup}
As explained in Subsection~\ref{subsec:expes_settings} of the main text, we perform each of our statistical tests on random subsamples from a clean set CS and a shifted set SS that has been crafted from CS by applying a controlled shift to its elements.
We could alternatively have crafted SS by applying the shift to another ``clean'' set drawn from the same distribution as CS.
The advantage of our current setup is that it ensures that the only difference between the distributions of CS and SS comes from the shift, which we completely control, as opposed to some (small) pre-shift difference between the two clean sets that could arise from the sampling process in this alternative setup.
In practice, however, both setups should yield extremely similar results; indeed, we randomly subsample from CS and SS to get the sets on which we perform the tests, and those sets are typically of size $\sim 100$, whereas CS and SS are of size $3'925$, $10'000$ or $26'032$ depending on the dataset. Hence the
subsampled sets from the CS and from SS should have almost no element in common
(by which we mean the same picture appearing shifted in one of the sets and unshifted
in the other), as in the other setup.

\paragraph{An example of activation graph computation}
We illustrate the definition of the activation graph from Subsection \ref{subsec:activatin_graphs_def}, using the same notations.
A layer $f_\ell : \R^{n_\ell} \rightarrow \R^{n_{\ell+1}}$ is shown on the left of Figure \ref{fig:activation_graph_computation}. Here $n_\ell =2$ and $n_{\ell+1} = 3$, and the weight matrix associated to $f_\ell$ is
$W_\ell = \begin{bmatrix}
-1 & 0 \\
2 & 4 \\
3 & -2 
\end{bmatrix}$. Given an input $x$ to the network that corresponds to an input $x_\ell = (1,-2) $ to the layer $f_\ell$, the associated activation graph $G_\ell(x)$ is shown on the right of Figure \ref{fig:activation_graph_computation}.

\begin{figure}[h]
    \centering
    \includegraphics[width= 0.35\columnwidth]{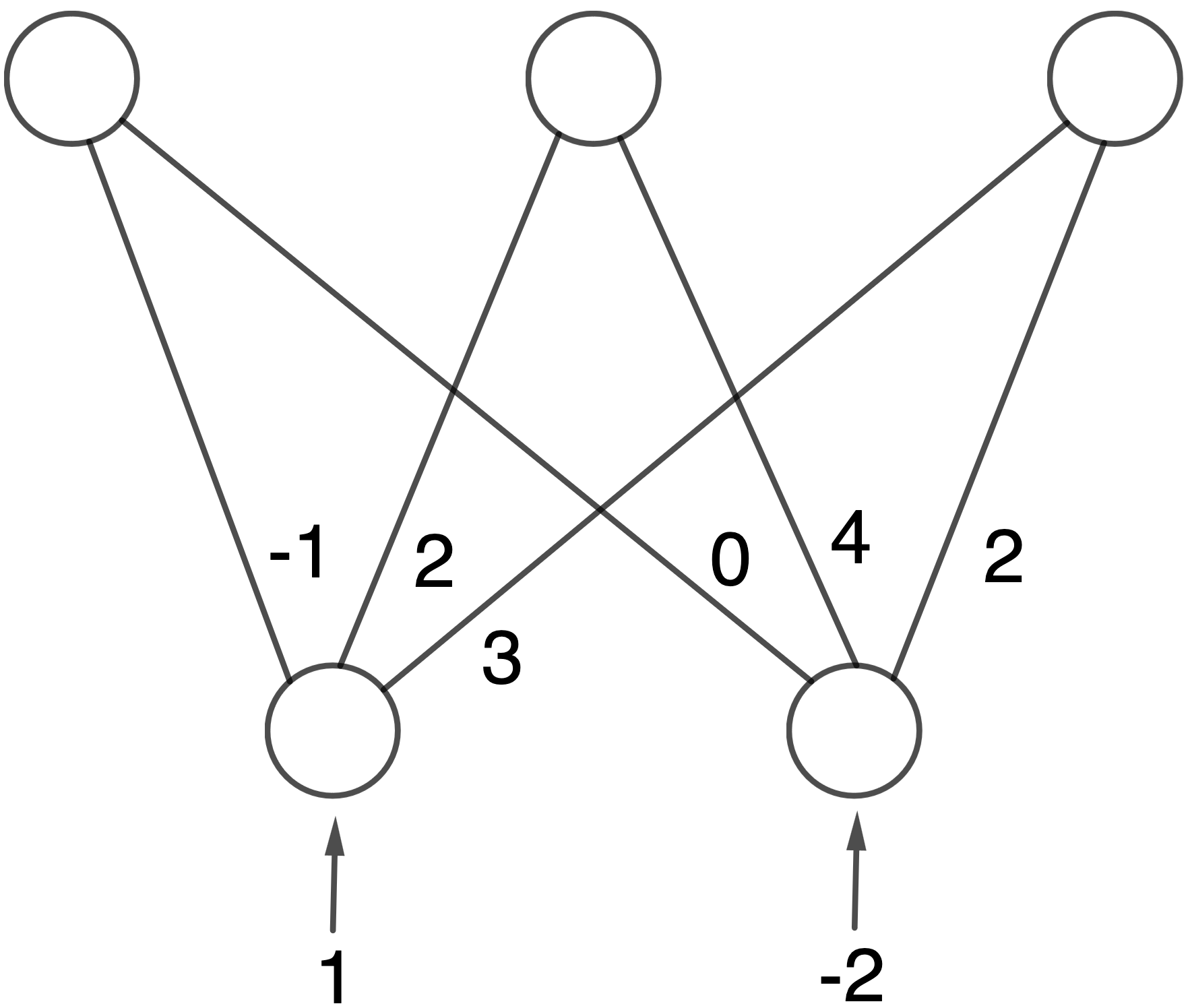}
    \includegraphics[width= 0.35\columnwidth]{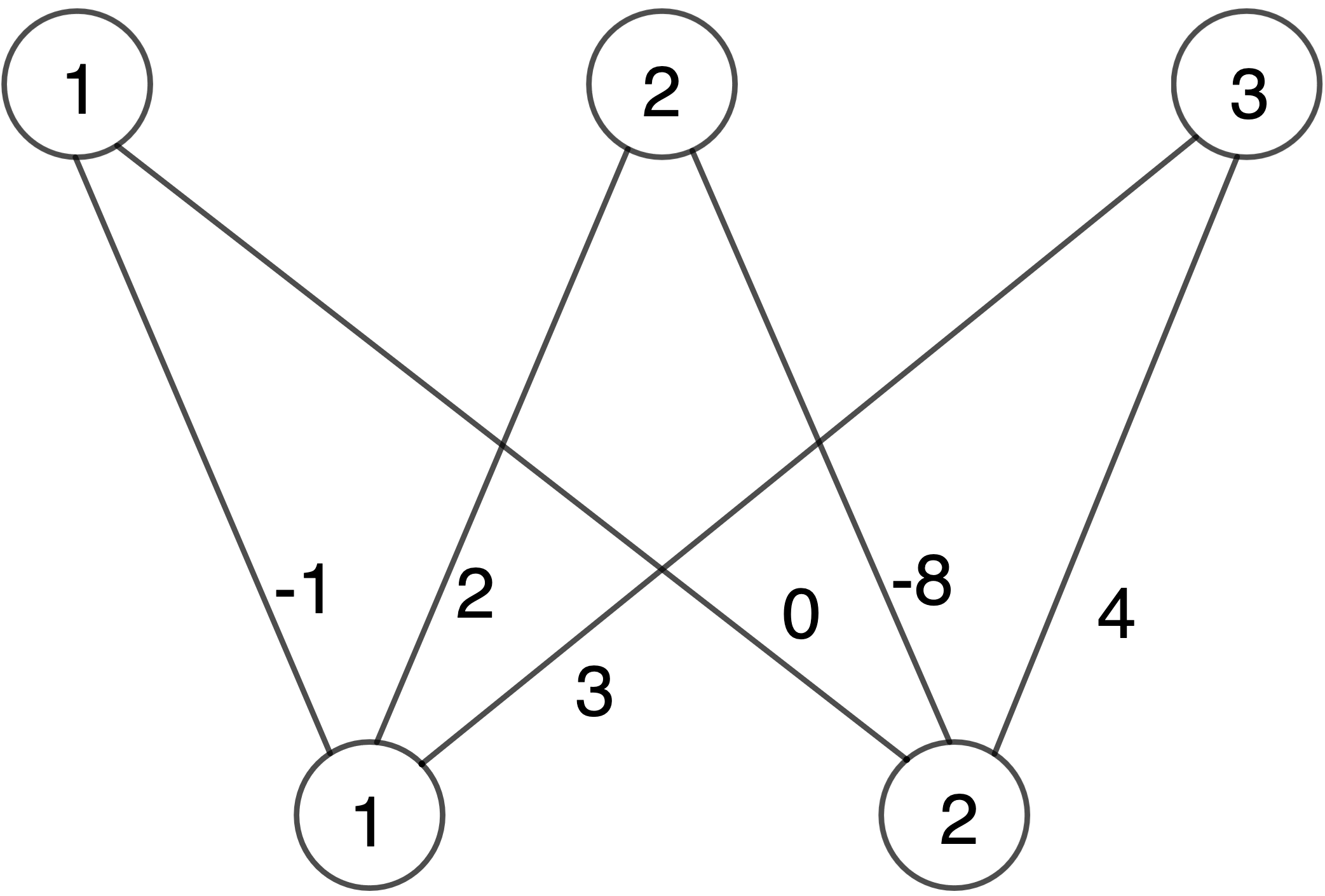}
   \caption{On the left, a neural network layer $f_\ell$ with the associated weights and an input vector $x_\ell$. On the right, the associated activation graph.}
     \label{fig:activation_graph_computation}
\end{figure}

\newpage

\subsection{Additional Experimental Results}

% \paragraph{Shift type.}

% \begin{table*}
% \begin{center}
% \input{./tables/CIFAR10_delta=0.5_SS=100.tex}
% \caption{Power of the two methods as a function of the shift intensity for various shift types on the CIFAR-10 dataset with $\delta = 0.5$, sample size $100$. Red indicates that the estimated power is below $10\%$, blue that it is above $50\%$.
% The $95\%$-confidence intervals have been estimated as mentioned in Section~5.1 \felix{cross-reference?}, \textbf{Experimental protocol}.
% }
% \label{table:Experiment3-CIFAR10}
% \end{center}
% \end{table*}

\paragraph{Sample size.}

To further support our claims, we include comprehensive results of the power with respect to the sample size for the MNIST,  Imagenette and CIFAR-10  datasets in Tables~\ref{table:MNIST-wrt-sample-size}, \ref{table:Imagenette-wrt-sample-size} and \ref{table:CIFAR10-wrt-sample-size}.
We provide all results for the shift intensities II, IV and VI, for all shift types, and fixed $\delta =0.5$ for MNIST, CIFAR-10, respectively $\delta = 1.0$ for Imagenette (the $\delta$ were chosen so that the task is comparatively easy at high shift intensity and hard at low shift intensity for both methods).

\paragraph{Shift intensity.}
In Figures~\ref{fig:shift_intensity_MNIST_supplementary}, \ref{fig:shift_intensity_FMNIST_supplementary} and \ref{fig:shift_intensity_CIFAR10_SVHN_Imagenette_supplementary}, we collect the plots of the estimated powers of the test for multiple cases, in addition to the one presented in the main article.
Note that the only situation in which \statname{} is very slightly outperformed by the baseline CV, is the case of FMNIST, when we consider \statname{} representations of layer $l_{-1}$.
In all other cases, shift detection using \statname{} representations clearly outperforms the baseline of CV by a large margin.
%In the case of CIFAR-10 and SVHN and for Image shift we decided to show the results for sample size $200$ \charles{check if also needed for Imagenette}since the shift detection for lower sample size is more difficult than for other datasets and shift types, and hence the corresponding plots are less expressive.

\paragraph{Model accuracy.}

In Figure~\ref{fig:model_accuracy_supplementary} we show the impact of the shift type and intensity on the model accuracy.
It is interesting to note that, even in cases where the model accuracy is only minimally impacted (\textit{e.g.},~for Gaussian blur on the MNIST and FMNIST datasets), our method can still reliably detect the presence of the shift.

\begin{table*}
\small{
\begin{center}
\setlength{\tabcolsep}{0.4em}
\begin{tabular}{lcr c ccccccc}
 & & & & \multicolumn{7}{c}{MNIST}\\
\hhline{~~~~-------}
 \multirow{2}{*}{Shift} &  \multirow{2}{*}{Int.} & \multirow{2}{*}{Feat.}  & &\multicolumn{7}{c}{Sample size}\\
\hhline{~~~~-------}
 &  & &  & $10$ & $20$ & $50$ & $100$ & $200$ & $500$ & $1000$ \\
\hhline{===~=======}
\multirow{6}{*}{GN}  & \multirow{2}{*}{II} & \texttt{MD} &  &  $ 2.7 \pm 0.8 $ & $ 7.7 \pm 1.3 $ & $ 11.2 \pm 1.6 $ & $ 29.3 \pm 2.3 $ & $ 55.9 \pm 2.5 $ & $ 94.7 \pm 1.1 $ & $ 99.9 \pm 0.1 $ \\
 & & CV &  &  $ 0.0 + 0.2 $ & $ 0.1 \pm 0.1 $ & $ 0.0 + 0.2 $ & $ 0.1 \pm 0.1 $ & $ 0.3 \pm 0.3 $ & $ 1.5 \pm 0.6 $ & $ 6.1 \pm 1.2 $\\
  & \multirow{2}{*}{IV} & \texttt{MD} &  &  $ 8.6 \pm 1.4 $ & $ 26.1 \pm 2.2 $ & $ 60.9 \pm 2.5 $ & $ 93.3 \pm 1.3 $ & $ 100.0 - 0.0 $ & $ 100.0 - 0.0 $ & $ 100.0 - 0.0 $ \\
 & & CV &  &  $ 0.1 \pm 0.2 $ & $ 1.0 \pm 0.5 $ & $ 3.8 \pm 1.0 $ & $ 17.6 \pm 1.9 $ & $ 58.4 \pm 2.5 $ & $ 99.3 \pm 0.4 $ & $ 100.0 - 0.0 $\\
   & \multirow{2}{*}{VI} & \texttt{MD} &  &  $ 16.5 \pm 1.9 $ & $ 54.2 \pm 2.5 $ & $ 93.5 \pm 1.3 $ & $ 100.0 - 0.0 $ & $ 100.0 - 0.0 $ & $ 100.0 - 0.0 $ & $ 100.0 - 0.0 $ \\
 & & CV &  &  $ 3.7 \pm 1.0 $ & $ 20.4 \pm 2.0 $ & $ 65.6 \pm 2.4 $ & $ 98.9 \pm 0.5 $ & $ 100.0 - 0.0 $ & $ 100.0 - 0.0 $ & $ 100.0 - 0.0 $  \\
 
\hhline{---~-------}
\multirow{6}{*}{GB}  & \multirow{2}{*}{II} & \texttt{MD} &  &  $ 1.9 \pm 0.7 $ & $ 2.7 \pm 0.8 $ & $ 3.3 \pm 0.9 $ & $ 4.3 \pm 1.0 $ & $ 9.9 \pm 1.5 $ & $ 22.2 \pm 2.1 $ & $ 40.7 \pm 2.5 $ \\
 & & CV &  &  $ 0.0 + 0.2 $ & $ 0.1 \pm 0.1 $ & $ 0.0 + 0.2 $ & $ 0.0 + 0.2 $ & $ 0.0 + 0.2 $ & $ 0.1 \pm 0.1 $ & $ 0.1 \pm 0.1 $\\
  & \multirow{2}{*}{IV} & \texttt{MD} &  &  $ 4.8 \pm 1.1 $ & $ 13.1 \pm 1.7 $ & $ 30.3 \pm 2.3 $ & $ 63.1 \pm 2.4 $ & $ 93.9 \pm 1.2 $ & $ 100.0 - 0.0 $ & $ 100.0 - 0.0 $ \\
 & & CV &  &  $ 0.0 + 0.2 $ & $ 0.0 + 0.2 $ & $ 0.1 \pm 0.2 $ & $ 0.3 \pm 0.3 $ & $ 1.5 \pm 0.6 $ & $ 11.3 \pm 1.6 $ & $ 44.9 \pm 2.5 $\\
   & \multirow{2}{*}{VI} & \texttt{MD} &  &  $ 9.2 \pm 1.5 $ & $ 25.1 \pm 2.2 $ & $ 57.5 \pm 2.5 $ & $ 92.4 \pm 1.3 $ & $ 100.0 - 0.0 $ & $ 100.0 - 0.0 $ & $ 100.0 - 0.0 $ \\
 & & CV &  &  $ 0.1 \pm 0.1 $ & $ 0.5 \pm 0.4 $ & $ 1.4 \pm 0.6 $ & $ 5.1 \pm 1.1 $ & $ 22.1 \pm 2.1 $ & $ 88.5 \pm 1.6 $ & $ 100.0 - 0.0 $  \\
\hhline{---~-------}

\multirow{6}{*}{IS}  & \multirow{2}{*}{II} & \texttt{MD} &  &  $ 3.5 \pm 0.9 $ & $ 8.6 \pm 1.4 $ & $ 15.1 \pm 1.8 $ & $ 32.7 \pm 2.4 $ & $ 66.3 \pm 2.4 $ & $ 98.0 \pm 0.7 $ & $ 100.0 - 0.0 $ \\
 & & CV &  &  $ 0.0 + 0.2 $ & $ 0.0 + 0.2 $ & $ 0.0 + 0.2 $ & $ 0.1 \pm 0.2 $ & $ 0.7 \pm 0.4 $ & $ 6.9 \pm 1.3 $ & $ 28.4 \pm 2.3 $  \\
  & \multirow{2}{*}{IV} & \texttt{MD} &  &  $ 5.6 \pm 1.2 $ & $ 18.5 \pm 2.0 $ & $ 42.1 \pm 2.5 $ & $ 78.5 \pm 2.1 $ & $ 98.0 \pm 0.7 $ & $ 100.0 - 0.0 $ & $ 100.0 - 0.0 $ \\
 & & CV &  &  $ 0.1 \pm 0.2 $ & $ 0.9 \pm 0.5 $ & $ 2.4 \pm 0.8 $ & $ 15.5 \pm 1.8 $ & $ 50.0 \pm 2.5 $ & $ 99.5 \pm 0.3 $ & $ 100.0 - 0.0 $ \\
   & \multirow{2}{*}{VI} & \texttt{MD} &  &  $ 10.4 \pm 1.5 $ & $ 34.0 \pm 2.4 $ & $ 72.6 \pm 2.3 $ & $ 98.9 \pm 0.5 $ & $ 100.0 - 0.0 $ & $ 100.0 - 0.0 $ & $ 100.0 - 0.0 $ \\
 & & CV &  &  $ 1.3 \pm 0.6 $ & $ 7.3 \pm 1.3 $ & $ 31.7 \pm 2.4 $ & $ 83.3 \pm 1.9 $ & $ 100.0 - 0.0 $ & $ 100.0 - 0.0 $ & $ 100.0 - 0.0 $ \\
\hline
\end{tabular}
\caption{Power of the statistical test with \statname{} (abbreviated as \texttt{MD}) and CV representations for the shift types Gaussian noise (GN), Gaussian blur (GB) and Image shift (IS), three different shift intensities (II, IV, VI) and fixed $\delta = 0.5$ for the MNIST dataset. The estimated $95\%$-confidence intervals are indicated.}
\label{table:MNIST-wrt-sample-size}
\end{center}
}
\end{table*}

\begin{table*}
\small{
\begin{center}
\setlength{\tabcolsep}{0.4em}
\begin{tabular}{lcr c ccccccc}
 & & & & \multicolumn{7}{c}{Imagenette}\\
\hhline{~~~~-------}
 \multirow{2}{*}{Shift} &  \multirow{2}{*}{Int.} & \multirow{2}{*}{Feat.}  & &\multicolumn{7}{c}{Sample size}\\
\hhline{~~~~-------}
 &  & &  & $10$ & $20$ & $50$ & $100$ & $200$ & $500$ & $1000$ \\
\hhline{===~=======}
\multirow{6}{*}{GN}  & \multirow{2}{*}{II} & \texttt{MD} &  &  $ 0.7 \pm 0.4 $ & $ 2.2 \pm 0.7 $ & $ 6.3 \pm 1.2 $ & $ 16.7 \pm 1.9 $ & $ 46.2 \pm 2.5 $ & $ 93.3 \pm 1.3 $ & $ 99.9 \pm 0.1 $\\
 & & CV &  &  $ 2.4 \pm 0.8 $ & $ 4.5 \pm 1.1 $ & $ 2.7 \pm 0.8 $ & $ 4.1 \pm 1.0 $ & $ 4.2 \pm 1.0 $ & $ 7.3 \pm 1.3 $ & $ 10.3 \pm 1.5 $\\
  & \multirow{2}{*}{IV} & \texttt{MD} &  &  $ 0.9 \pm 0.5 $ & $ 3.6 \pm 0.9 $ & $ 7.3 \pm 1.3 $ & $ 22.1 \pm 2.1 $ & $ 60.7 \pm 2.5 $ & $ 98.5 \pm 0.6 $ & $ 100.0 - 0.0 $\\
 & & CV &  &  $ 1.6 \pm 0.6 $ & $ 4.1 \pm 1.0 $ & $ 3.2 \pm 0.9 $ & $ 3.9 \pm 1.0 $ & $ 4.7 \pm 1.1 $ & $ 7.5 \pm 1.3 $ & $ 10.1 \pm 1.5 $\\
   & \multirow{2}{*}{VI} & \texttt{MD} &  &  $ 0.9 \pm 0.5 $ & $ 3.6 \pm 0.9 $ & $ 7.3 \pm 1.3 $ & $ 22.1 \pm 2.1 $ & $ 60.7 \pm 2.5 $ & $ 98.5 \pm 0.6 $ & $ 100.0 - 0.0 $ \\
 & & CV &  &  $ 2.2 \pm 0.7 $ & $ 4.9 \pm 1.1 $ & $ 3.5 \pm 0.9 $ & $ 5.3 \pm 1.1 $ & $ 6.8 \pm 1.3 $ & $ 15.5 \pm 1.8 $ & $ 33.5 \pm 2.4 $  \\
\hhline{---~-------}

\multirow{6}{*}{GB}  & \multirow{2}{*}{II} & \texttt{MD} &  &  $ 0.5 \pm 0.3 $ & $ 2.5 \pm 0.8 $ & $ 4.1 \pm 1.0 $ & $ 15.3 \pm 1.8 $ & $ 40.6 \pm 2.5 $ & $ 91.7 \pm 1.4 $ & $ 99.9 \pm 0.2 $ \\
 & & CV &  &  $ 2.1 \pm 0.7 $ & $ 3.7 \pm 1.0 $ & $ 3.2 \pm 0.9 $ & $ 3.7 \pm 1.0 $ & $ 7.0 \pm 1.3 $ & $ 11.3 \pm 1.6 $ & $ 18.2 \pm 2.0 $\\
  & \multirow{2}{*}{IV} & \texttt{MD} &  &  $ 1.0 \pm 0.5 $ & $ 3.5 \pm 0.9 $ & $ 9.1 \pm 1.5 $ & $ 29.3 \pm 2.3 $ & $ 67.9 \pm 2.4 $ & $ 99.1 \pm 0.5 $ & $ 100.0 - 0.0 $ \\
 & & CV &  &  $ 2.3 \pm 0.8 $ & $ 4.5 \pm 1.0 $ & $ 3.7 \pm 1.0 $ & $ 4.5 \pm 1.1 $ & $ 6.3 \pm 1.2 $ & $ 13.1 \pm 1.7 $ & $ 26.9 \pm 2.2 $\\
   & \multirow{2}{*}{VI} & \texttt{MD} &  &  $ 1.3 \pm 0.6 $ & $ 5.0 \pm 1.1 $ & $ 17.2 \pm 1.9 $ & $ 50.5 \pm 2.5 $ & $ 89.9 \pm 1.5 $ & $ 100.0 - 0.0 $ & $ 100.0 - 0.0 $ \\
 & & CV &  &  $ 2.1 \pm 0.7 $ & $ 3.5 \pm 0.9 $ & $ 3.1 \pm 0.9 $ & $ 5.0 \pm 1.1 $ & $ 6.9 \pm 1.3 $ & $ 18.5 \pm 2.0 $ & $ 46.8 \pm 2.5 $  \\
\hhline{---~-------}

\multirow{6}{*}{IS}  & \multirow{2}{*}{II} & \texttt{MD} &  &  $ 0.5 \pm 0.4 $ & $ 1.1 \pm 0.5 $ & $ 1.9 \pm 0.7 $ & $ 4.7 \pm 1.1 $ & $ 13.6 \pm 1.7 $ & $ 44.5 \pm 2.5 $ & $ 83.1 \pm 1.9 $ \\
 & & CV &  & $ 2.5 \pm 0.8 $ & $ 3.7 \pm 1.0 $ & $ 3.9 \pm 1.0 $ & $ 3.1 \pm 0.9 $ & $ 4.3 \pm 1.0 $ & $ 6.3 \pm 1.2 $ & $ 9.9 \pm 1.5 $\\
  & \multirow{2}{*}{IV} & \texttt{MD} &  &  $ 0.3 \pm 0.3 $ & $ 1.9 \pm 0.7 $ & $ 2.9 \pm 0.9 $ & $ 9.1 \pm 1.5 $ & $ 28.1 \pm 2.3 $ & $ 75.1 \pm 2.2 $ & $ 98.3 \pm 0.7 $ \\
 & & CV &  &  $ 1.7 \pm 0.7 $ & $ 3.0 \pm 0.9 $ & $ 3.6 \pm 0.9 $ & $ 3.9 \pm 1.0 $ & $ 4.8 \pm 1.1 $ & $ 8.3 \pm 1.4 $ & $ 12.8 \pm 1.7 $\\
   & \multirow{2}{*}{VI} & \texttt{MD} &  &  $ 0.6 \pm 0.4 $ & $ 2.5 \pm 0.8 $ & $ 5.0 \pm 1.1 $ & $ 14.9 \pm 1.8 $ & $ 44.3 \pm 2.5 $ & $ 93.5 \pm 1.2 $ & $ 99.9 \pm 0.1 $\\
 & & CV &  &  $ 2.0 \pm 0.7 $ & $ 3.9 \pm 1.0 $ & $ 1.9 \pm 0.7 $ & $ 4.6 \pm 1.1 $ & $ 6.1 \pm 1.2 $ & $ 8.3 \pm 1.4 $ & $ 15.5 \pm 1.8 $ \\
\hline
\end{tabular}
\caption{Power of the statistical test with \statname{} (abbreviated as \texttt{MD}) and CV representations for the shift types Gaussian noise (GN), Gaussian blur (GB) and Image shift (IS), three different shift intensities (II, IV, VI) and fixed $\delta = 1$ for the Imagenette dataset. The estimated $95\%$-confidence intervals are indicated.}
\label{table:Imagenette-wrt-sample-size}
\end{center}
}
\end{table*}

\begin{table*}
\small{
\begin{center}
\setlength{\tabcolsep}{0.4em}
\begin{tabular}{lcr c ccccccc}
 & & & & \multicolumn{7}{c}{CIFAR-10}\\
\hhline{~~~~-------}
 \multirow{2}{*}{Shift} &  \multirow{2}{*}{Int.} & \multirow{2}{*}{Feat.}  & &\multicolumn{7}{c}{Sample size}\\
\hhline{~~~~-------}
 &  & &  & $10$ & $20$ & $50$ & $100$ & $200$ & $500$ & $1000$ \\
\hhline{===~=======}
\multirow{6}{*}{GN}  & \multirow{2}{*}{II} & \texttt{MD} &  &  $ 1.8 \pm 0.7 $ & $ 5.3 \pm 1.1 $ & $ 16.7 \pm 1.9 $ & $ 47.0 \pm 2.5 $ & $ 86.9 \pm 1.7 $ & $ 100.0 - 0.0 $ & $ 100.0 - 0.0 $\\
 & & CV &  &  $ 2.3 \pm 0.8 $ & $ 4.5 \pm 1.1 $ & $ 6.9 \pm 1.3 $ & $ 19.1 \pm 2.0 $ & $ 38.3 \pm 2.5 $ & $ 88.3 \pm 1.6 $ & $ 99.9 \pm 0.1 $\\
  & \multirow{2}{*}{IV} & \texttt{MD} &  &  $ 2.5 \pm 0.8 $ & $ 11.1 \pm 1.6 $ & $ 36.7 \pm 2.4 $ & $ 81.1 \pm 2.0 $ & $ 99.3 \pm 0.4 $ & $ 100.0 - 0.0 $ & $ 100.0 - 0.0 $\\
 & & CV &  &  $ 2.5 \pm 0.8 $ & $ 6.7 \pm 1.3 $ & $ 11.7 \pm 1.6 $ & $ 29.9 \pm 2.3 $ & $ 63.4 \pm 2.4 $ & $ 99.3 \pm 0.4 $ & $ 100.0 - 0.0 $\\
   & \multirow{2}{*}{VI} & \texttt{MD} &  &  $ 2.7 \pm 0.8 $ & $ 14.7 \pm 1.8 $ & $ 49.2 \pm 2.5 $ & $ 91.2 \pm 1.4 $ & $ 99.9 \pm 0.1 $ & $ 100.0 - 0.0 $ & $ 100.0 - 0.0 $ \\
 & & CV &  &  $ 2.7 \pm 0.8 $ & $ 7.3 \pm 1.3 $ & $ 14.2 \pm 1.8 $ & $ 37.5 \pm 2.5 $ & $ 77.3 \pm 2.1 $ & $ 99.9 \pm 0.2 $ & $ 100.0 - 0.0 $  \\
\hhline{---~-------}

\multirow{6}{*}{GB}  & \multirow{2}{*}{II} & \texttt{MD} &  &  $ 0.8 \pm 0.5 $ & $ 2.8 \pm 0.8 $ & $ 6.6 \pm 1.3 $ & $ 18.9 \pm 2.0 $ & $ 49.5 \pm 2.5 $ & $ 93.2 \pm 1.3 $ & $ 100.0 - 0.0 $ \\
 & & CV &  &  $ 2.6 \pm 0.8 $ & $ 4.0 \pm 1.0 $ & $ 3.4 \pm 0.9 $ & $ 6.8 \pm 1.3 $ & $ 11.1 \pm 1.6 $ & $ 30.4 \pm 2.3 $ & $ 58.9 \pm 2.5 $\\
  & \multirow{2}{*}{IV} & \texttt{MD} &  &  $ 1.8 \pm 0.7 $ & $ 5.7 \pm 1.2 $ & $ 19.5 \pm 2.0 $ & $ 49.7 \pm 2.5 $ & $ 89.6 \pm 1.5 $ & $ 99.9 \pm 0.1 $ & $ 100.0 - 0.0 $ \\
 & & CV &  &  $ 2.5 \pm 0.8 $ & $ 6.4 \pm 1.2 $ & $ 7.3 \pm 1.3 $ & $ 13.9 \pm 1.7 $ & $ 35.9 \pm 2.4 $ & $ 84.0 \pm 1.9 $ & $ 99.8 \pm 0.2 $\\
   & \multirow{2}{*}{VI} & \texttt{MD} &  &  $ 2.1 \pm 0.7 $ & $ 6.3 \pm 1.2 $ & $ 23.4 \pm 2.1 $ & $ 62.5 \pm 2.4 $ & $ 96.1 \pm 1.0 $ & $ 100.0 - 0.0 $ & $ 100.0 - 0.0 $ \\
 & & CV &  &  $ 3.0 \pm 0.9 $ & $ 8.9 \pm 1.4 $ & $ 14.7 \pm 1.8 $ & $ 44.2 \pm 2.5 $ & $ 85.5 \pm 1.8 $ & $ 100.0 - 0.0 $ & $ 100.0 - 0.0 $  \\
\hhline{---~-------}

\multirow{6}{*}{IS}  & \multirow{2}{*}{II} & \texttt{MD} &  &  $ 0.3 \pm 0.3 $ & $ 1.3 \pm 0.6 $ & $ 3.6 \pm 0.9 $ & $ 6.7 \pm 1.3 $ & $ 19.6 \pm 2.0 $ & $ 60.1 \pm 2.5 $ & $ 92.6 \pm 1.3 $ \\
 & & CV &  & $ 2.5 \pm 0.8 $ & $ 3.3 \pm 0.9 $ & $ 2.5 \pm 0.8 $ & $ 4.4 \pm 1.0 $ & $ 7.9 \pm 1.4 $ & $ 16.5 \pm 1.9 $ & $ 31.5 \pm 2.4 $\\
  & \multirow{2}{*}{IV} & \texttt{MD} &  &  $ 0.6 \pm 0.4 $ & $ 2.4 \pm 0.8 $ & $ 3.9 \pm 1.0 $ & $ 16.1 \pm 1.9 $ & $ 39.9 \pm 2.5 $ & $ 88.2 \pm 1.6 $ & $ 99.9 \pm 0.1 $ \\
 & & CV &  &  $ 2.0 \pm 0.7 $ & $ 3.9 \pm 1.0 $ & $ 2.9 \pm 0.9 $ & $ 6.7 \pm 1.3 $ & $ 10.8 \pm 1.6 $ & $ 25.4 \pm 2.2 $ & $ 53.1 \pm 2.5 $ \\
   & \multirow{2}{*}{VI} & \texttt{MD} &  &  $ 1.3 \pm 0.6 $ & $ 3.9 \pm 1.0 $ & $ 8.9 \pm 1.4 $ & $ 22.8 \pm 2.1 $ & $ 57.4 \pm 2.5 $ & $ 97.7 \pm 0.8 $ & $ 100.0 - 0.0 $\\
 & & CV &  &  $ 2.0 \pm 0.7 $ & $ 4.6 \pm 1.1 $ & $ 4.4 \pm 1.0 $ & $ 9.5 \pm 1.5 $ & $ 15.5 \pm 1.8 $ & $ 44.5 \pm 2.5 $ & $ 83.2 \pm 1.9 $ \\
\hline
\end{tabular}
\caption{Power of the statistical test with \statname{} (abbreviated as \texttt{MD}) and CV representations for the shift types Gaussian noise (GN), Gaussian blur (GB) and Image shift (IS), three different shift intensities (II, IV, VI) and fixed $\delta = 0.5$ for the CIFAR-10 dataset. The estimated $95\%$-confidence intervals are indicated.}
\label{table:CIFAR10-wrt-sample-size}
\end{center}
}
\end{table*}

\begin{figure*}[h]
    \centering
    \includegraphics[width=0.32\textwidth]{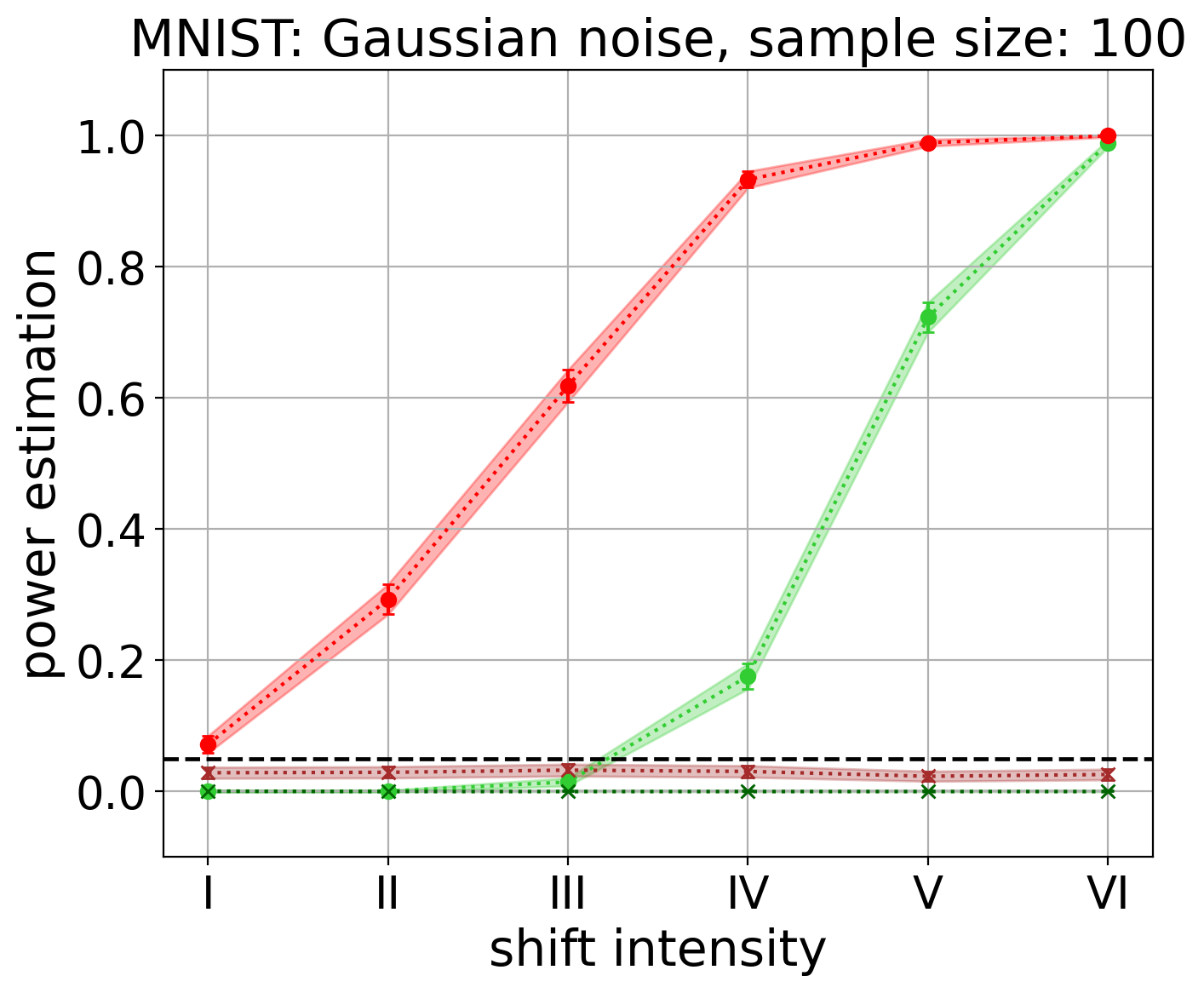}
    \includegraphics[width=0.32\textwidth]{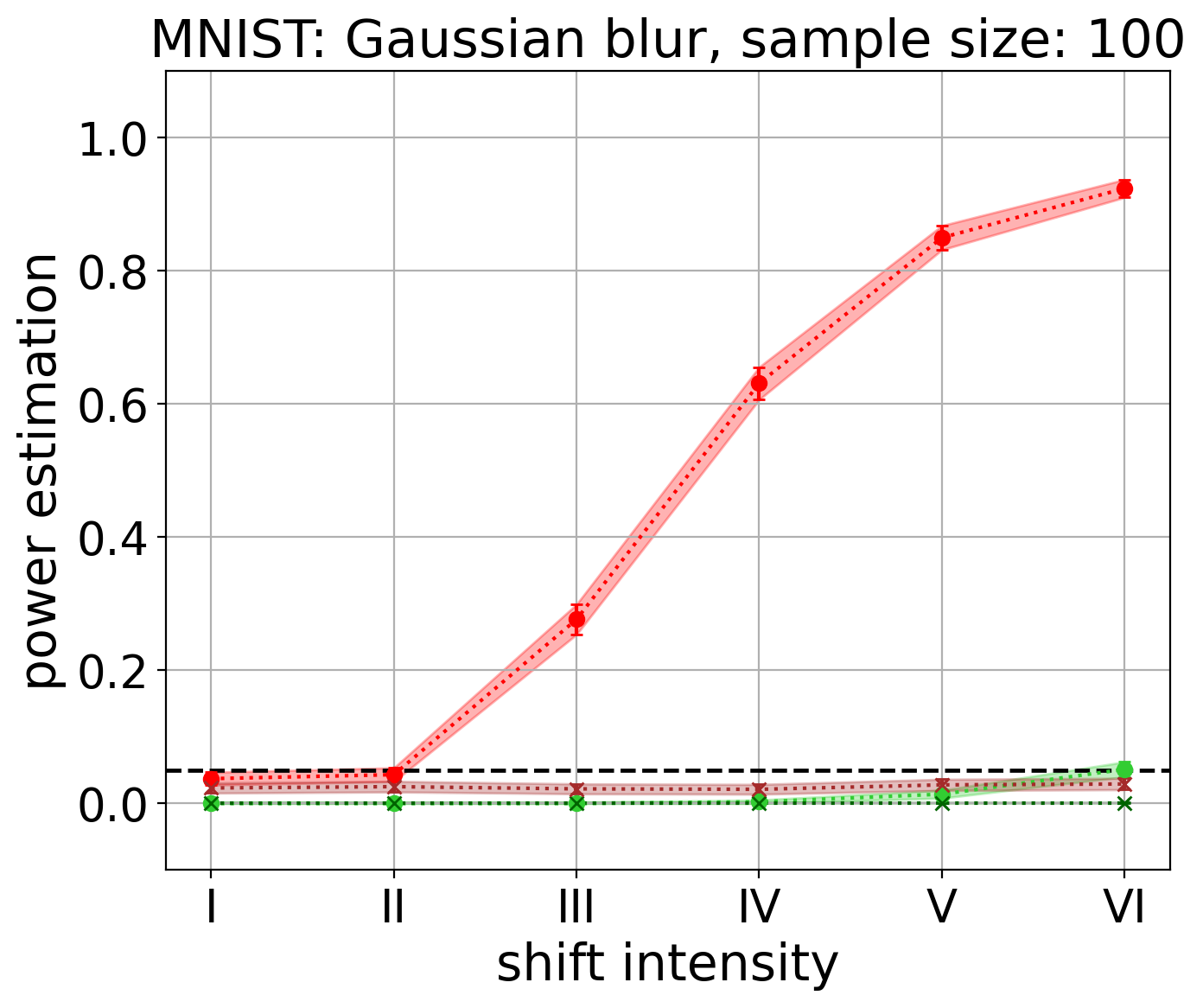}
   \includegraphics[width=0.32\textwidth]{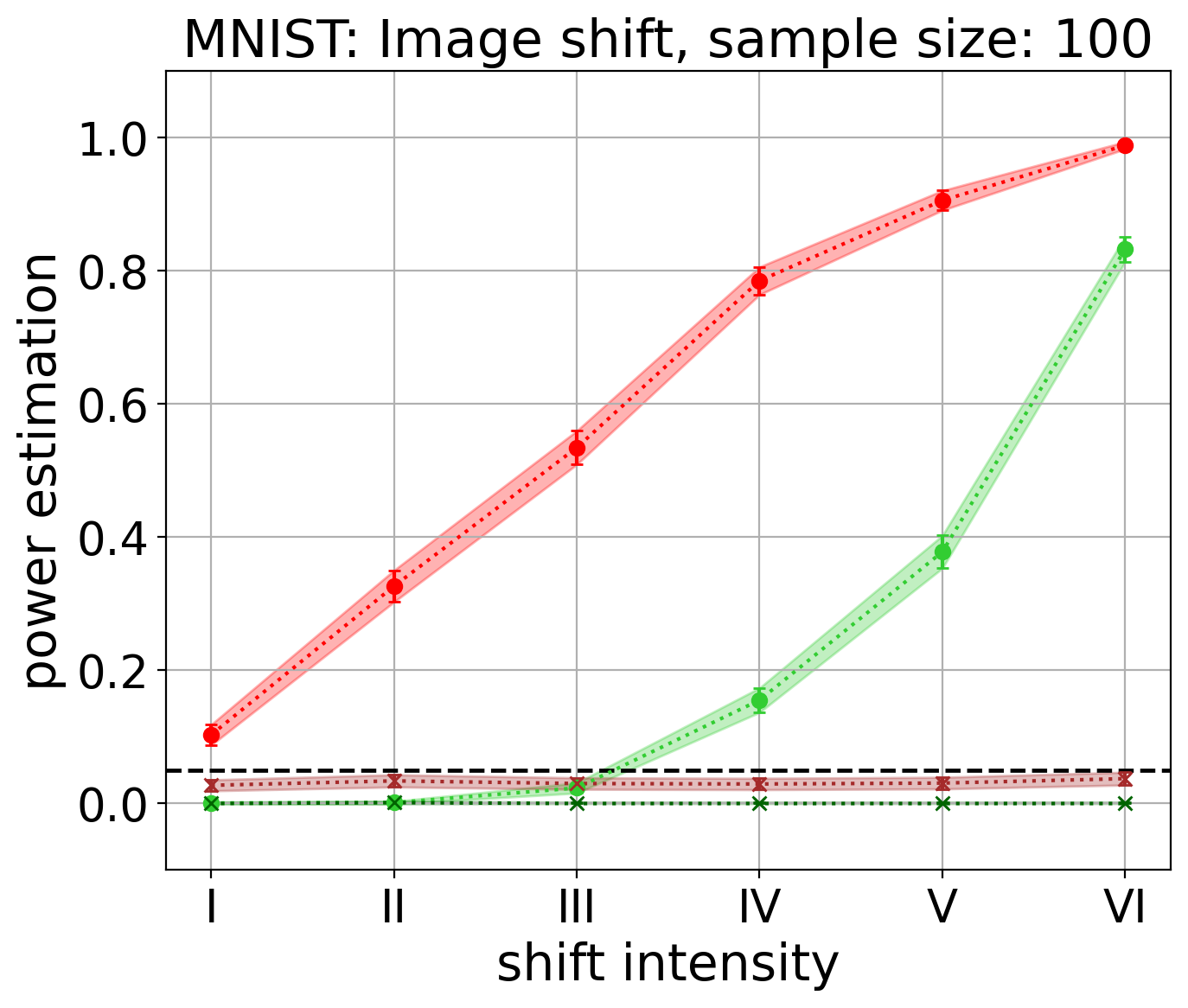}
    \includegraphics[width=0.32\textwidth]{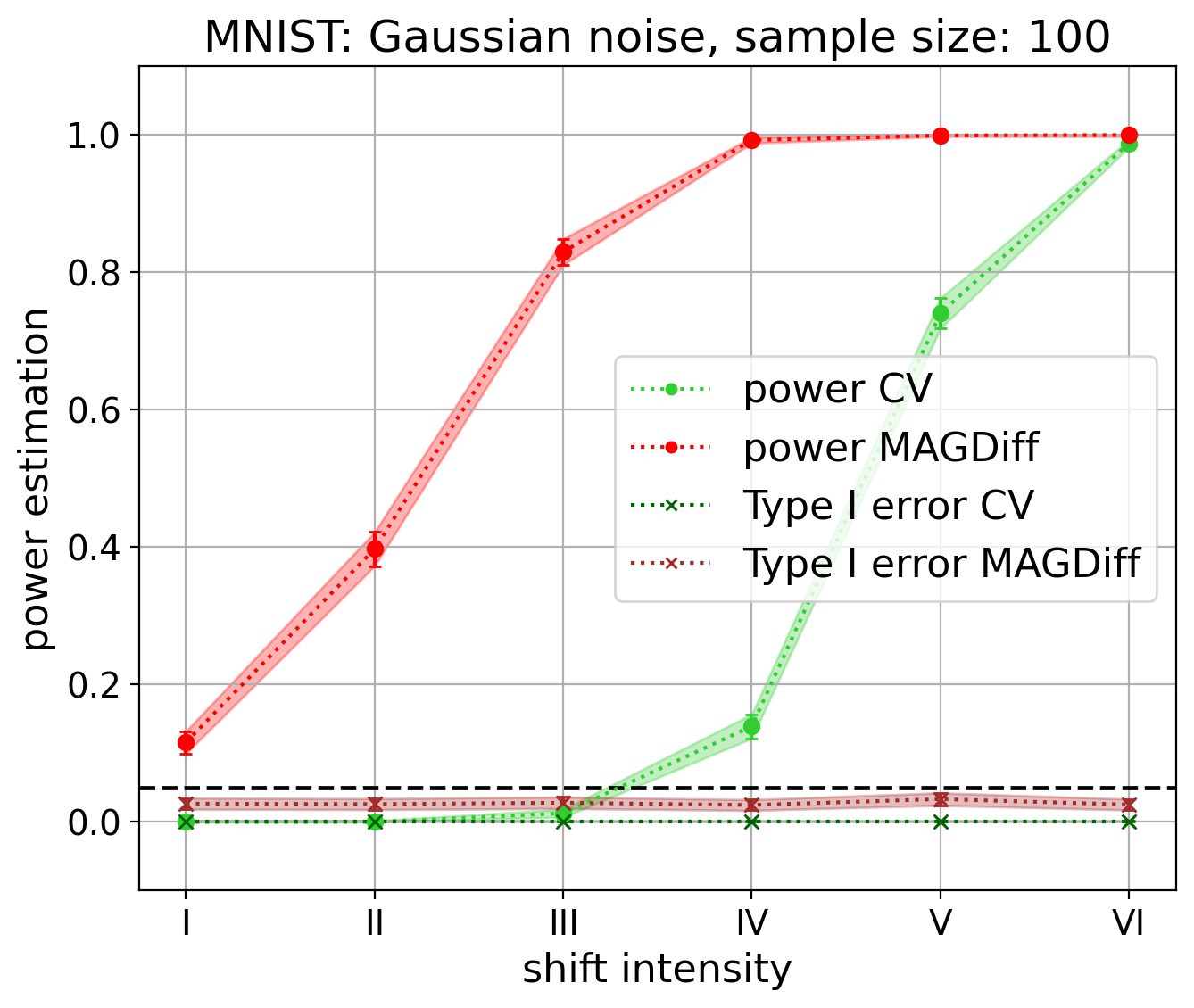}
    \includegraphics[width=0.32\textwidth]{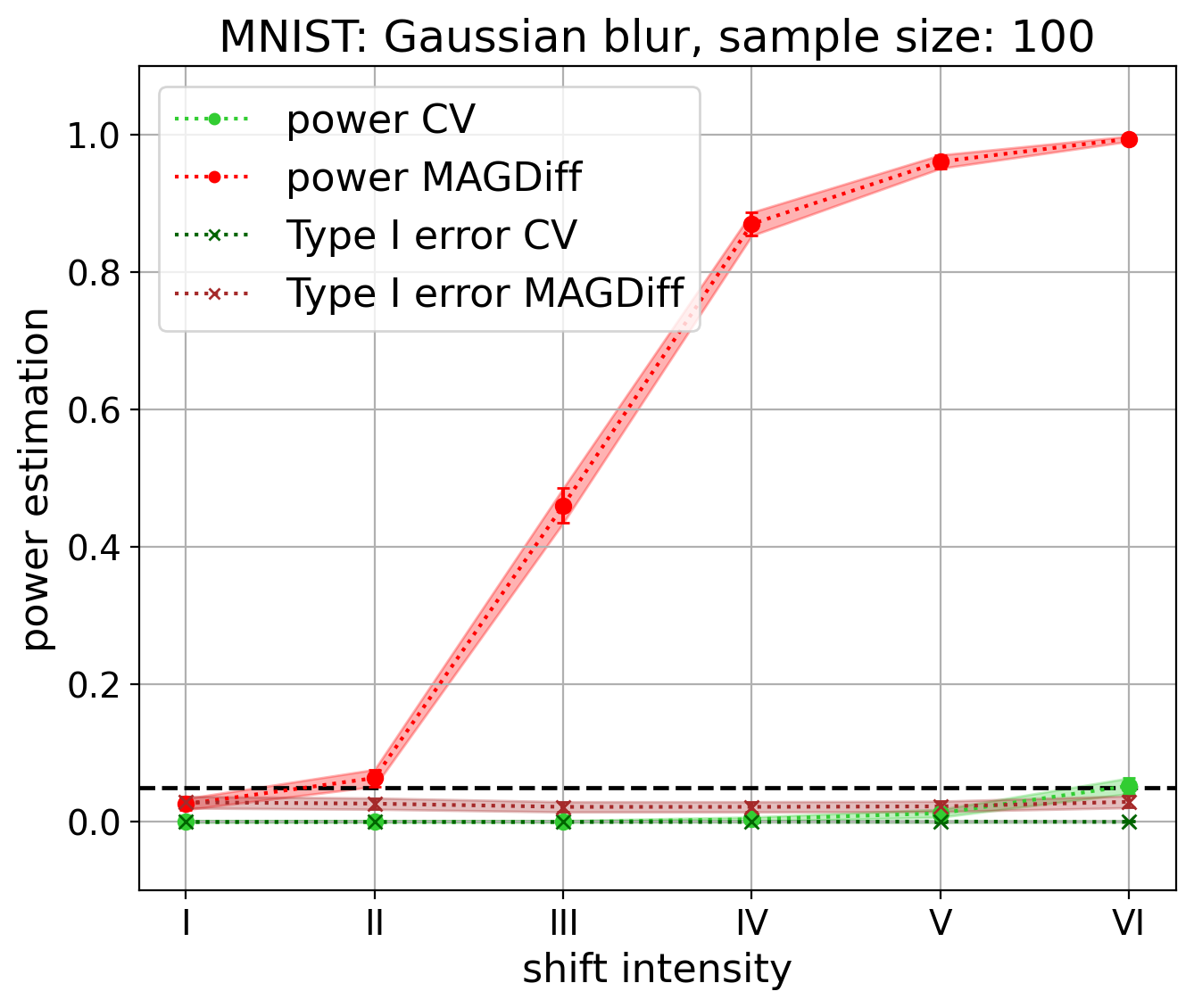}
   \includegraphics[width=0.32\textwidth]{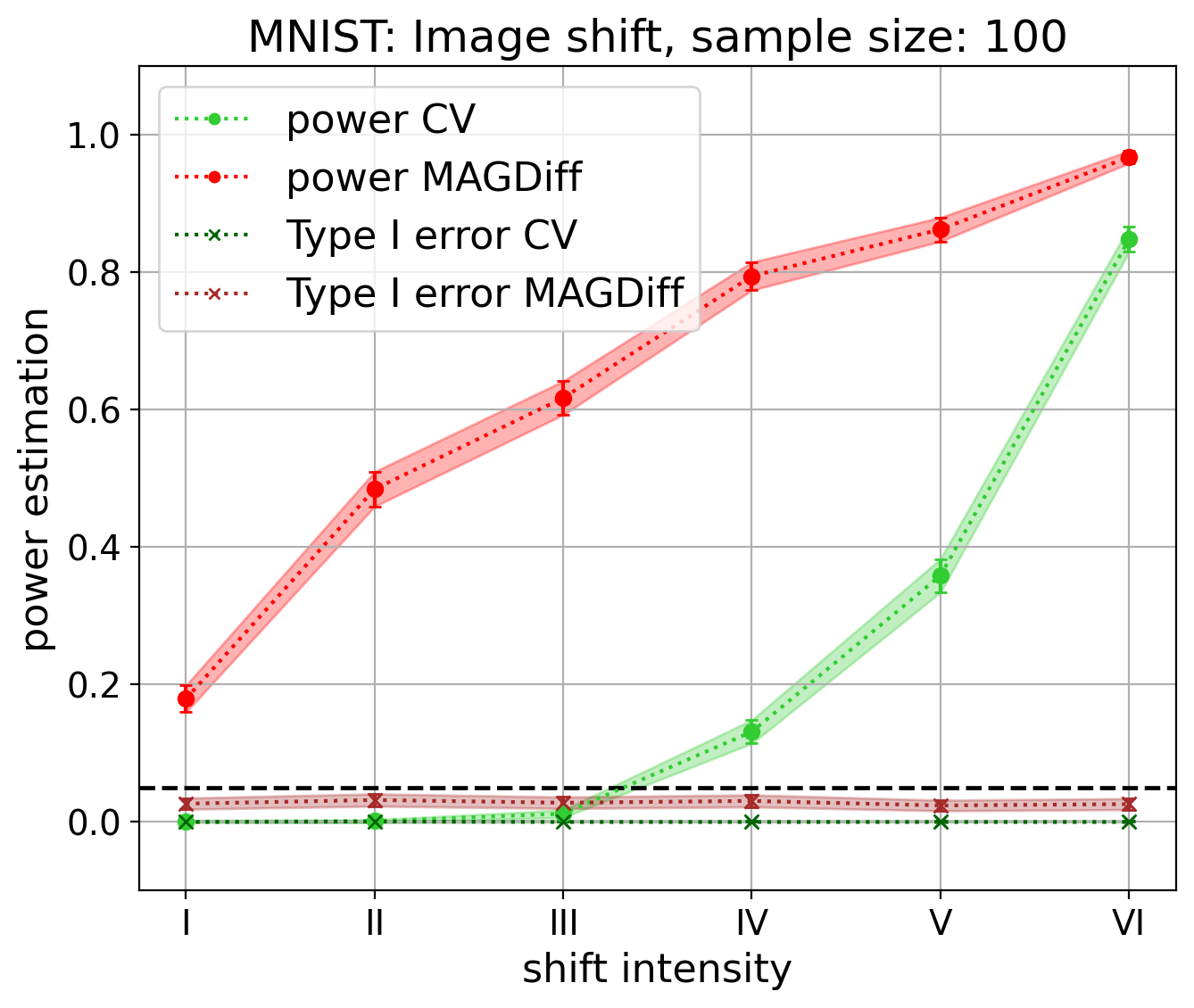}
   \caption{Power and type I error of the test with \statname{} (red) and CV (green) representations w.r.t.~the shift intensity for various shift types on the MNIST dataset with $\delta = 0.5$, sample size $100$, for layers $\ell_{-1}$ (top row) and  $\ell_{-3}$ (bottom row).}
     \label{fig:shift_intensity_MNIST_supplementary}
\end{figure*}

\begin{figure*}[h]
    \centering
    \includegraphics[width=0.32\textwidth]{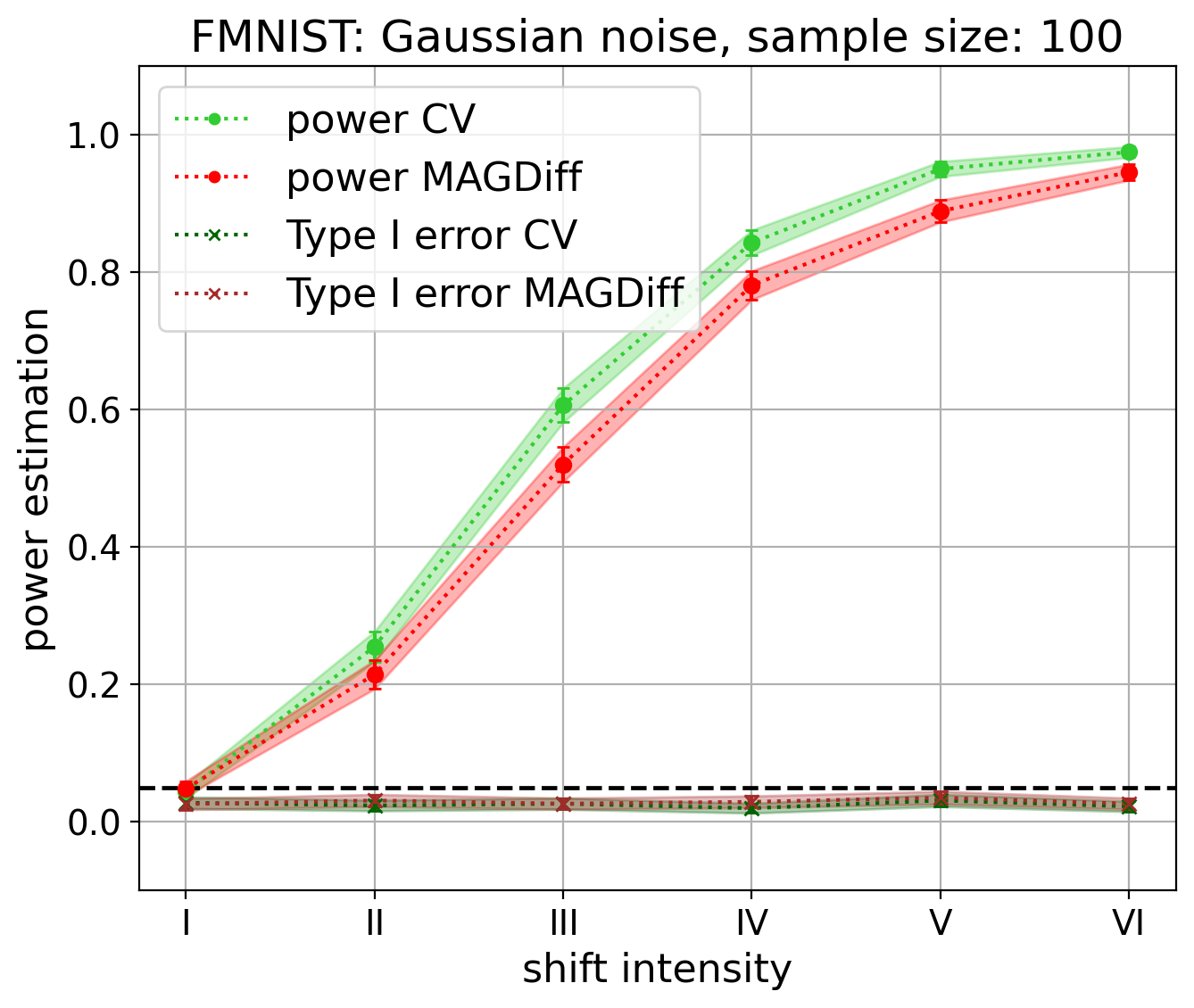}
    \includegraphics[width=0.32\textwidth]{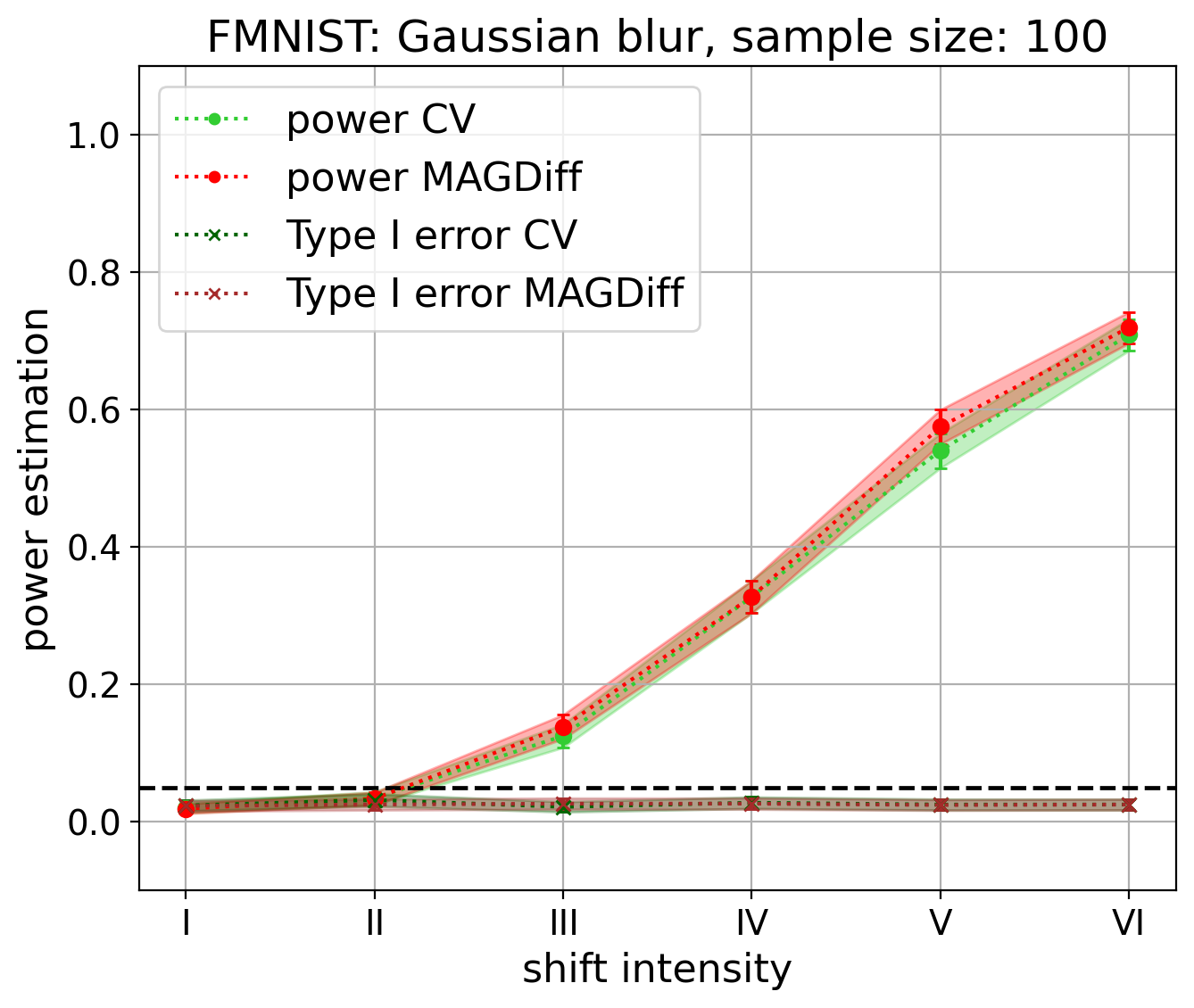}
   \includegraphics[width=0.32\textwidth]{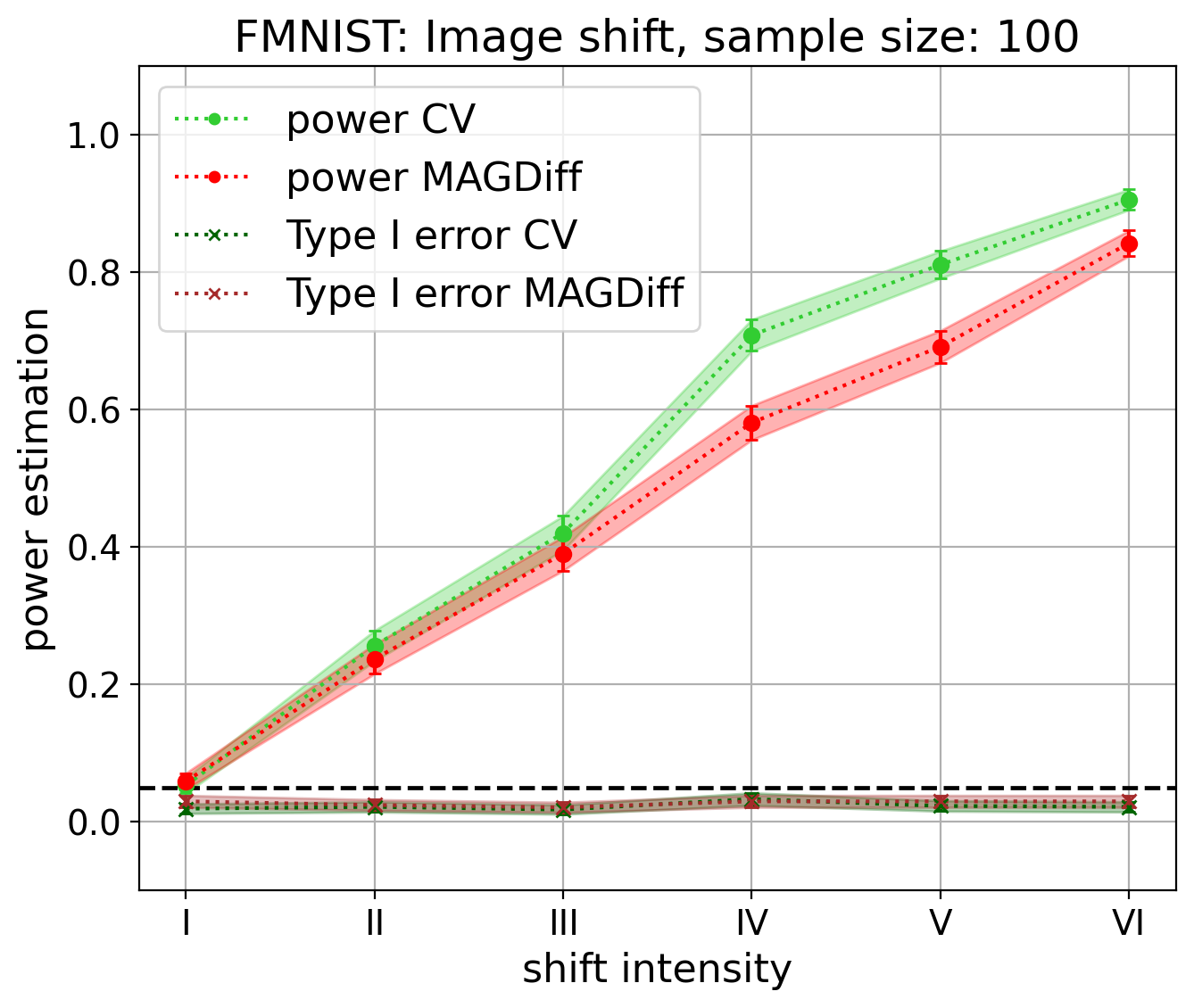}
    \includegraphics[width=0.32\textwidth]{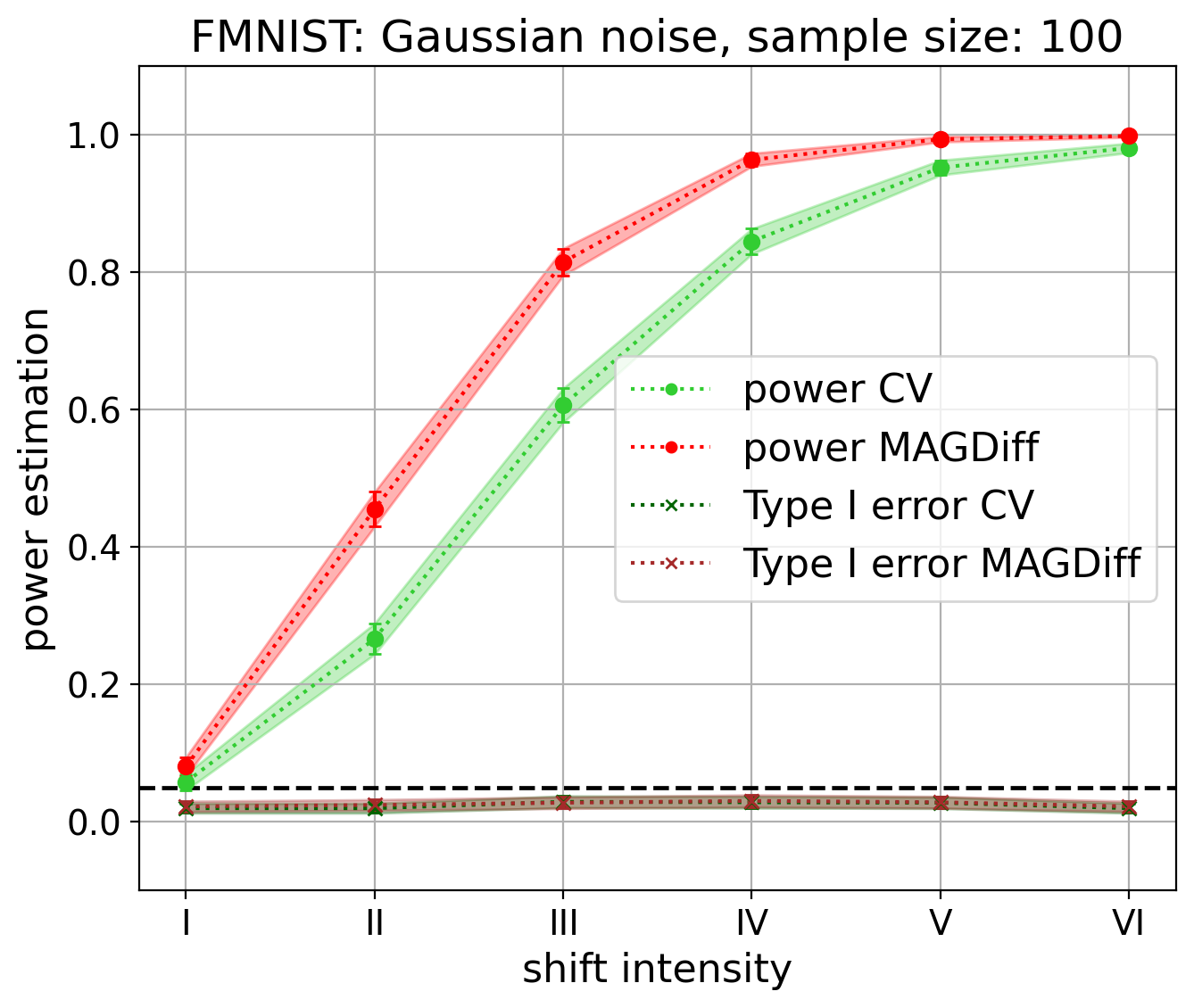}
    \includegraphics[width=0.32\textwidth]{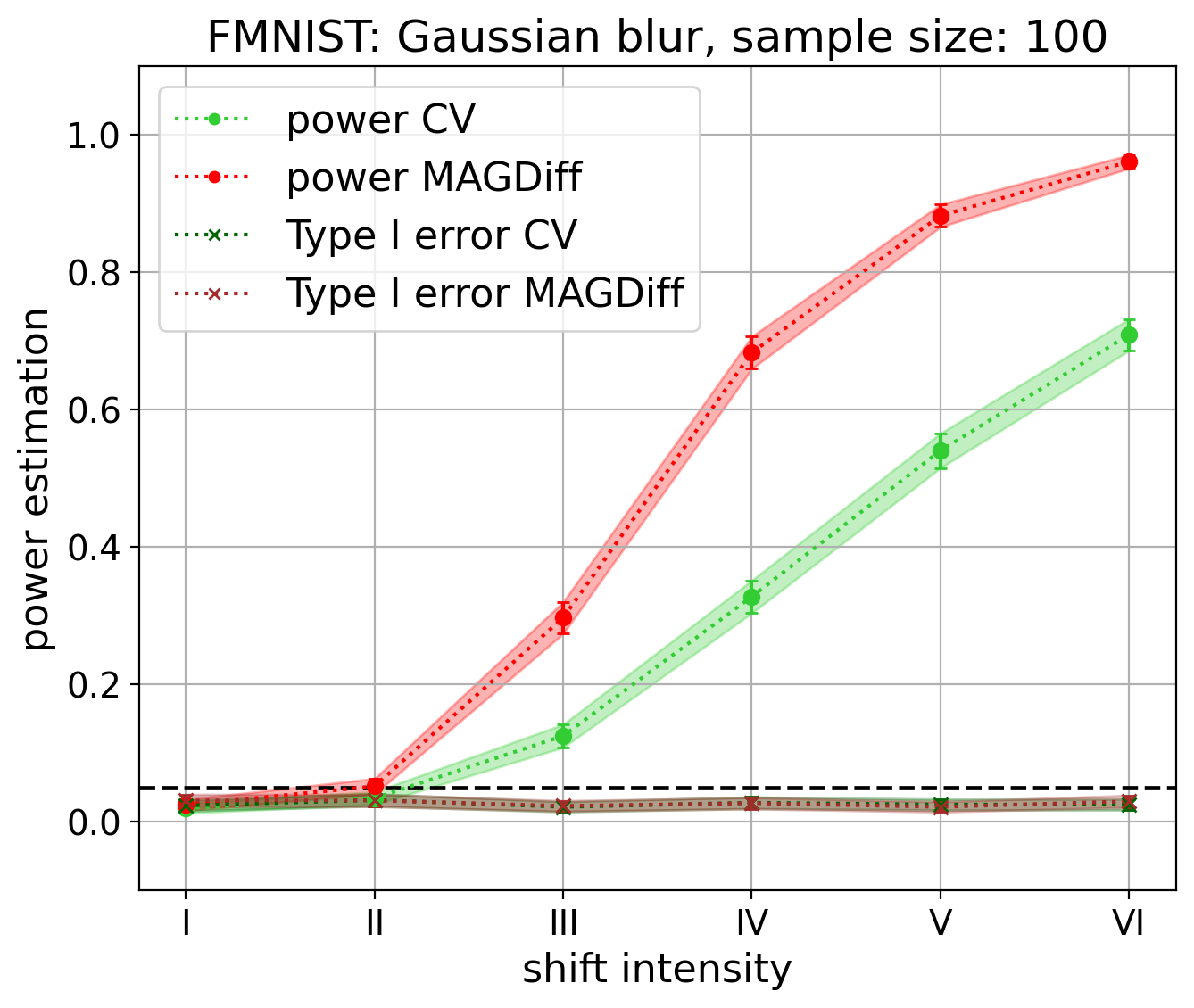}
   \includegraphics[width=0.32\textwidth]{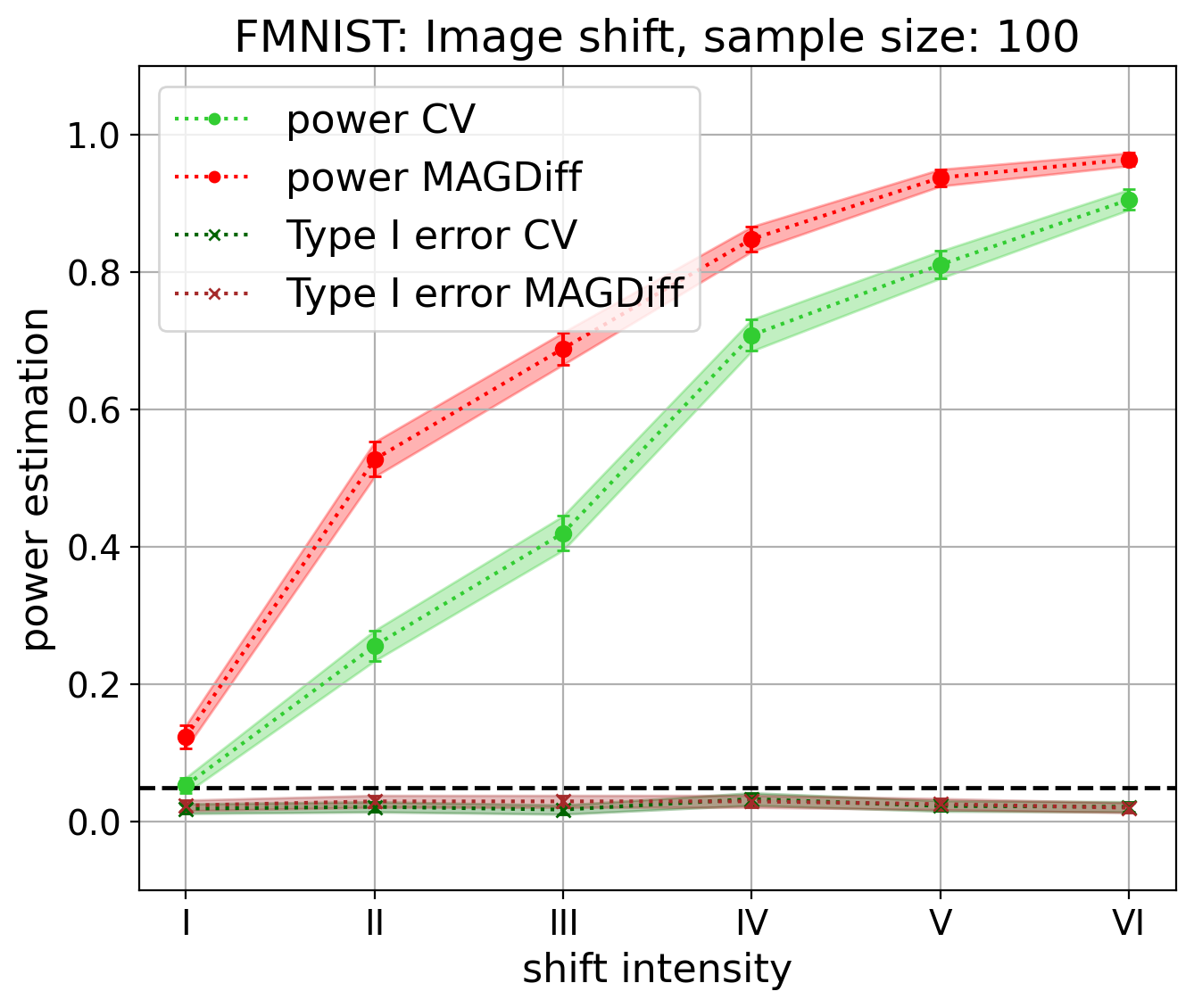}
   \caption{Power and type I error of the test with \statname~(red) and CV (green) representations w.r.t.~the shift intensity for various shift types on the FMNIST dataset with $\delta = 0.5$, sample size $100$, for layers $\ell_{-1}$ (top row) and $\ell_{-3}$ (bottom row).}
     \label{fig:shift_intensity_FMNIST_supplementary}
\end{figure*}

\begin{figure*}[h]
\centering
    \includegraphics[width=0.32\textwidth]{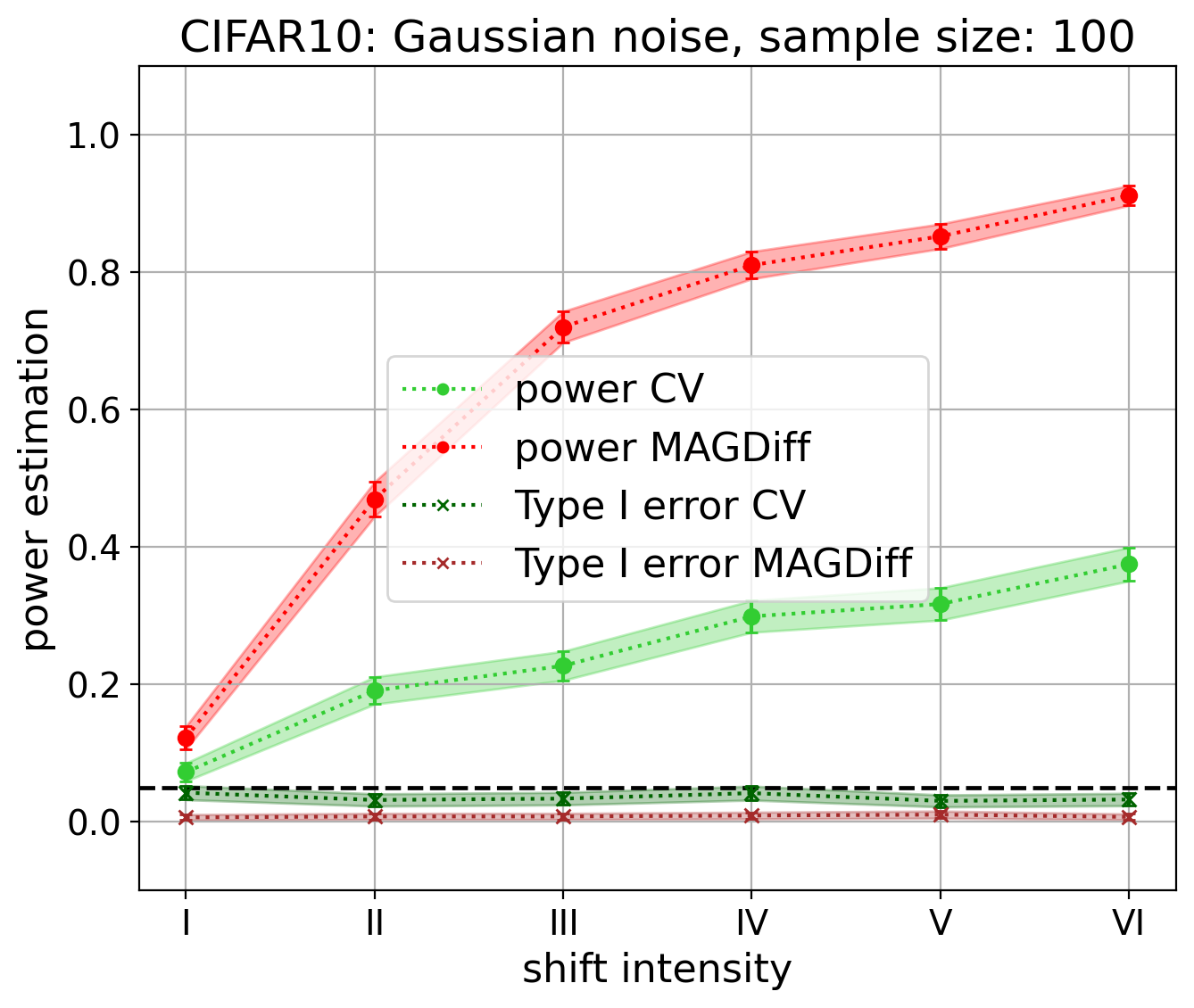}
    \includegraphics[width=0.32\textwidth]{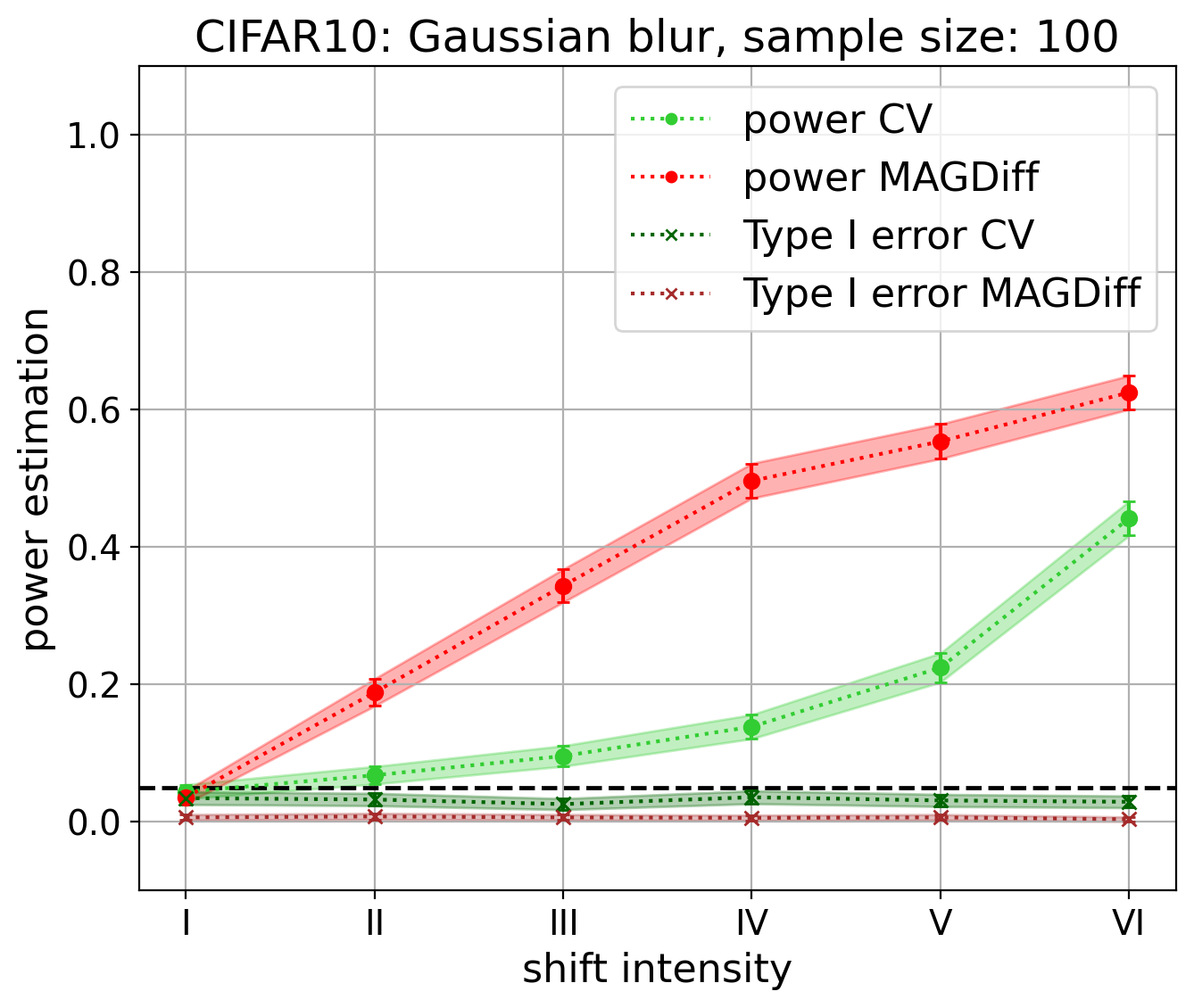}
   \includegraphics[width=0.32\textwidth]{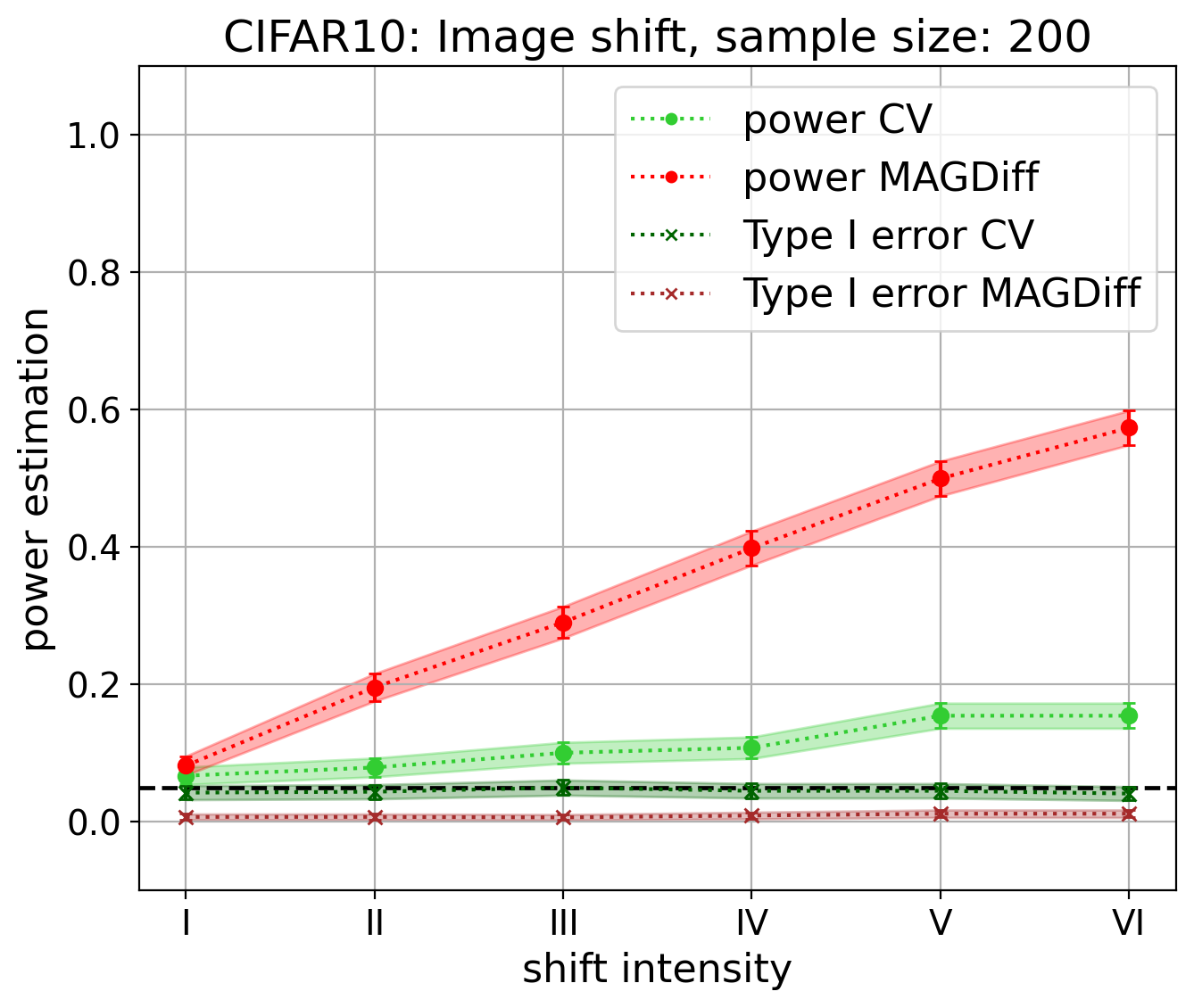}
   \includegraphics[width=0.32\textwidth]{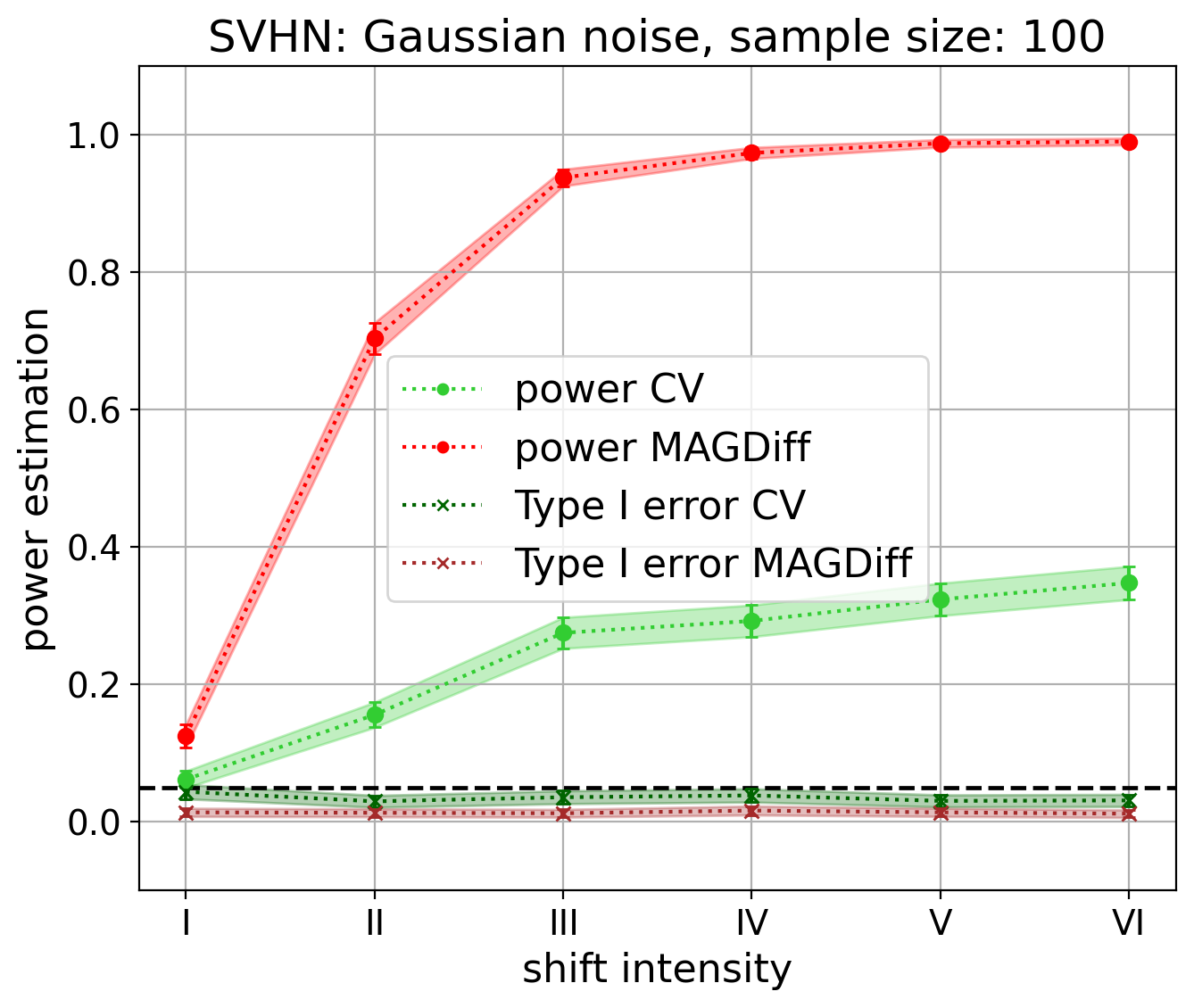}
    \includegraphics[width=0.32\textwidth]{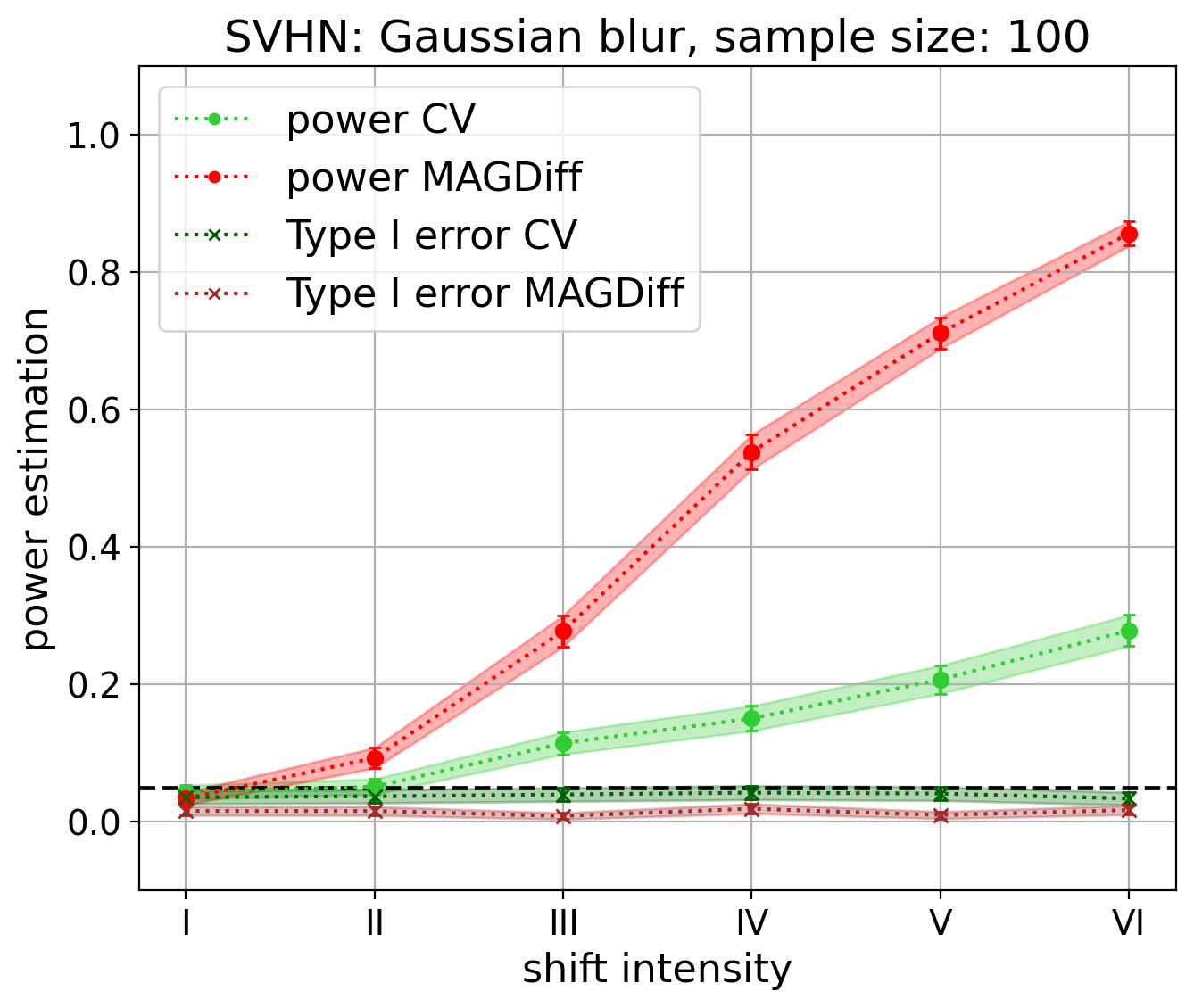}
   \includegraphics[width=0.32\textwidth]{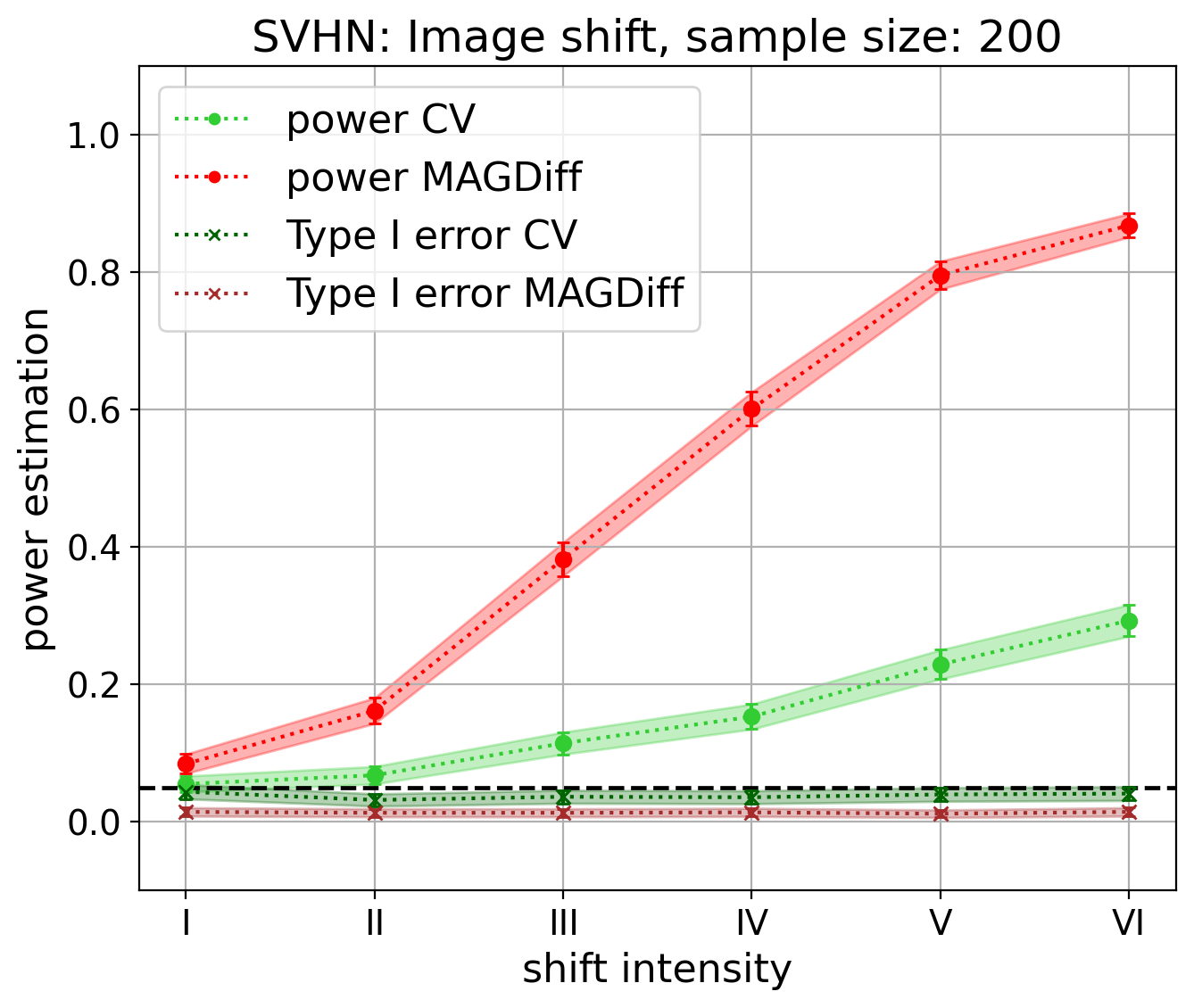}
    \includegraphics[width=0.325\textwidth]{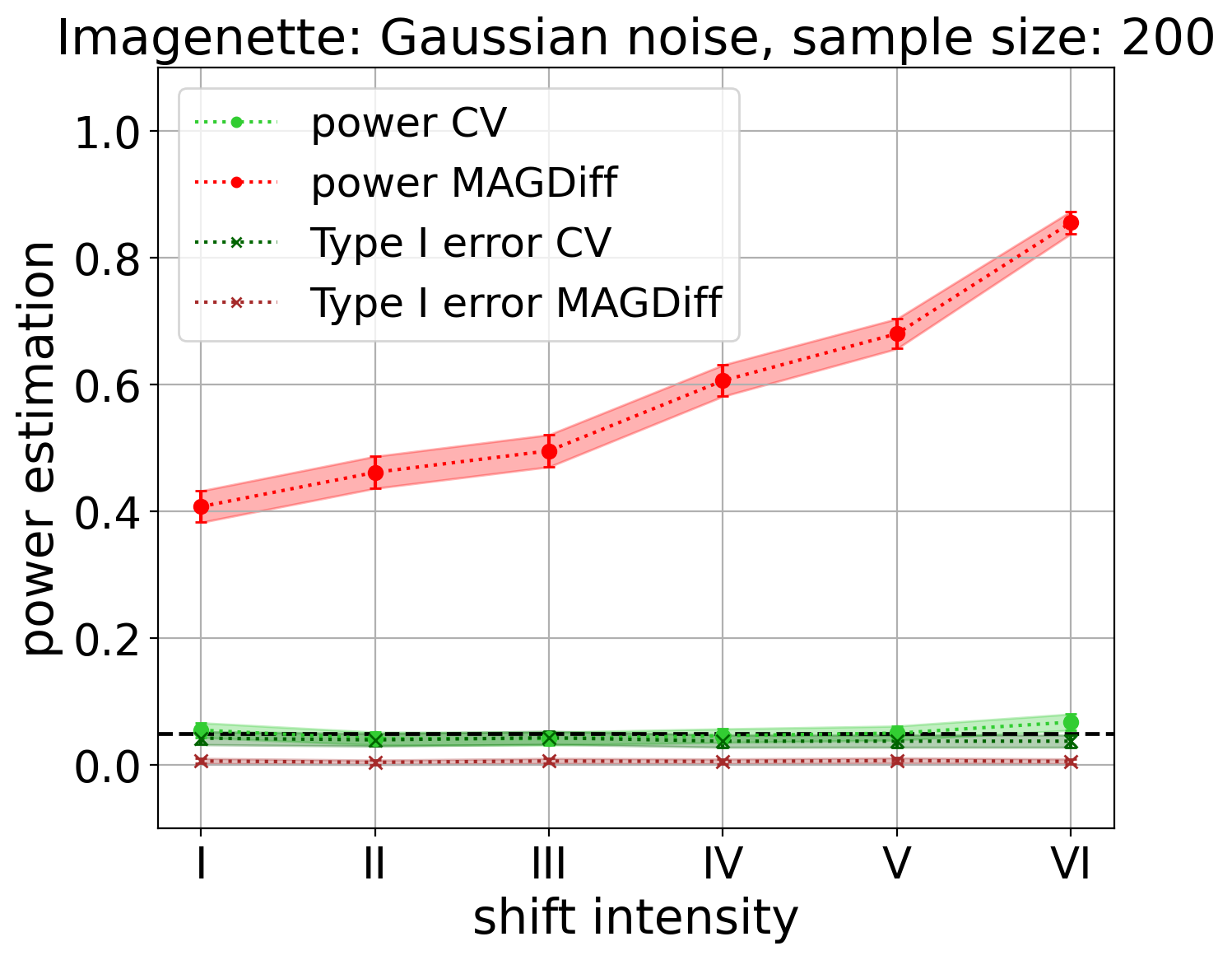}
    \includegraphics[width=0.325\textwidth]{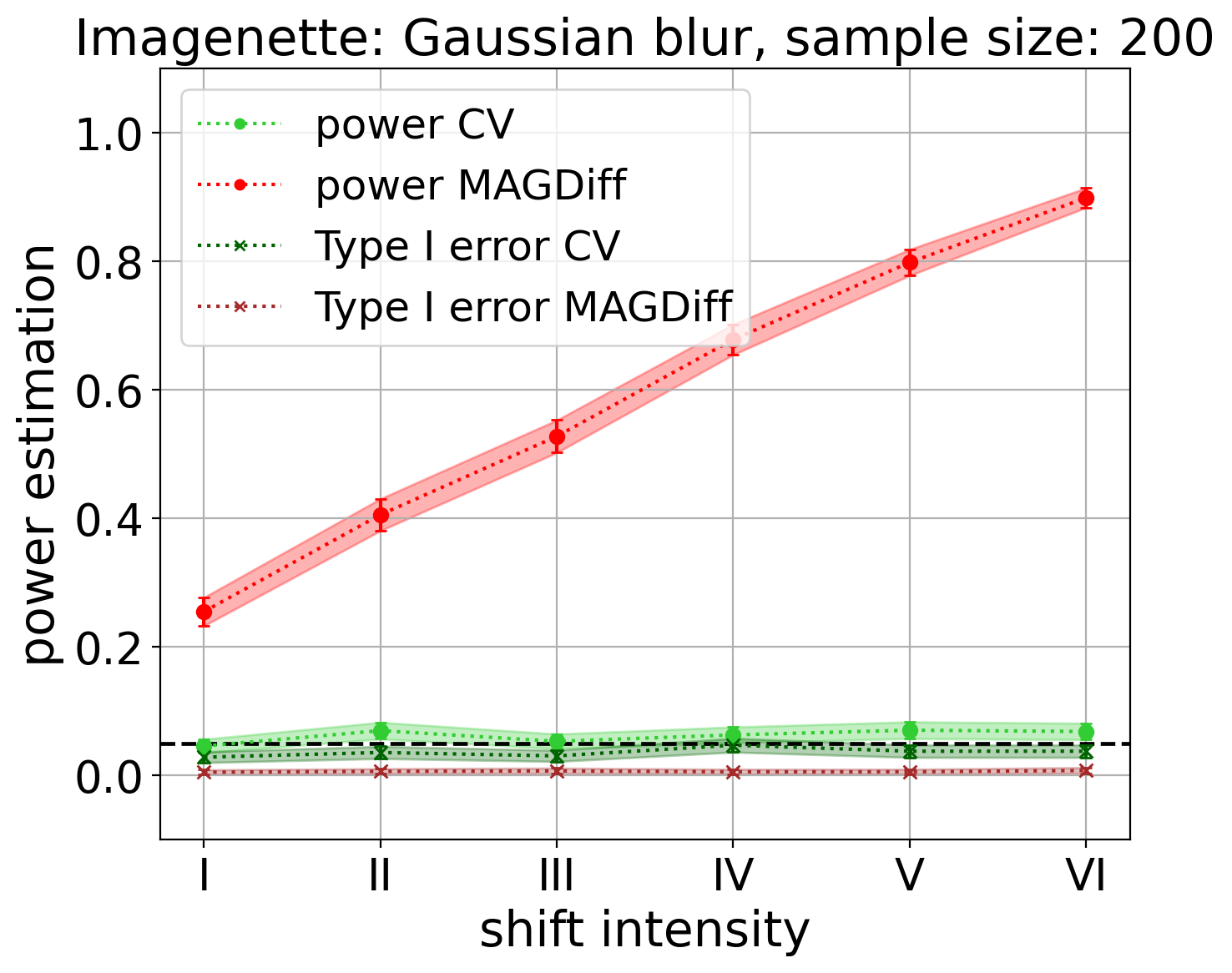}
   \includegraphics[width=0.325\textwidth]{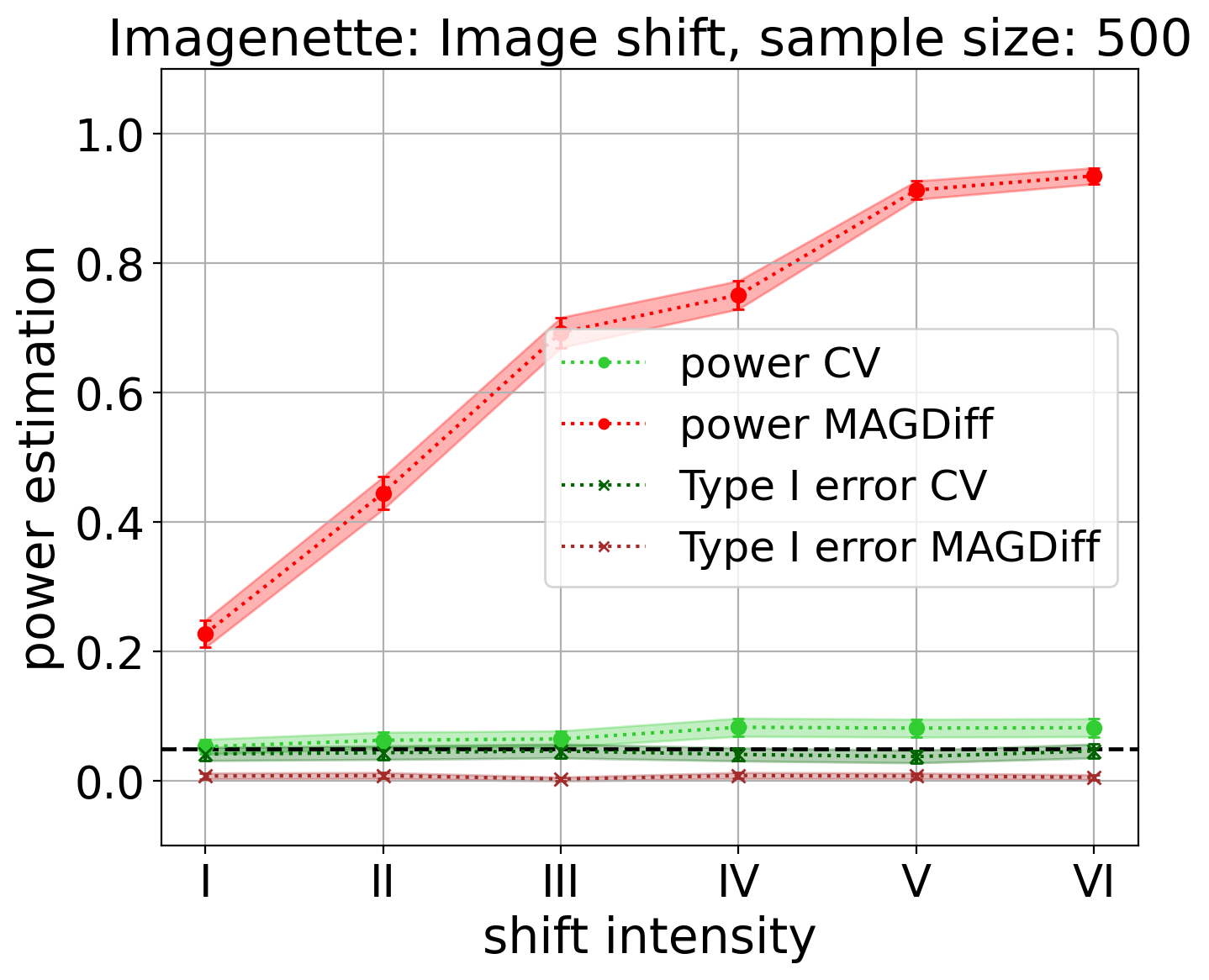}
    \caption{Power and type I error of the test with \statname~(red) and CV (green) representations w.r.t.~the shift intensity for various shift types on the CIFAR-10, SVHN (with $\delta = 0.5$) and Imagenette (with $\delta = 1$) datasets. Sample sizes and values of $\delta$ were chosen to make the plots as expressive as possible (low power for low shift intensity, high power for high shift intensity), as the difficulty of the task varies depending on the shift type and dataset.}
     \label{fig:shift_intensity_CIFAR10_SVHN_Imagenette_supplementary}
\end{figure*}

   \begin{figure*}[h]
    \centering
    \includegraphics[width=0.32\textwidth]{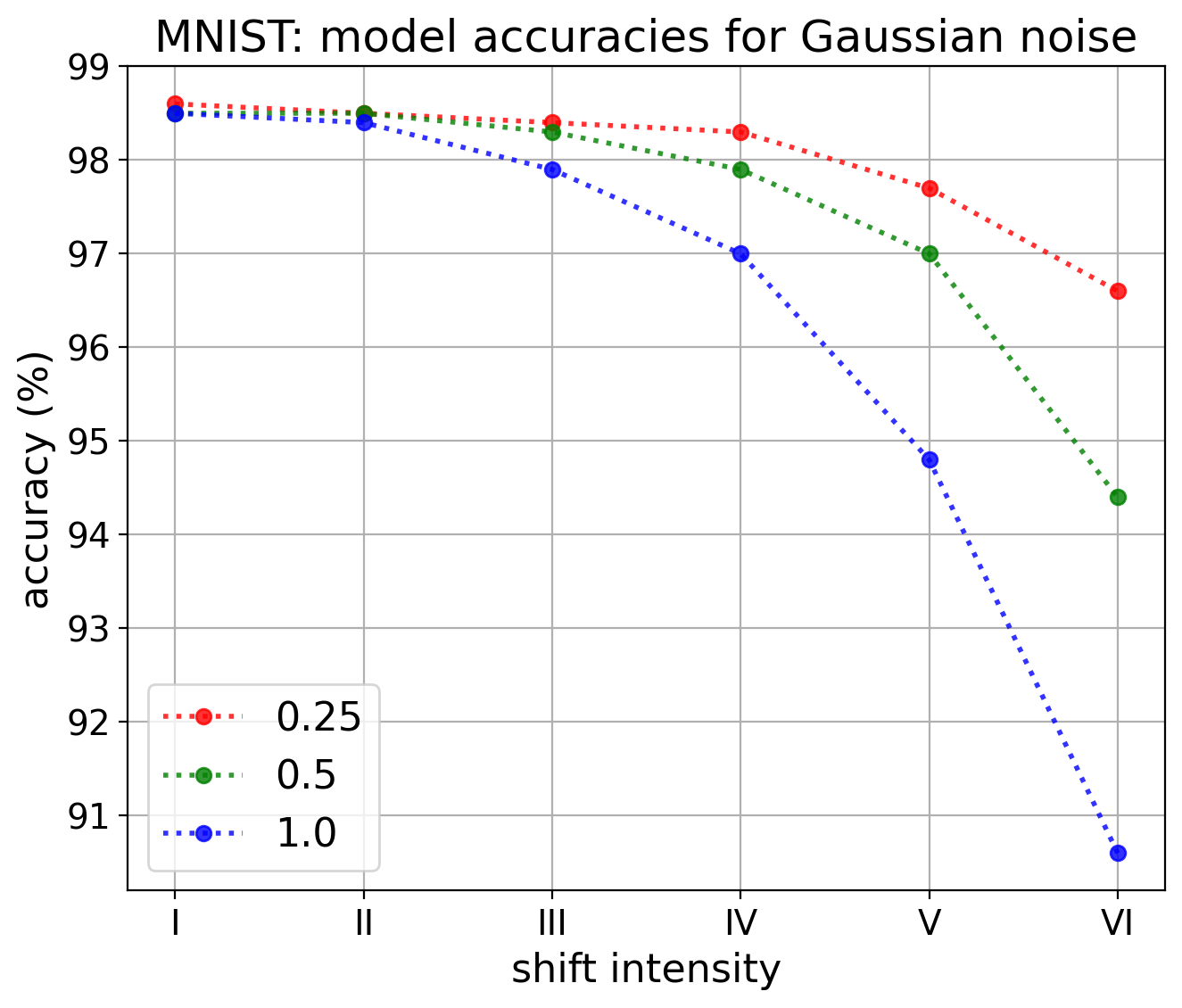}
    \includegraphics[width=0.335\textwidth]{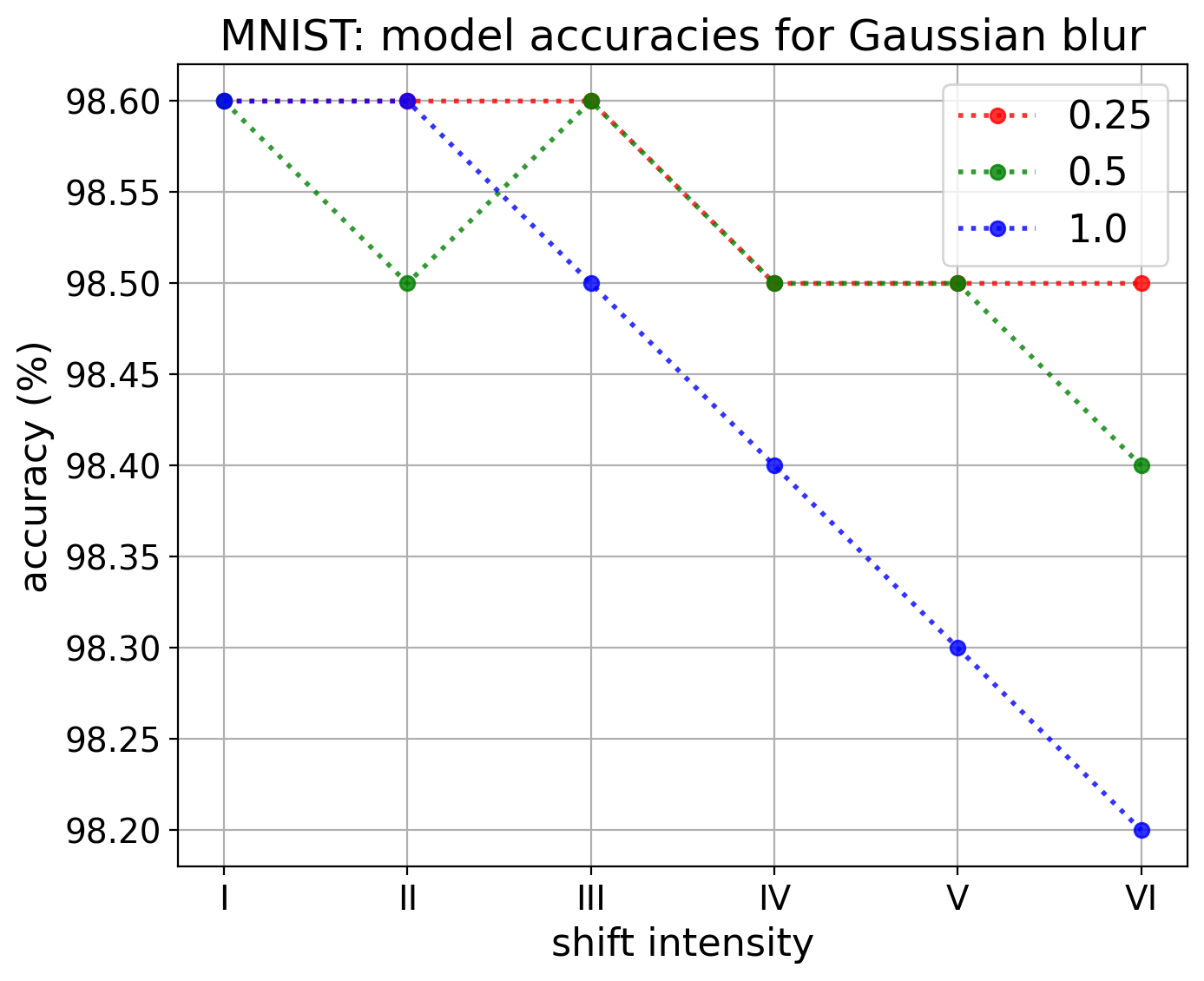}
   \includegraphics[width=0.32\textwidth]{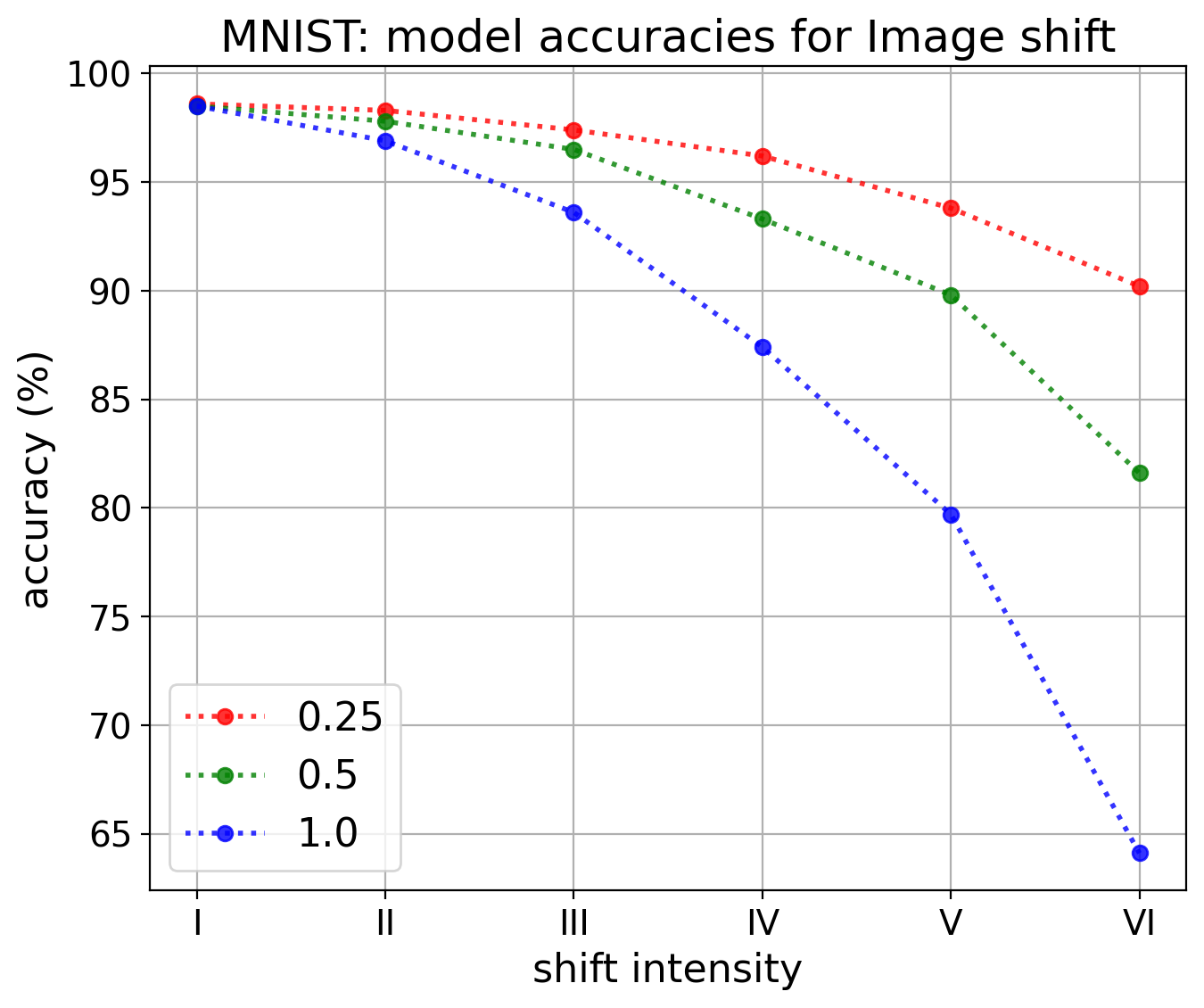}
   \includegraphics[width=0.32\textwidth]{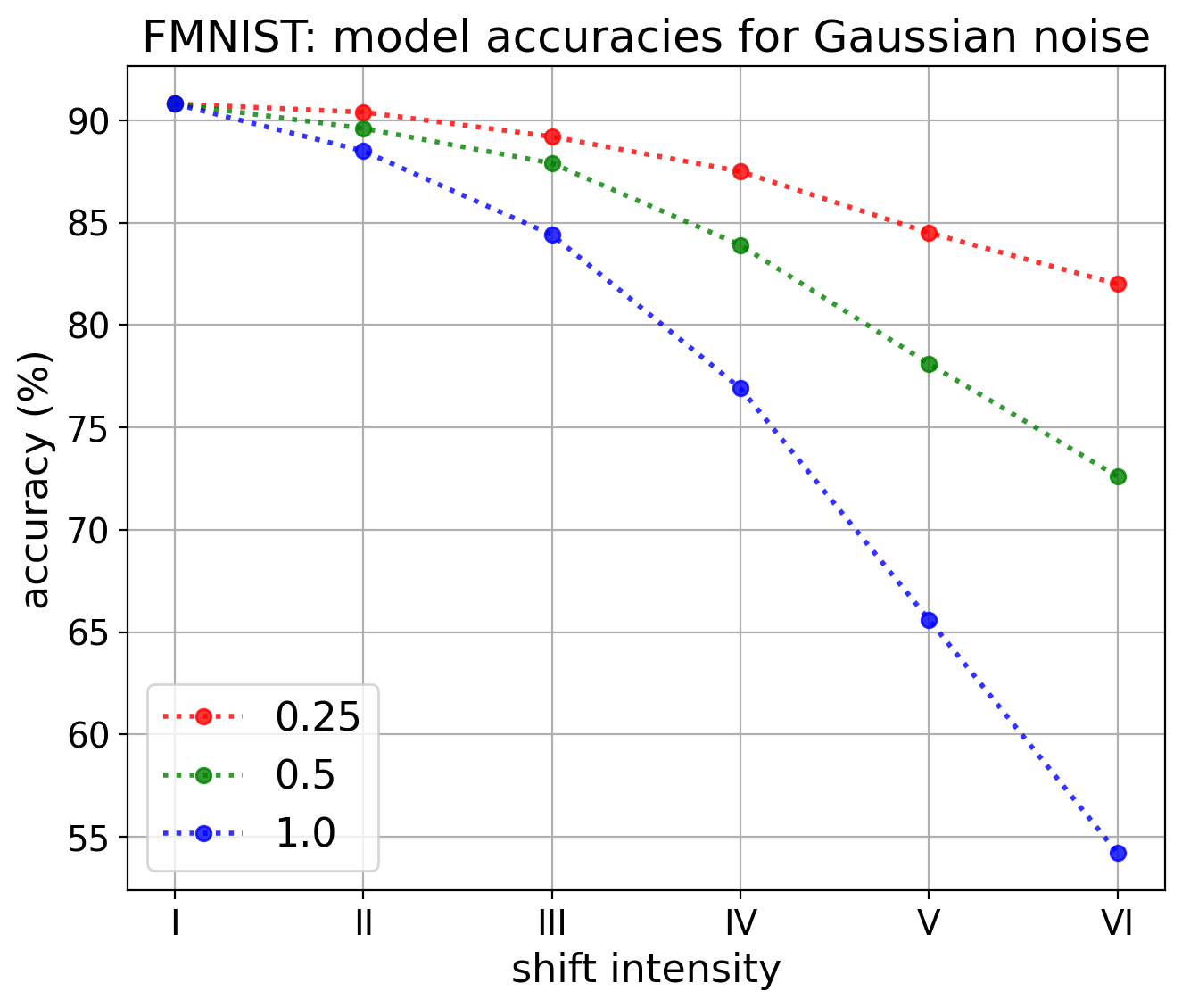}
    \includegraphics[width=0.33\textwidth]{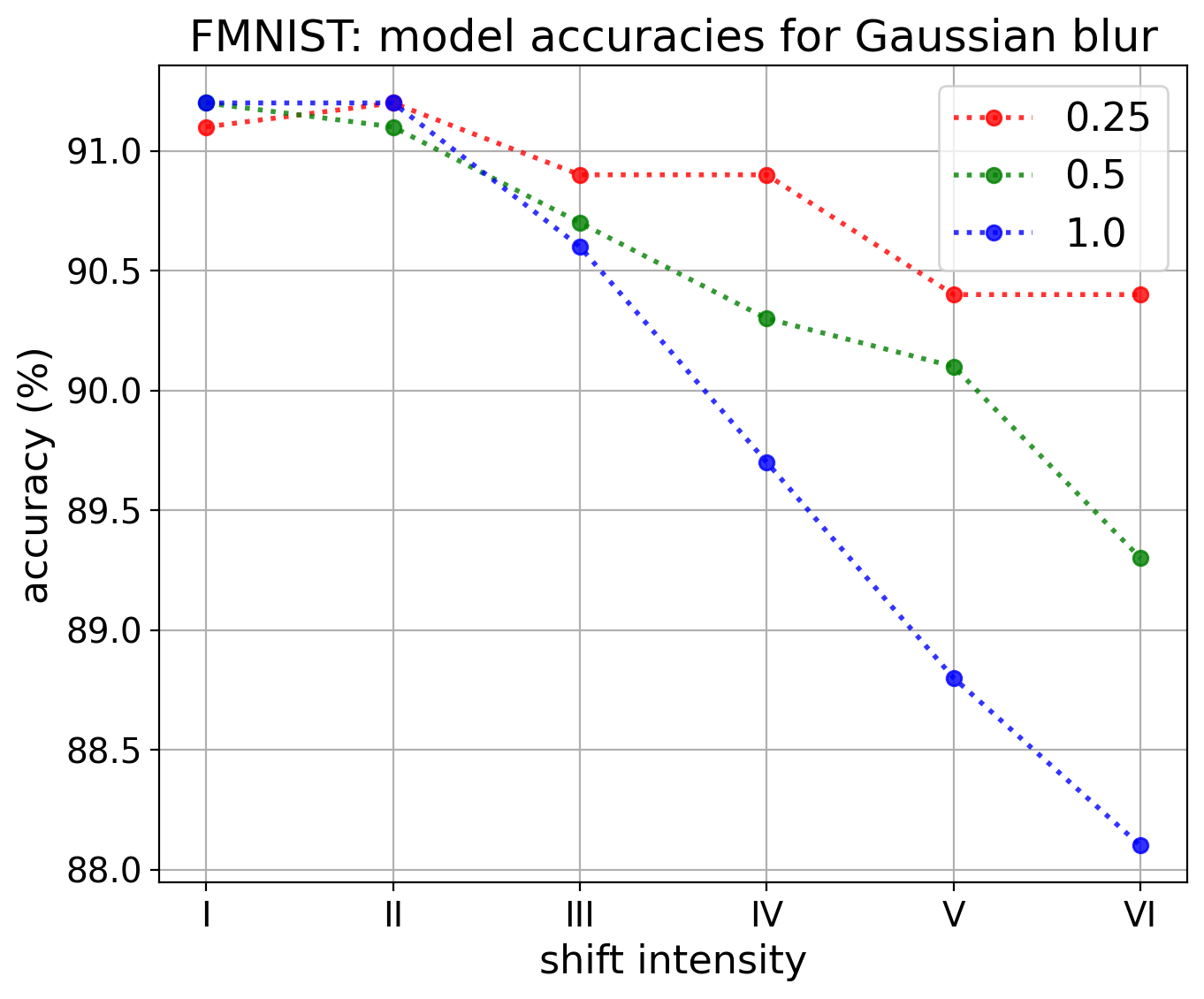}
   \includegraphics[width=0.32\textwidth]{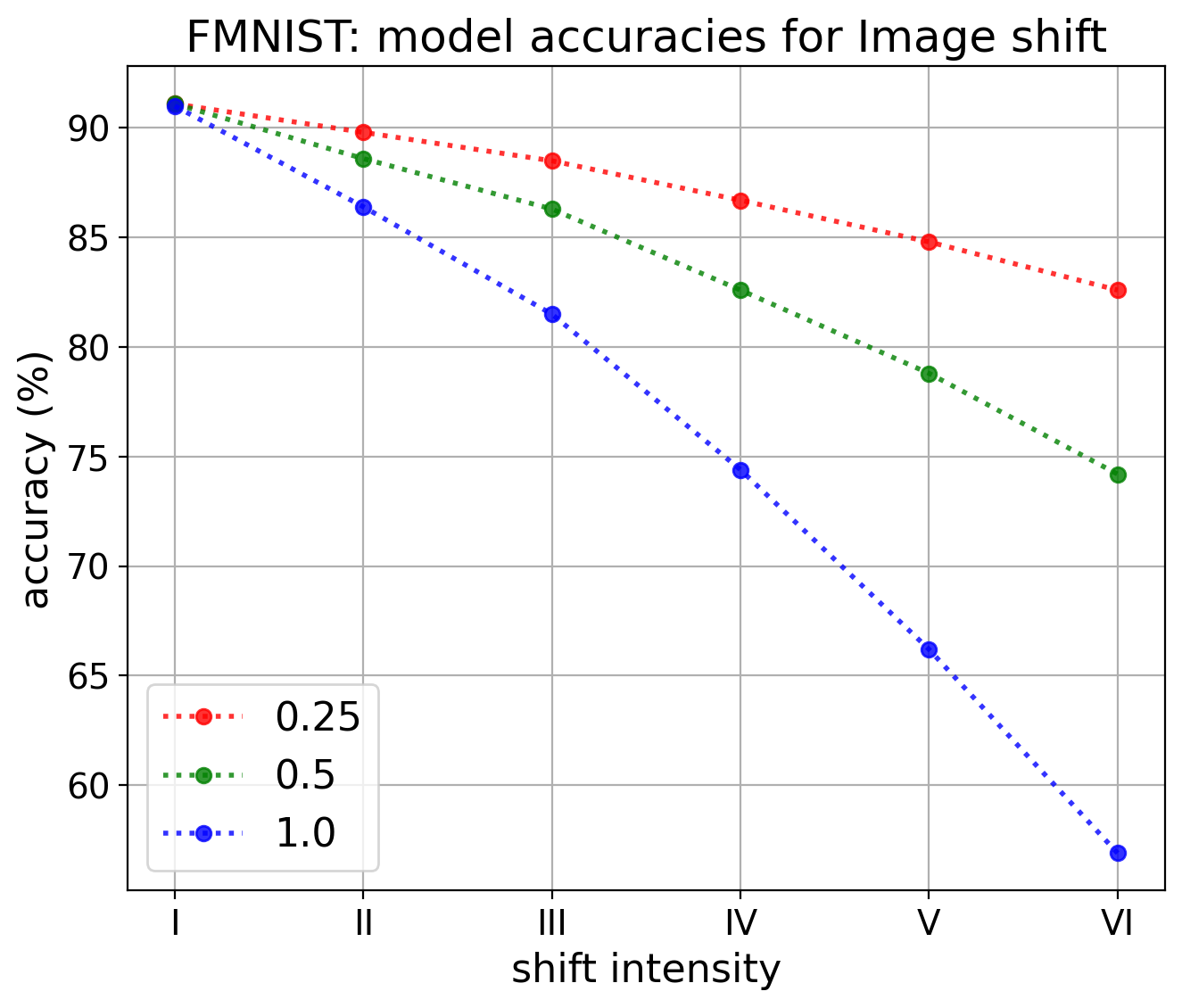}
   \includegraphics[width=0.32\textwidth]{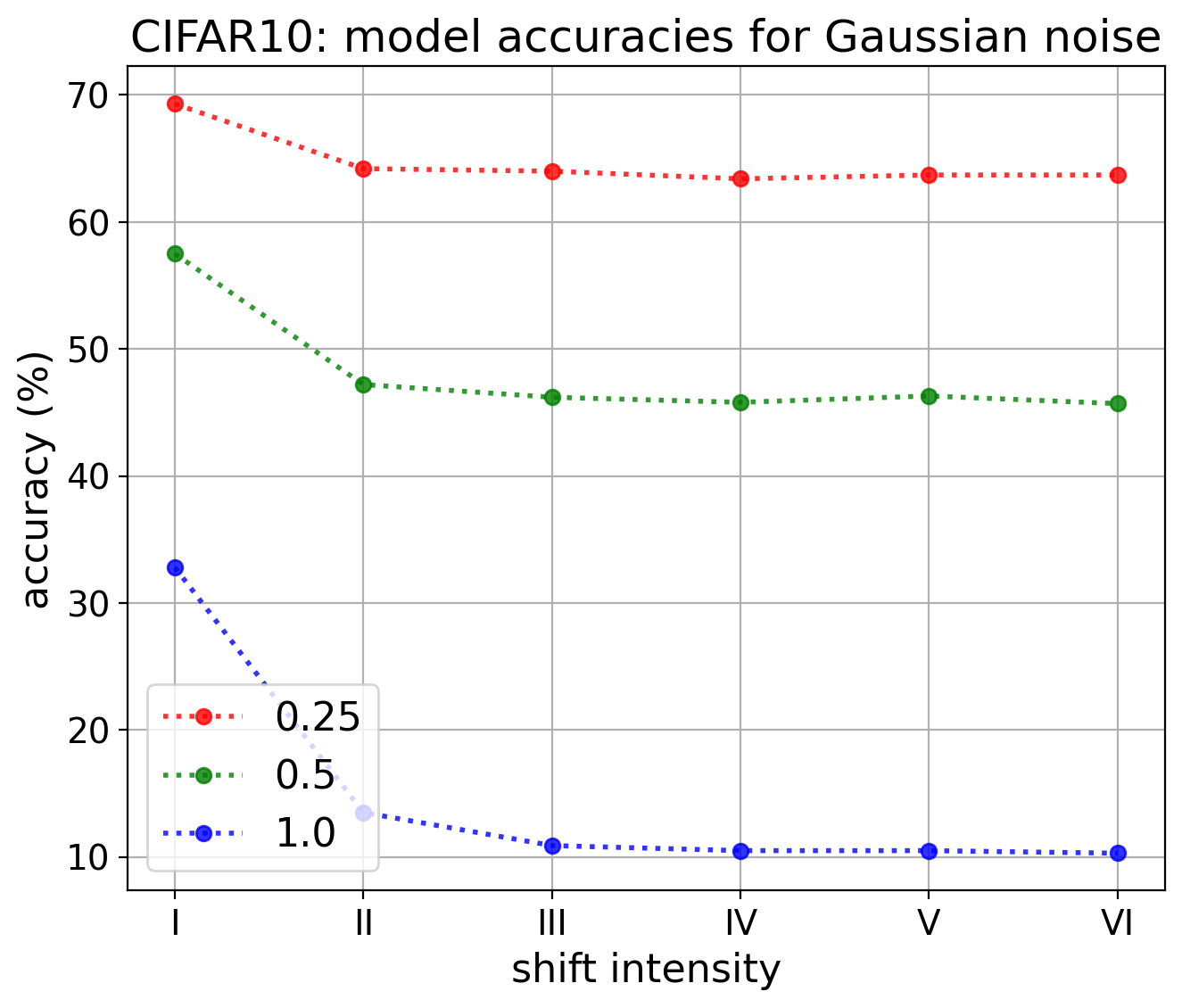}
    \includegraphics[width=0.32\textwidth]{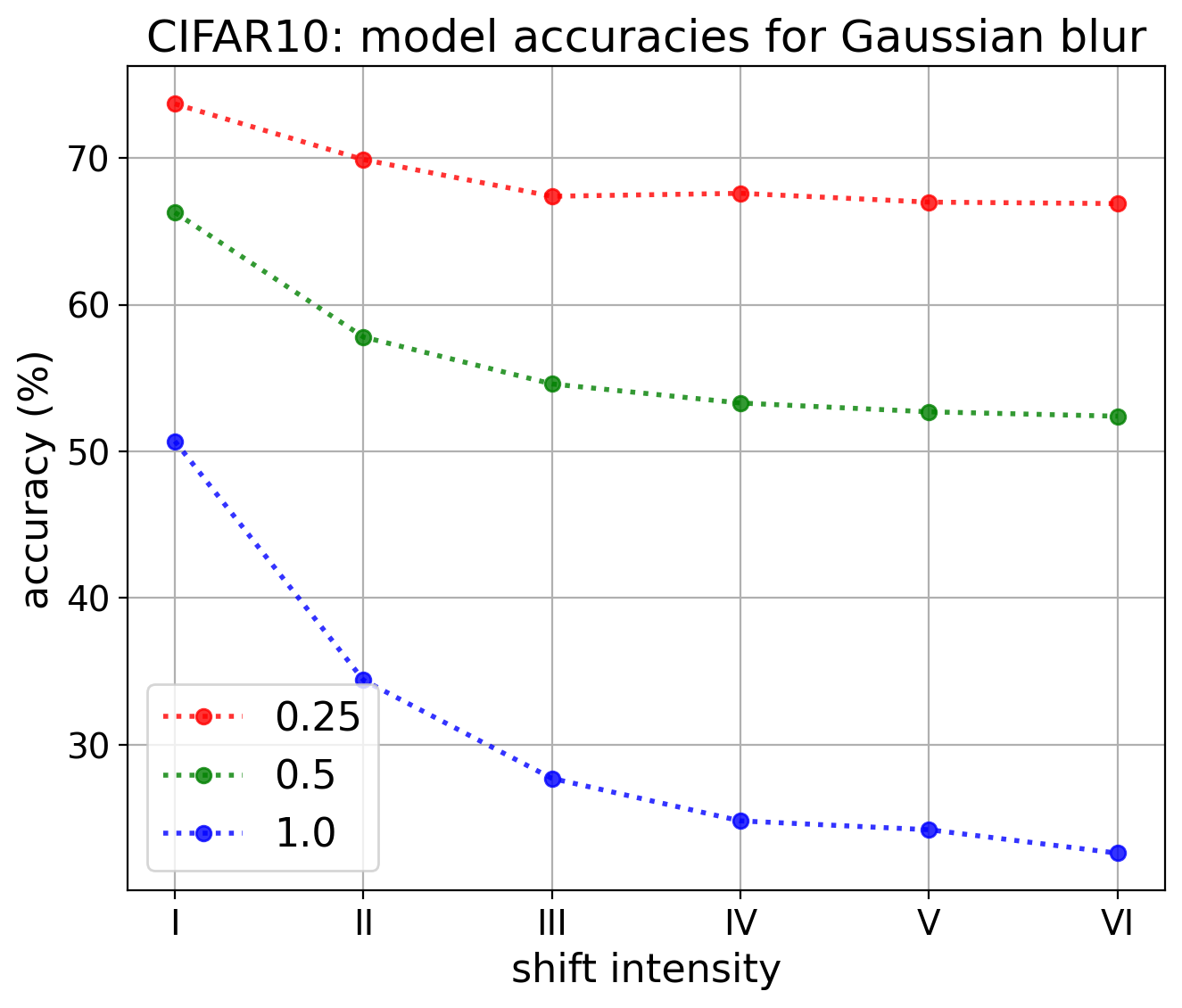}
   \includegraphics[width=0.32\textwidth]{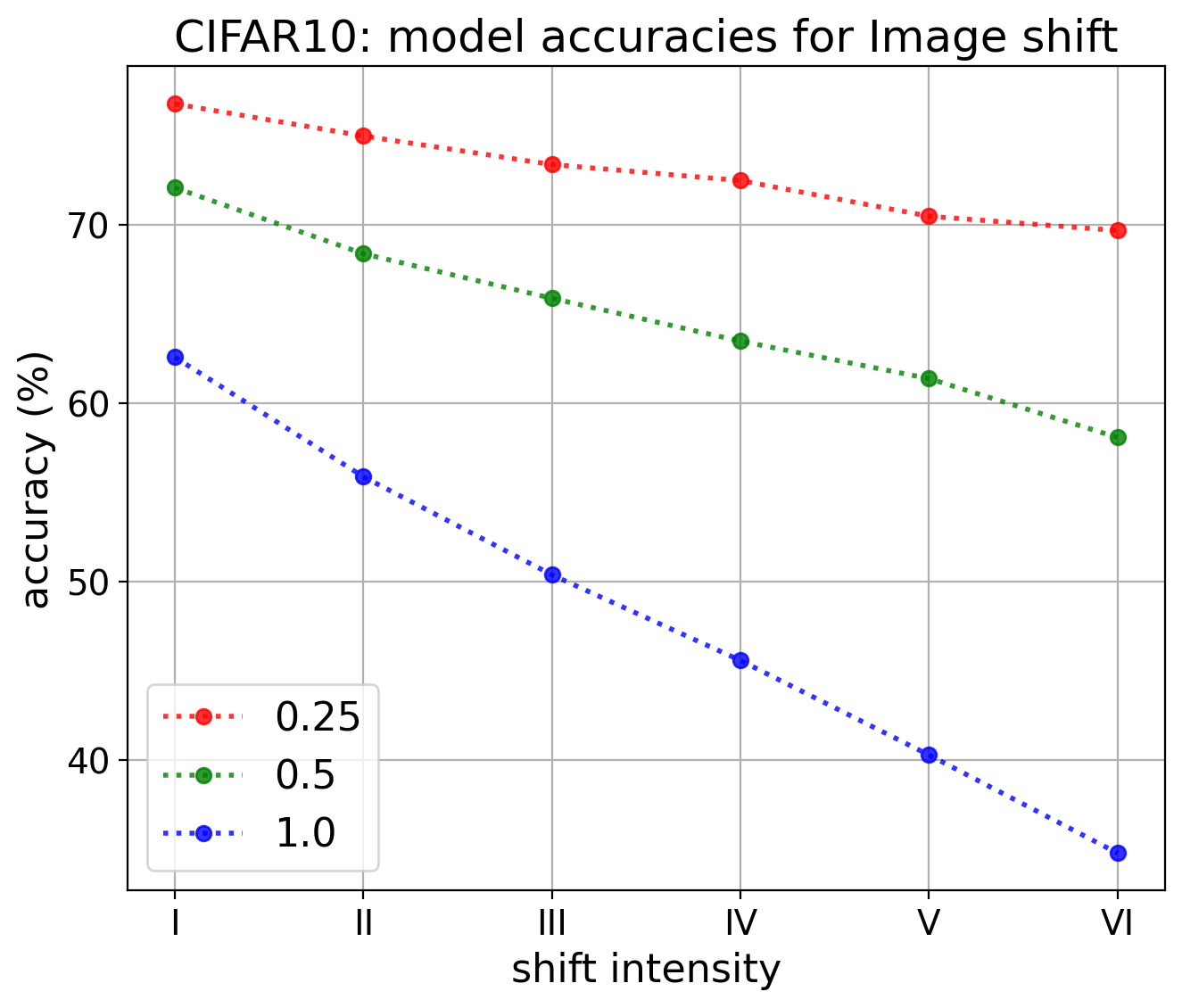}
    \includegraphics[width=0.32\textwidth]{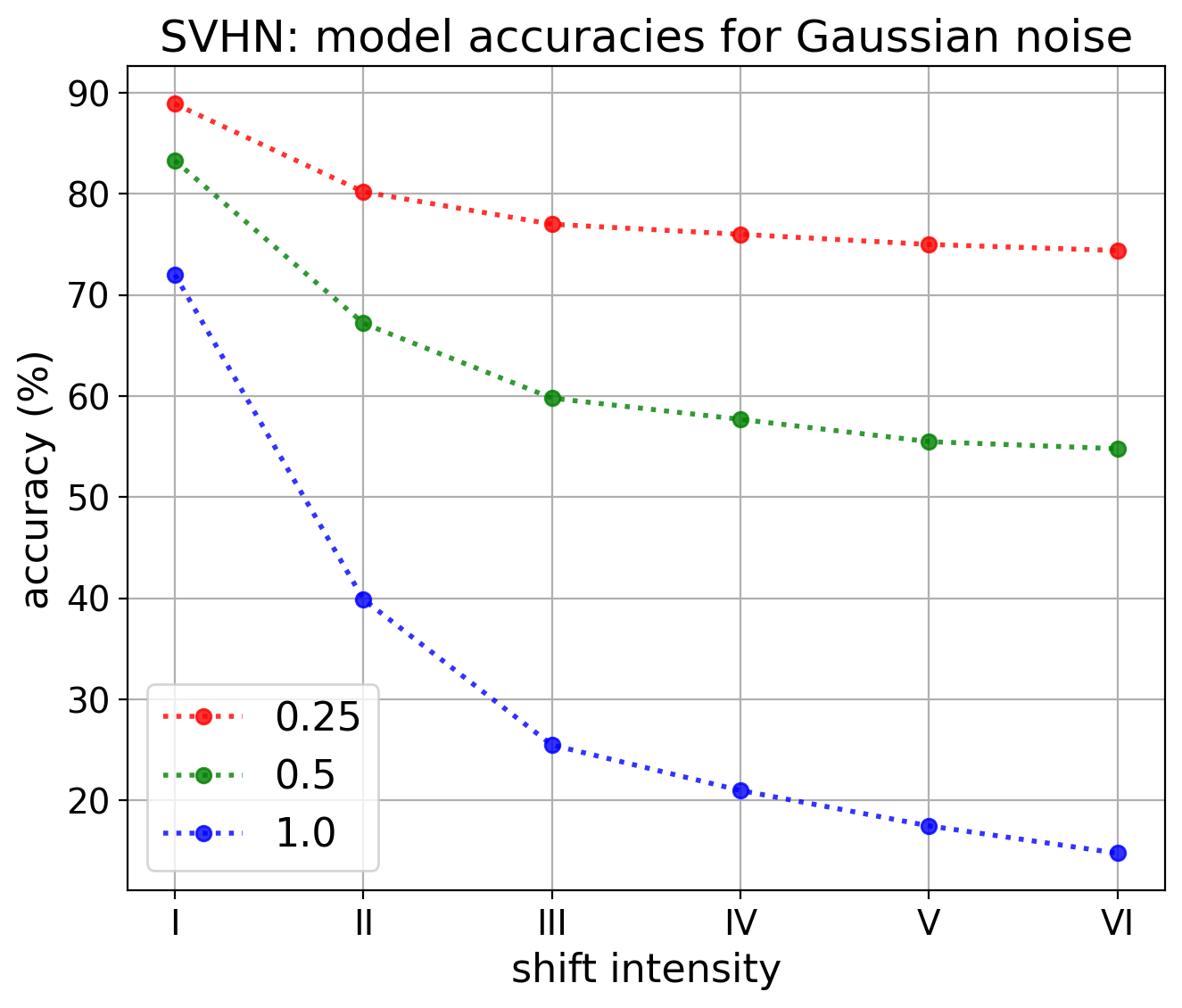}
    \includegraphics[width=0.32\textwidth]{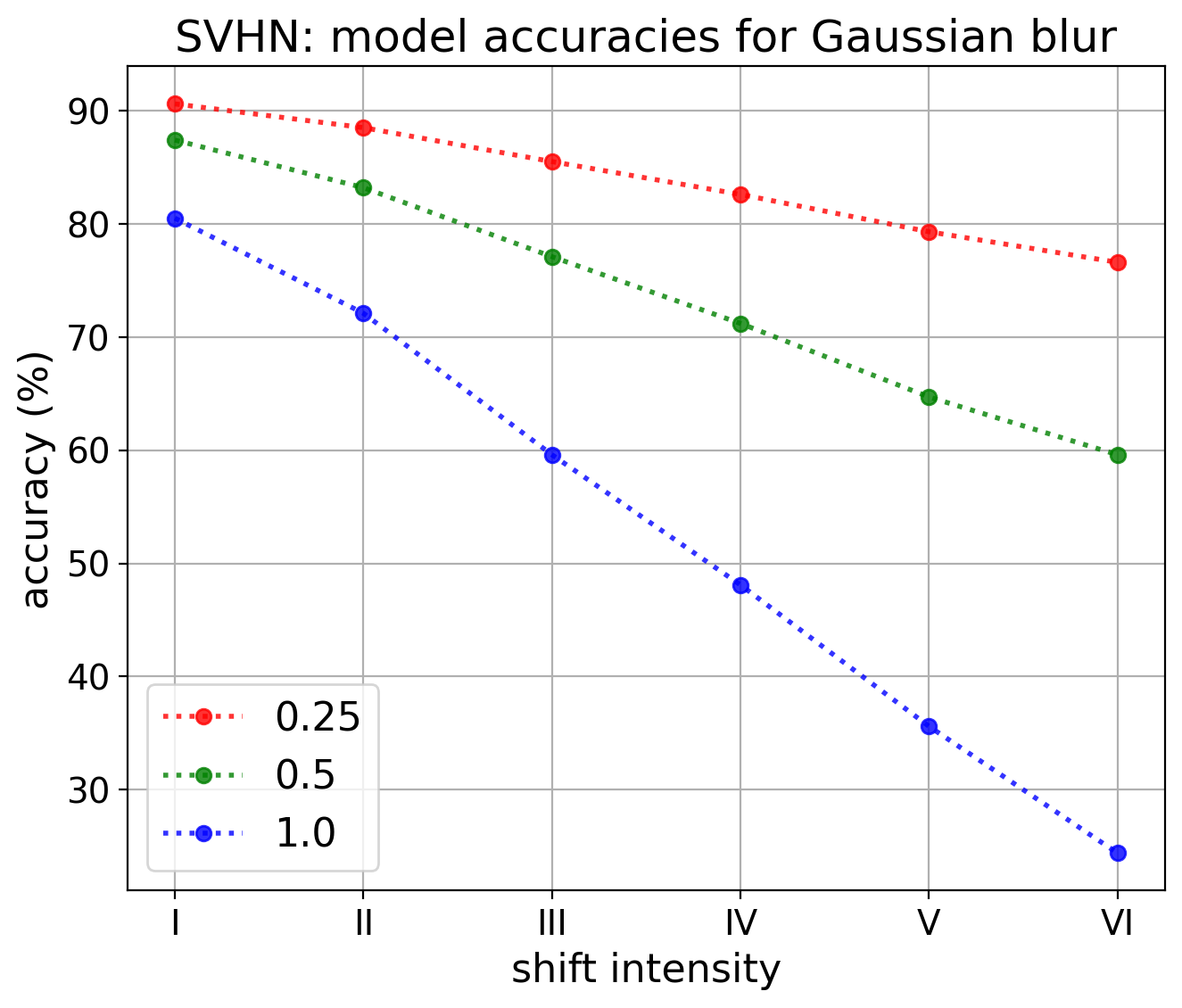}
   \includegraphics[width=0.32\textwidth]{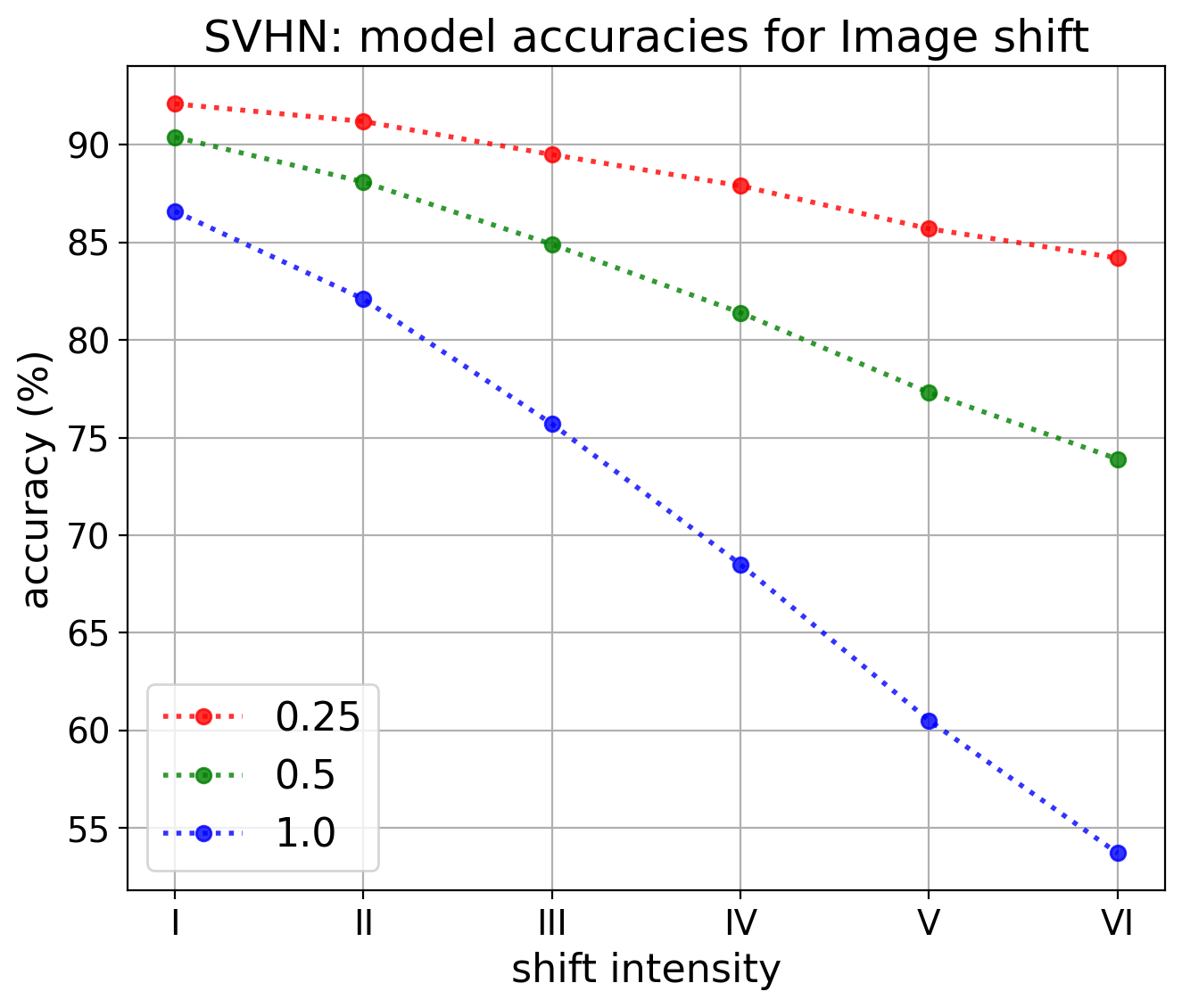}
    \includegraphics[width=0.327\textwidth]{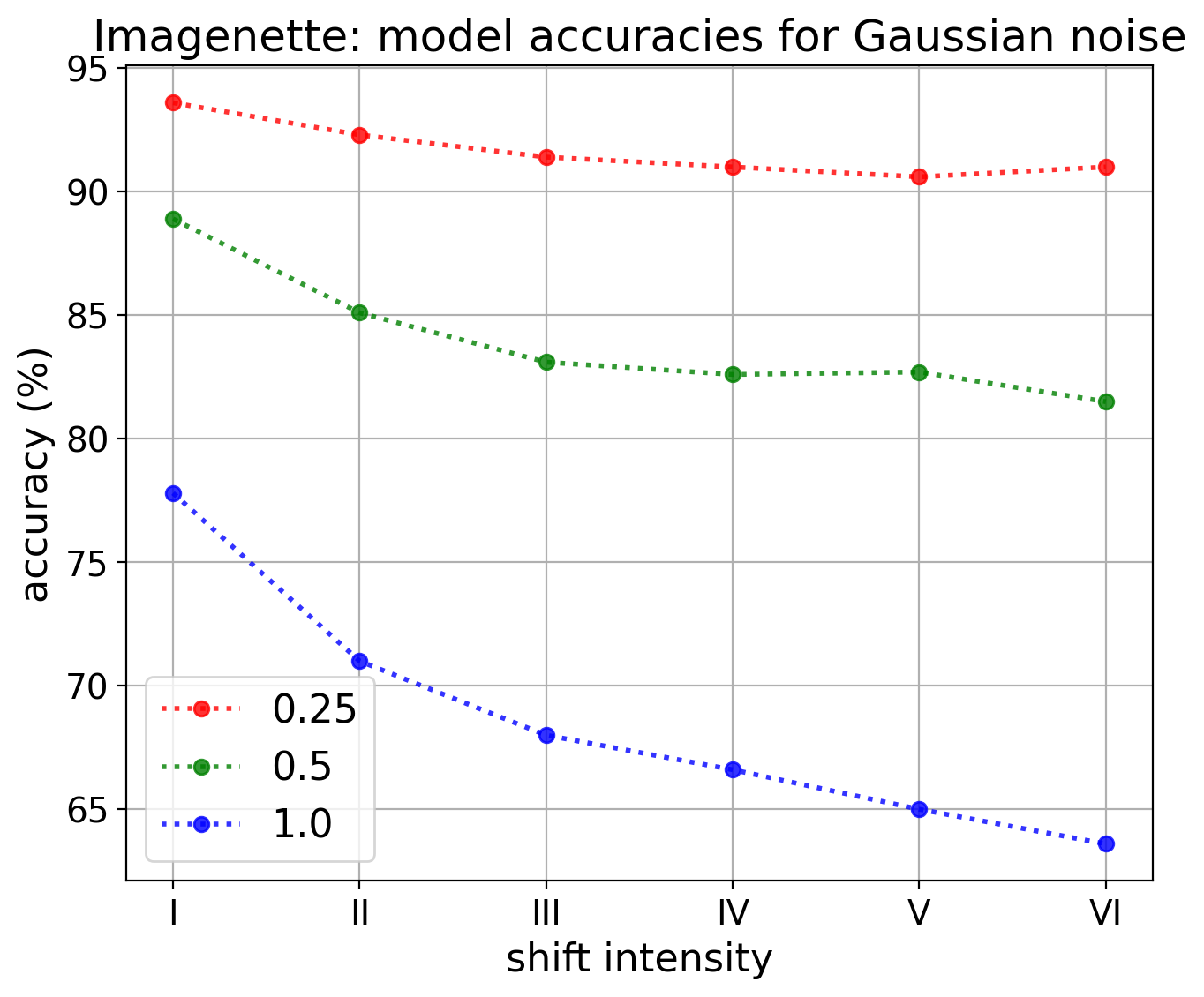}
    \includegraphics[width=0.32\textwidth]{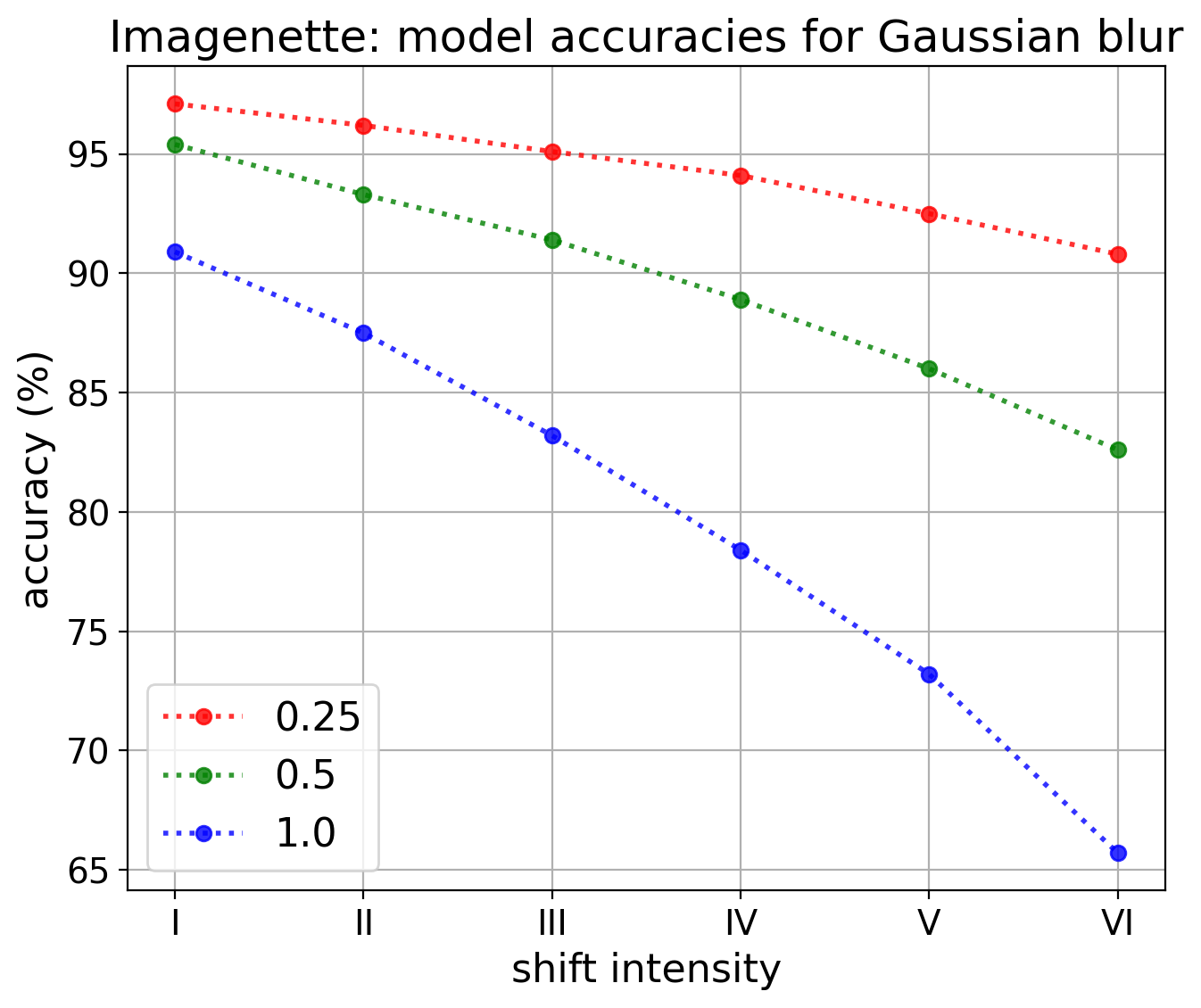}
   \includegraphics[width=0.32\textwidth]{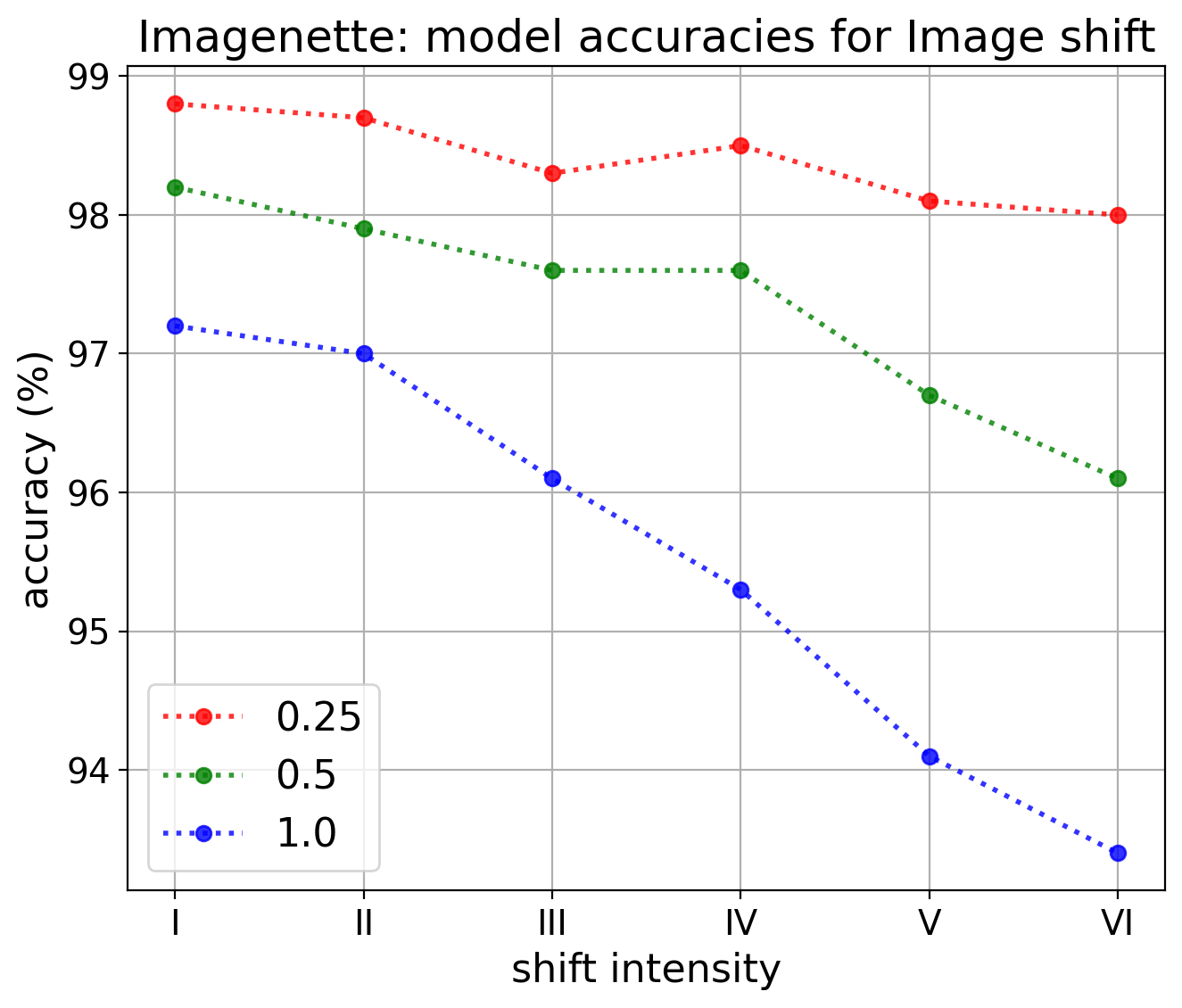}
    \caption{The impact of the shift type and intensity on the model accuracy for $\delta = 1.0$ (blue), $\delta = 0.5$ (green) and $\delta = 0.25$ (red).
    }
     \label{fig:model_accuracy_supplementary}
\end{figure*}

\paragraph{Norm variations.}
As mentioned in the main paper, many variations of \statname{} are conceivable.
Here, we present some experimental results for variations on the type of norm that is used to construct the \statname{} representations.
In Figure~\ref{fig:shift_intensity_norm_variations} we show the results where, instead of the Frobenius-norm, we consider the \emph{spectral norm} as well as the operator norm $\|\cdot\|_\infty$ induced by the sup-norm on vectors.
The spectral norm is equal to the largest singular value and $\|\cdot\|_\infty$ is defined by:
$$
\| M \|_\infty \coloneqq \sup_{x\neq 0} \frac{\| M x\|_\infty}{\|x\|_\infty} = \max_{1\leq i\leq m} \sum_{j=1}^n |m_{ij}|
$$
for $M\in \R^{m\times n}$.
Comparing to Figure~\ref{fig:shift_intensity_MNIST_supplementary}, we observe that the results for the Frobenius-norm and the spectral norm are almost identical.
However, while the results for the $\|\cdot\|_\infty$ are still better (in almost all cases) than those of the baseline CV, they are less powerful than those of the Frobenius norm.

\newpage
\subsection{Theoretical observations regarding the preservation of shift distributions by continuous functions}
In the main article, we mentioned the fact that under generic conditions, two distinct distributions remain distinct under the application of a non-constant continuous function (though this does not necessarily translate to good quantitative guarantees).
In this section, we make this assertion more formal and provide an elementary proof.

Let $X$ be a separable metric space, and denote by  $\mathcal{P}(X)$  the set of probability measures on $X$ equipped with its Borel $\sigma$-algebra.
Let $C_b(X)$ be the real bounded continuous functions on $X$. We consider the weak convergence topology on $\mathcal{P}(X)$; remember that a subbase for this topology is given by the sets 
$$U_{f,a,b} := \left\{\mu \in  \mathcal{P}(X) | \int_X f d \mu \in ]a,b[\right\},$$
for $f\in C_b(X)$ and $a<b\in\R$ (see for example \cite{kallianpur1961topology}).

Now let $X,Y$ be two such separable metric spaces with their Borel $\sigma$-algebra. Any measurable map $F:X\rightarrow Y$ induces a map 
\begin{align*}
    F_* :& \mathcal{P}(X)\rightarrow \mathcal{P}(Y)\\
     & \mu \mapsto F_*(\mu),
\end{align*}
 where $F_*(\mu)$ is the pushforward of $\mu$ by $F$, that is the measure on $\mathcal{P}(Y)$ characterized by $F_*(\mu)(A) = \mu(F^{-1}(A))$ for any Borel set $A \subset Y$.

\begin{fact}\label{Fact:pushforward_continuous}
    If $F:X\rightarrow Y $ is continuous, then $ F_* : \mathcal{P}(X)\rightarrow \mathcal{P}(Y)$ is continuous for the weak convergence topology.
\end{fact}
\begin{proof}
    Given $f\in C_b(X)$ and $a<b\in\R$, we see that $F_*^{-1}(U_{f,a,b}) = U_{f\circ F,a,b}$, which is enough to conclude by the definition of subbases.
\end{proof}

The following result follows from  standard arguments; we give an elementary proof for the convenience of the reader.

\begin{proposition}\label{prop:change_in_distribution}
Let $F:X\rightarrow \R $ be continuous and non-constant for $X$ a separable metric space, and let $\nu \in F_*(\mathcal{P}(X))\subset \mathcal{P}(\R)$. Then the complement $F_*^{-1}(\{\nu\})^c = \mathcal{P}(X)\backslash F_*^{-1}(\{\nu\})$ of the set $F_*^{-1}(\{\nu\})$ is a dense open set of $\mathcal{P}(X) $ for the weak topology.
\end{proposition}

\begin{proof}
As $\R$ is separable and metric, it is easy to show that the singleton $\{\nu\}\in \mathcal{P}(\R)$ is closed (see for example \cite[Thm 4.1]{kallianpur1961topology}). As we know from Fact \ref{Fact:pushforward_continuous} that $F_*$ is continuous, we conclude that $F_*^{-1}(\{\nu\})$ is closed and $F_*^{-1}(\{\nu\})^c$ is open.

It remains to show that it is dense in $\mathcal{P}(X)$. 
Let $\mu$ belong to $F_*^{-1}(\{\nu\})$, and let $V\subset\mathcal{P}(X)$ be an open set containing $\mu$. We have to show that $F_*^{-1}(\{\nu\})^c\cap V$ is non-empty.
Thanks to the definition of the weak topology, we can assume (by potentially taking a subset of $V$) that 
$$V = \bigcap_{i=1}^n \left\{\tilde{\mu}\in \mathcal{P}(X) \text{ s.t. } \int_X f_i d\tilde{\mu}  \in ]a_i,b_i[\right\}$$ for some $f_1,\ldots, f_n\in  C_b(X)$ and $a_1, b_1,\ldots,a_n,b_n \in \R$ with $a_i<b_i$ for all $i$.
Let $x_1$ be any point in the support of $\mu$. Then $\mu(B(x_1,\delta)) >0$ for all $\delta>0$ by definition of the support.
As $F$ is  non-constant, there exists $x_2\in X$ such that $F(x_2)$ is not equal to $F(x_1)$. Let us assume that $F(x_1)>F(x_2)$ (the proof is similar if $F(x_2)>F(x_1)$).
By continuity, there exists $\epsilon>0$ such that $F(x)>F(x_2)$ for any $x\in B(x_1, \epsilon)$. Define $m:=\mu(B(x_1,\epsilon))>0$.
For $t\in ]0,1[$, we define a new measure $\mu_t$ as follows : for any measurable set $A$, we let 
$$\mu_t\left(A \right)= \mu\left(A\backslash  B(x_1,\epsilon)  \right) + (1-t) \mu(B(x_1,\epsilon)\cap A) + tm 1_{x_2\in A}.$$
For any such $t\in ]0,1[$, observe that 
\begin{align*}
     F_*(\mu_t)(]F(x_2),+\infty[) &= 
     \mu_t(F^{-1}(]F(x_2),+\infty[) \\
     &=     F_*(\mu)(]F(x_2),+\infty[) -t\mu(B(x_1,\epsilon)) \\
     &<F_*(\mu)(]F(x_2),+\infty[),
\end{align*} 
which shows that $F_*(\mu)\neq F_*(\mu_t)$, hence that $\mu_t \in F_*^{-1}(\{\nu\})^c$.

On the other hand, we see that $|\int_X f_i d\mu_t  - \int_X f_i d\mu |<2tm||f_i||_\infty$ for $i=1,\ldots,n$. Since $\mu\in V = \bigcap_{i=1}^n \left\{\tilde{\mu}\in \mathcal{P}(X) \text{ s.t. } \int_X f_i d\tilde{\mu}  \in ]a_i,b_i[\right\}$, thus $\mu_t \in V $ for $t\in ]0,1[$ small enough.
This shows that $V\cap F_*^{-1}(\{\nu\})^c$ is non-empty, and thus we conclude that $F_*^{-1}(\{\nu\})^c$ is dense in $\mathcal{P}(X)$.

\end{proof}
As a direct corollary, we get the following statement, where \textit{generic}, as above, means that the property is true for any random variable $x'$ whose distribution belongs to a fixed dense open set of the space of distributions on $\R^n$ :
\begin{corollary}
Let $F:\R^n \rightarrow \R^k$ be a non-constant continuous function represented by a neural network, and let $x$ be a random variable on $\R^n$. For a generic random variable $x'$ on $\R^n$, the distribution of $F(x')$ will be different from that of $F(x)$.
\end{corollary}
\begin{proof}
$\R^n$ is a separable metric space, and if $F$ is non-constant, so is at least one of its coordinate functions $F_i:\R^n \rightarrow \R$, to which Proposition \ref{prop:change_in_distribution} then applies. If the distribution of $F_i(x')$ is different from that of $F_i(x)$, then the distribution of $F(x')$ is different from that of $F(x)$.
\end{proof}

\begin{figure*}[h]
    \centering
    
    \includegraphics[width=0.32\textwidth]{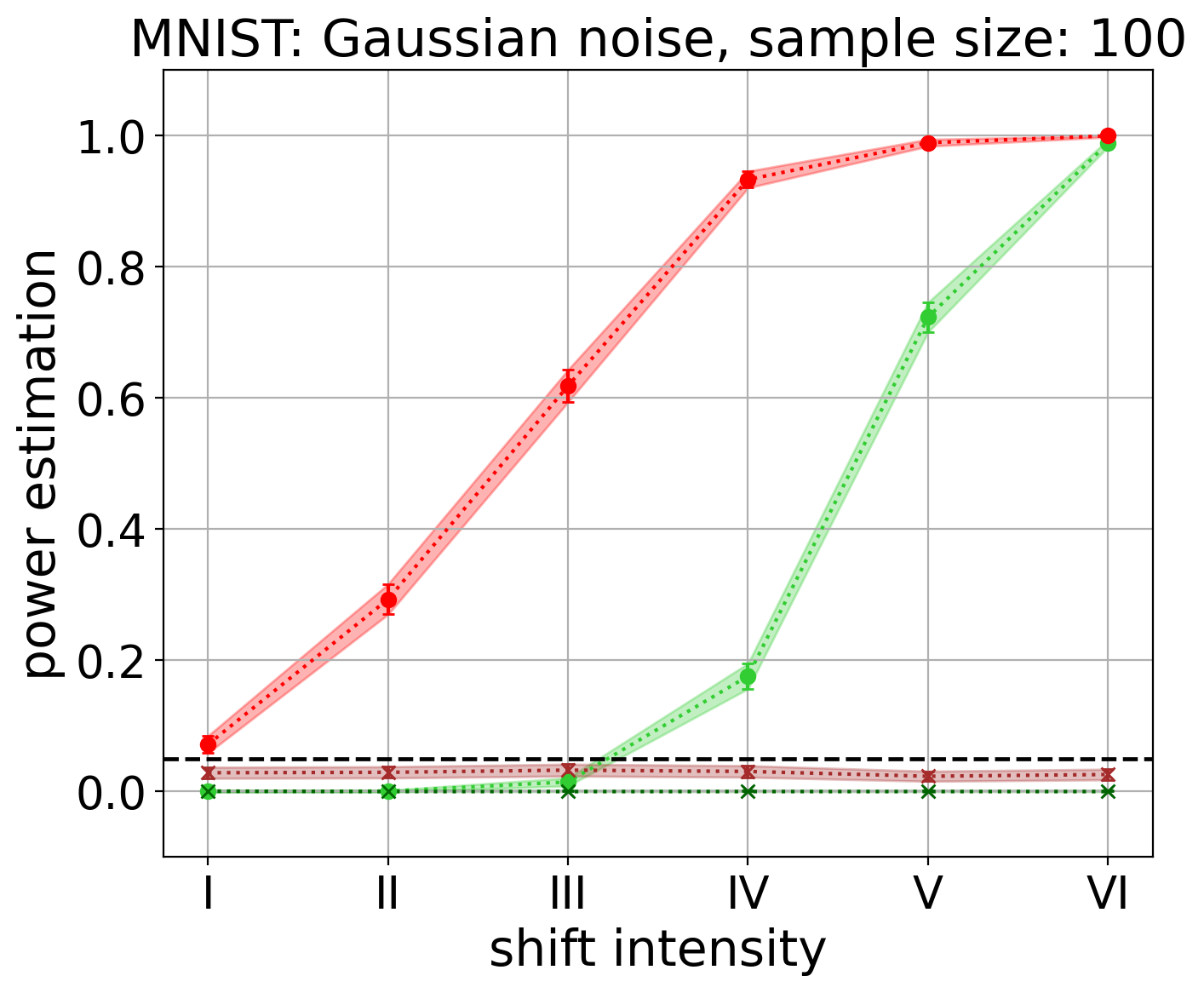}
    \includegraphics[width=0.32\textwidth]{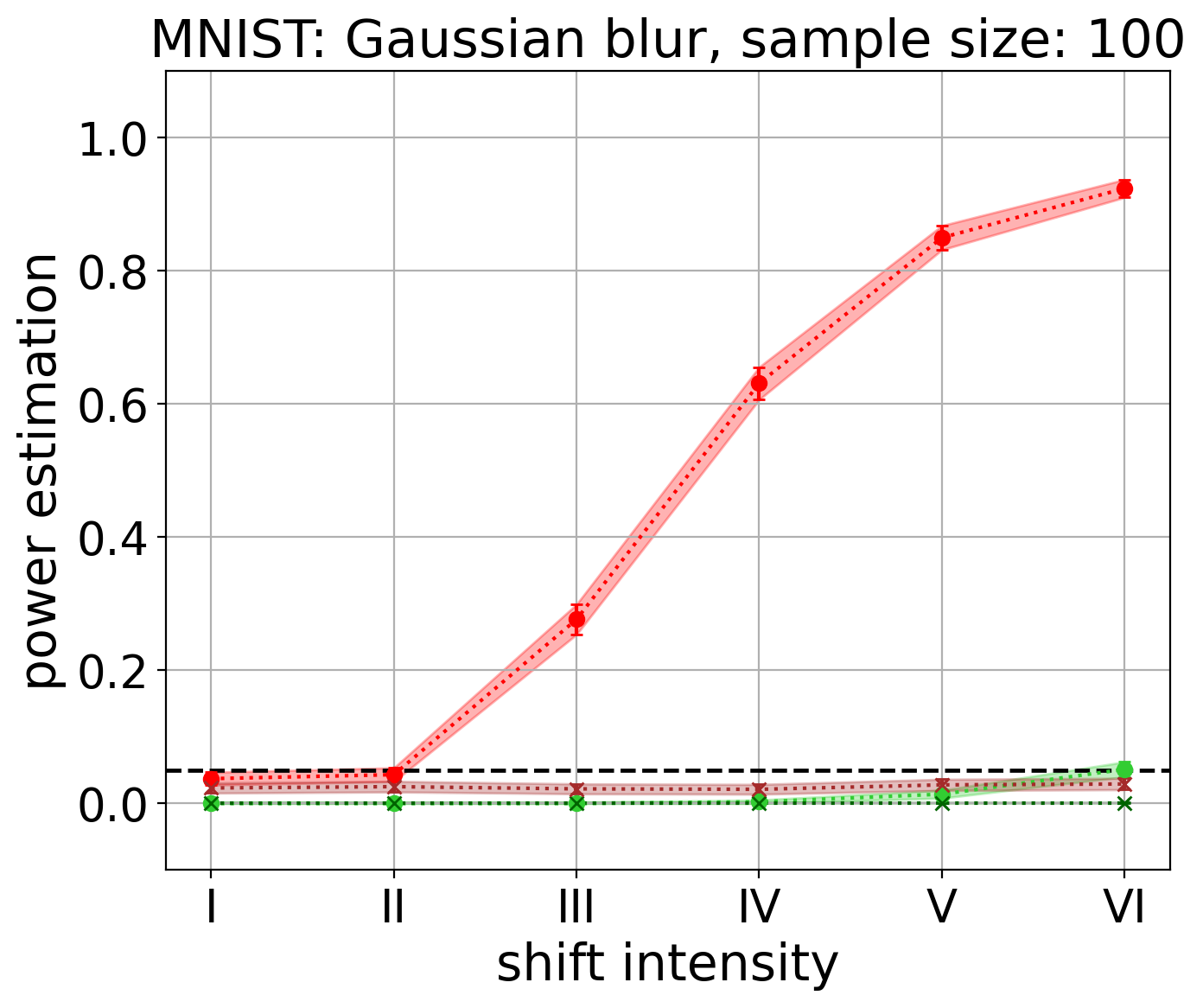}
   \includegraphics[width=0.32\textwidth]{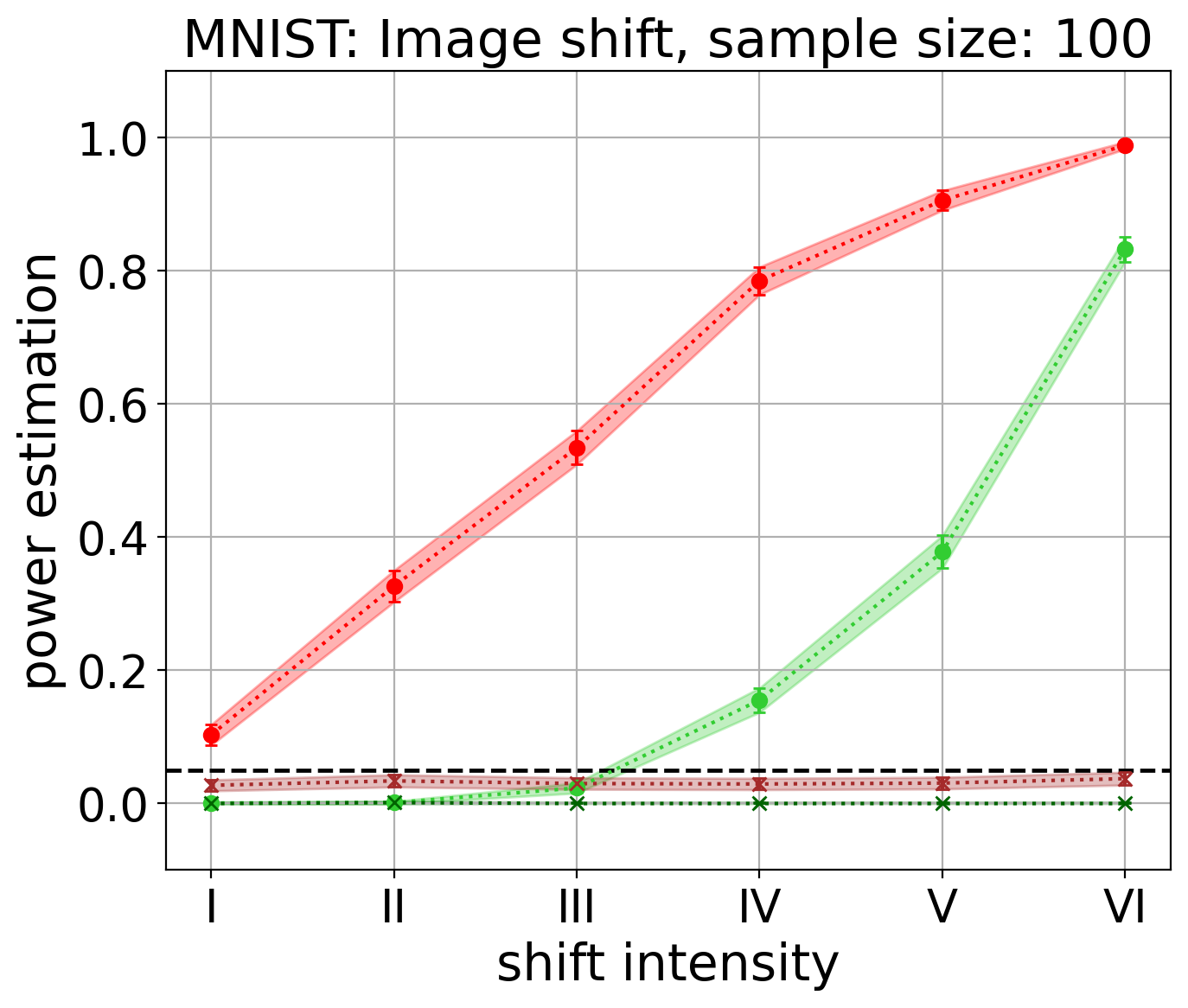}
    \includegraphics[width=0.32\textwidth]{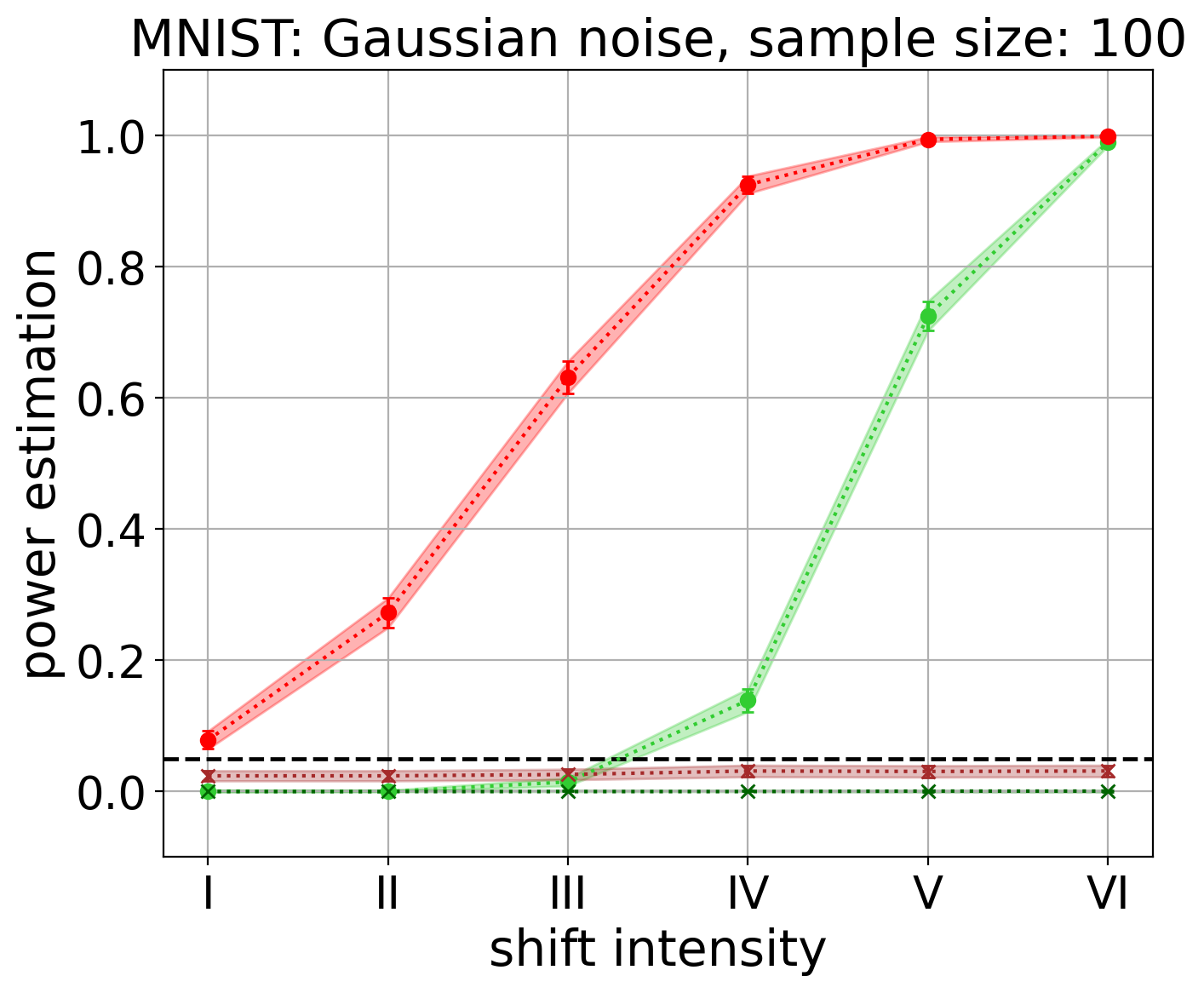}
    \includegraphics[width=0.32\textwidth]{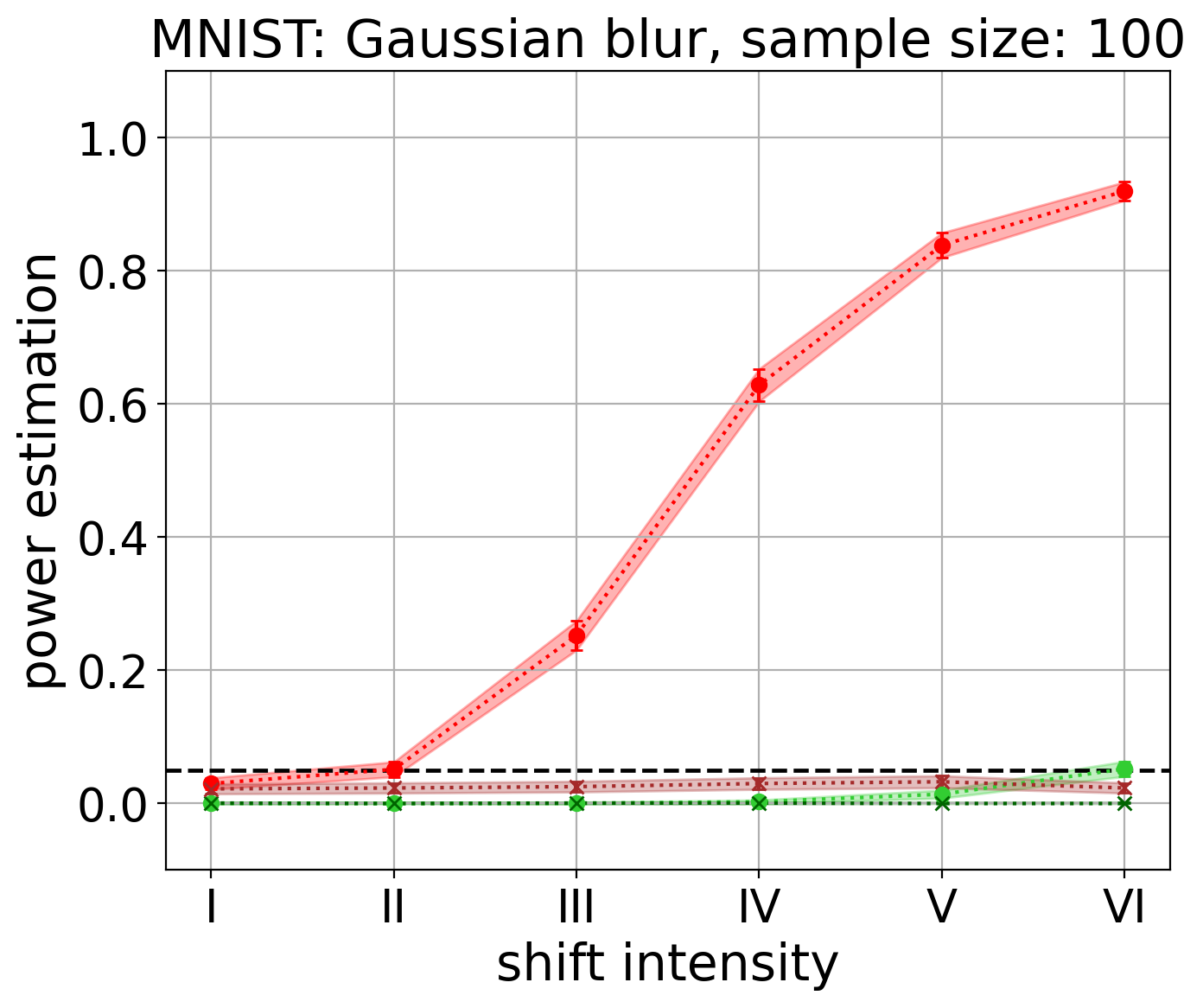}
   \includegraphics[width=0.32\textwidth]{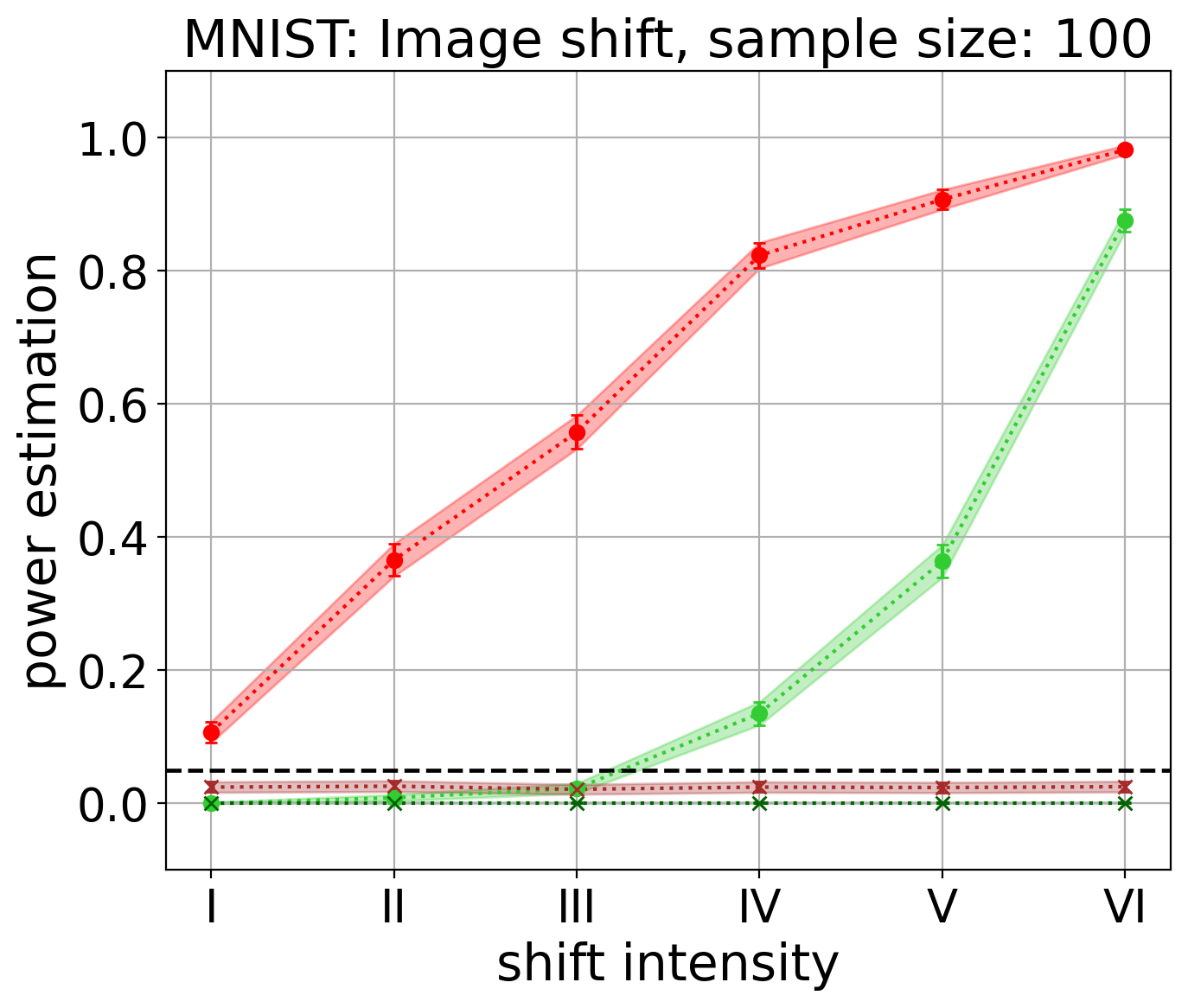}
    \includegraphics[width=0.32\textwidth]{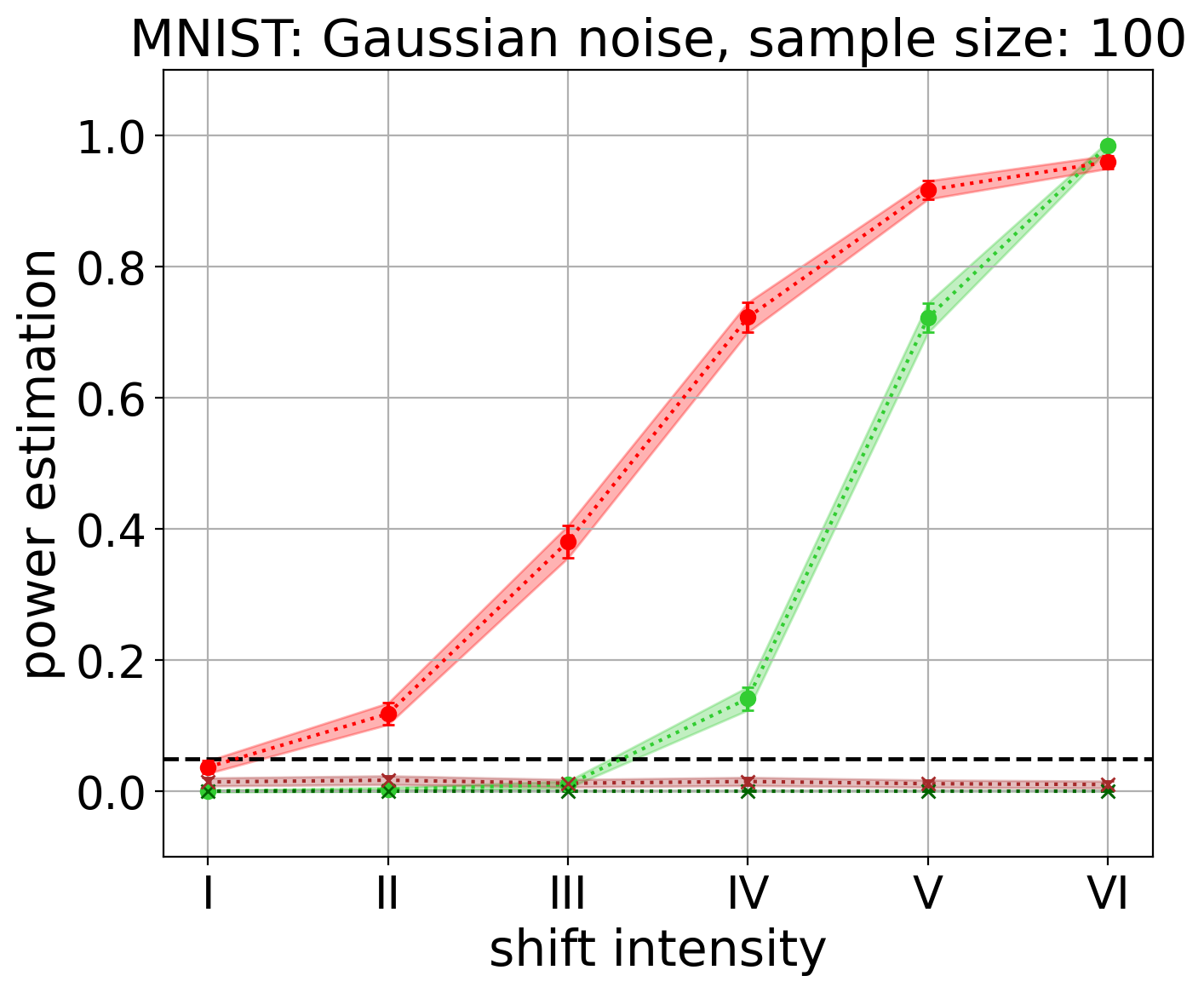}
    \includegraphics[width=0.32\textwidth]{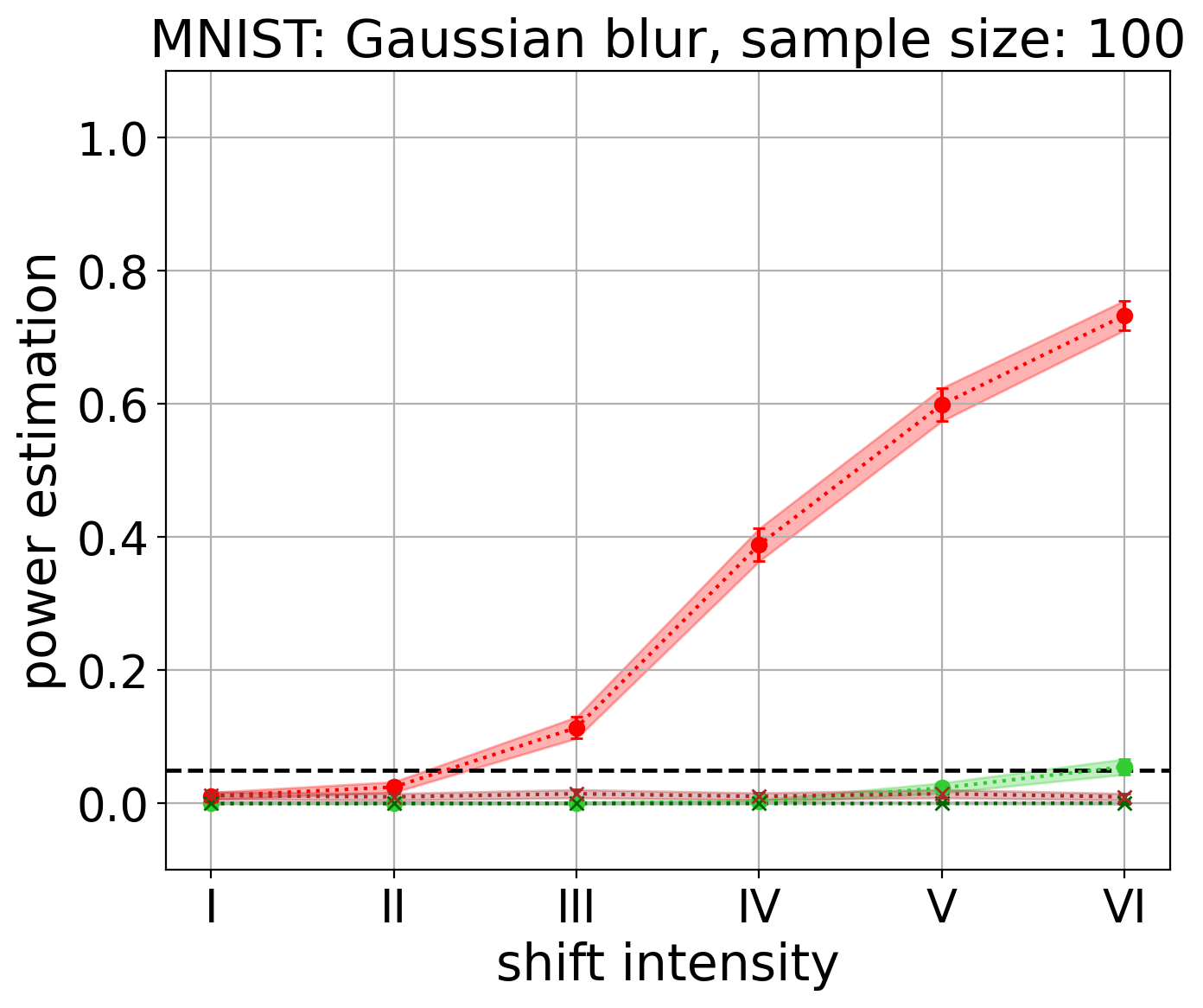}
   \includegraphics[width=0.32\textwidth]{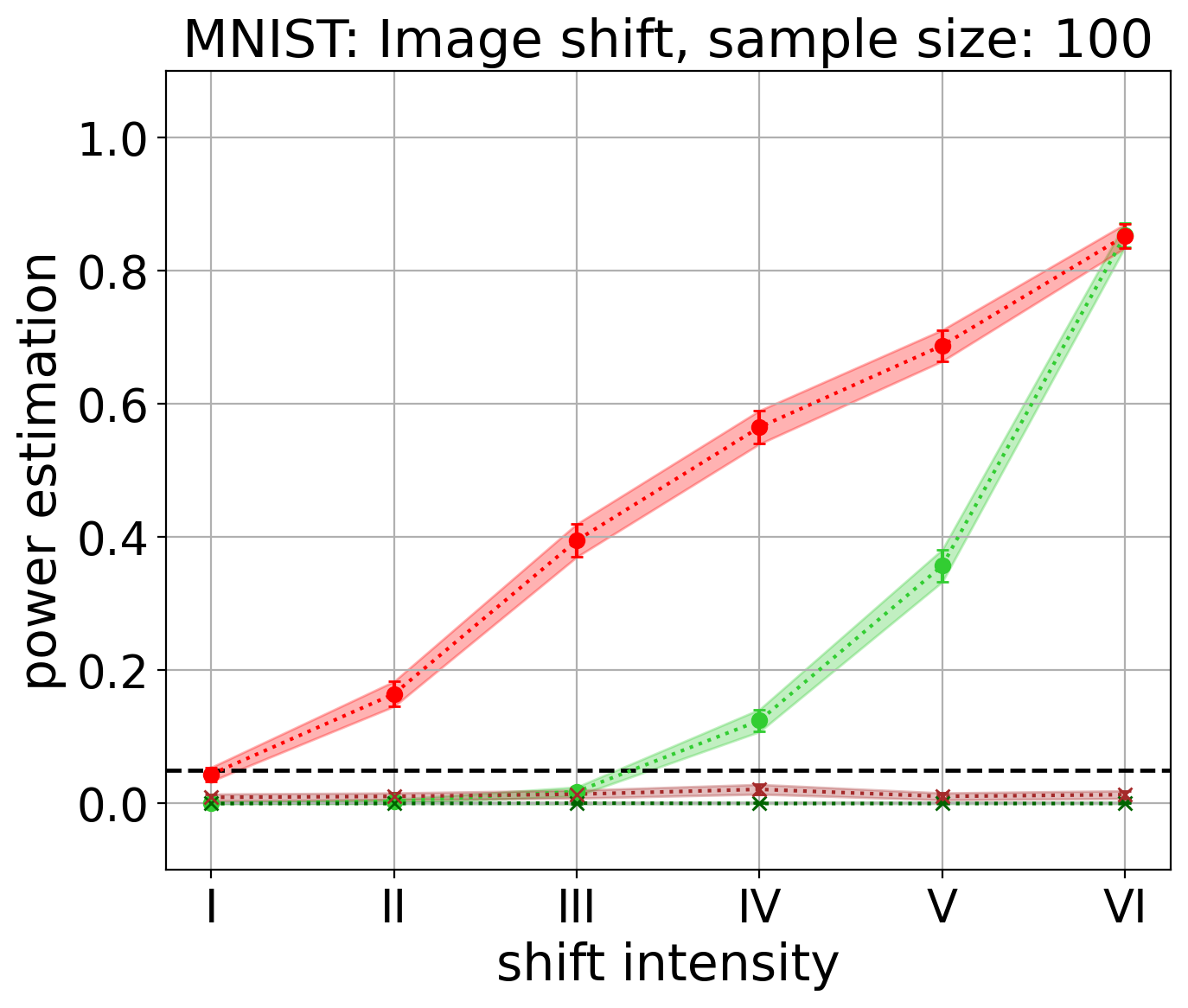}
   
   \caption{Power and type I error of the test with \statname{} (red) w.r.t.~the Frobenius norm, used in all other experiments, (top row), the spectral-norm (middle row) and $\|\cdot\|_{\infty}$ (bottom row) and CV (green) representations w.r.t.~the shift intensity for various shift types on the MNIST dataset with $\delta = 0.5$, sample size $100$, for layer $\ell_{-1}$.
   %(The legend of the curves is as in Figure~\ref{fig:shift_intensity_MNIST_supplementary}.)
   }
     \label{fig:shift_intensity_norm_variations}
\end{figure*}

\end{document}